\RequirePackage{tikz}
\pdfoutput=1

\documentclass[pdflatex]{sn-jnl}

\usepackage[english]{babel}
\usepackage[square,numbers]{natbib}
\usepackage{calc}
\usepackage{amssymb}
\usepackage{amsmath}
\usepackage{graphicx, epstopdf}
\usepackage{subcaption}
\algrenewcommand\algorithmicindent{1em}%
\usetikzlibrary{arrows,angles,quotes,calc}
\usetikzlibrary{decorations.pathreplacing}
\usetikzlibrary{positioning}
\usetikzlibrary{intersections}
\usepackage{booktabs, tabularx}   
\usepackage{float}
\usepackage{pgfplots}
\usepgfplotslibrary{groupplots}
\pgfplotsset{compat=newest}
\usepackage[math=pgfmath]{datatool}
\usepackage{csvsimple}
\usepackage{multirow}
\usepackage{array}
\usepackage{microtype}
\usepackage[autostyle=false]{csquotes}
\usepackage{xfp}
\usepackage{rotating}

\usetikzlibrary{external}
\usepgfplotslibrary{external}
\algnewcommand{\LineComment}[1]{\State \(\triangleright\) #1}
\algnewcommand{\GoTo}[1]{\textbf{goto} #1}

\newtheorem{lemma}{Lemma}
\newtheorem{corollary}{Corollary}

\newtheorem{remark}{Remark}

\newtheorem{procedure}{procedure} 

\newcommand{\cont}[1]{\mathcal{C}^{#1}}
\DeclareMathOperator{\F}{\mathcal{F}}

\newcommand{\Revision}[1]{{#1}}

\DeclareRobustCommand{\divby}{%
  \mathrel{\vbox{\baselineskip.65ex\lineskiplimit0pt\hbox{.}\hbox{.}\hbox{.}}}%
}

\begin{document}
  \title[Fillet-based RRT*]{Fillet-based RRT*: A Rapid Convergence Implementation of RRT* for Curvature Constrained Vehicles}

  \author*[1]{\fnm{James} \sur{Swedeen}}\email{james.swedeen@usu.edu}

  \author[1]{\fnm{Greg} \sur{Droge}}\email{greg.droge@usu.edu}

  \author[1]{\fnm{Randall} \sur{Christensen}}\email{rchristensen@blueorigin.com}

  \affil*[1]{\orgdiv{Department of Electrical and Computer Engineering},
             \orgname{Utah State University},
             \orgaddress{\street{Old Main Hill},
                         \city{Logan},
                         \postcode{84322},
                         \state{UT},
                         \country{USA.}}}

  \abstract{
    Rapidly exploring random trees (RRTs) have proven effective in quickly finding feasible solutions to complex motion planning problems.
    RRT* is an extension of the RRT algorithm that provides probabilistic asymptotic optimality guarantees when using straight-line motion primitives.
    This work provides extensions to RRT and RRT* that employ fillets as motion primitives, allowing path curvature constraints to be considered when planning.
    Two fillets are developed, an arc-based fillet that uses circular arcs to generate paths that respect maximum curvature constraints and a spline-based fillet that uses B\'{e}zier curves to additionally respect curvature continuity requirements.
    Planning with these fillets is shown to far exceed the performance of RRT* using Dubin's path motion primitives, approaching the performance of planning with straight-line path primitives.
    Path sampling heuristics are also introduced to accelerate convergence for nonholonomic motion planning.
    Comparisons to established RRT* approaches are made using the Open Motion Planning Library (OMPL).
  }



  Paper Categories: (1), (2), (3)

  \maketitle

  \section{Introduction}
  The ability to plan paths through complex obstacles is a fundamental requirement of many mobile robot applications and has been shown to be an NP-complete problem in general \citep{Lavalle2006}.
A variety of methods exist to decompose this NP-complete problem into manageable subproblems.
One class of methods that has seen explosive growth in recent years is sample-based motion planning techniques \citep{Gammell2014, Hussain2015, Karaman2011, Moon2015, Nasir2013, Noreen2016, Yang2014}.
Of particular note is the Rapidly-exploring Random Tree (RRT) and its optimal variant, RRT* \citep{Karaman2011}.
RRT quickly plans obstacle free paths for systems with arbitrary motion primitives.
RRT* provides probabilistic guarantees for asymptotically converging to the optimal path, although it is not well-suited for nonholonomic motion constraints.
This work contributes to the RRT* literature by developing techniques that can naturally consider the curvature constraints of a mobile robot with convergence times similar to that of straight-line motion primitives.

In its most simple form, RRT iteratively builds a search tree to randomly explore the environment.
A sample point is randomly selected at each iteration and primitive motions are used to extend the tree in the direction of the sampled point.
The use of general primitive motions enables the application of RRT to a wide range of problems with guarantees of probabilistic completeness \citep{Lavalle2006}.
However, the path that RRT finds is typically far from optimal.
\cite{Karaman2011} developed RRT*, which makes two modifications to the RRT algorithm that probabilistically result in asymptotic optimality.
Both modifications perform local optimizations to the tree using a neighborhood of nodes around each new point being added to the tree. These local optimizations work to add and remove edges between existing nodes, requiring the motion primitives to be able to find a continuous path to connect the corresponding states represented by the nodes. This exact connection requirement is not required by the original RRT algorithm, so the RRT* changes are not generally applicable to all applications of RRT.
Moreover, the same random sampling that ensures that RRT finds a solution causes the asymptotic convergence to the optimal solution by RRT* to be quite slow \citep{Akgun2011, Gammell2014, Kobilarov2012, Nasir2013, Noreen2016}.

Exactly connecting two states can become difficult when considering the motion constraints of wheeled vehicles, such as path curvature.
One common method for considering curvature constraints is to plan with straight-lines and arcs using techniques such as Dubin's and Reed-Shepp paths \citep{Beard2012, Lavalle2006}.
While these techniques provide the shortest paths between oriented waypoints, they are not well-suited for the local optimization procedures in RRT* due to the inclusion of orientation \citep{cui2018}.
The path exactly connecting two oriented points can vary significantly with small changes in orientation.
An alternative motion primitive in the form of a fillet was used in \citep{Yang2014, spline_rrt*}. Instead of connecting two points, a fillet connects two straight-line segments with a curve that starts on the first segment and ends on the second segment. Small changes in each line will produce small changes in the path length, making the fillet approach amenable to the local optimizations required by RRT*.
In the case of \citep{Yang2014,spline_rrt*}, B\'ezier curves were used to connect the line segments, with the added benefit that continuous change in curvature is guaranteed.

While fillets enable the use of local optimization techniques, convergence to the optimal solution is naturally rather slow in RRT*.
To overcome slow convergence rates, many alternative sampling and path refinement procedures have been introduced \citep{Akgun2011, Gammell2014, Kobilarov2012, Nasir2013, Tahir2018}.
This work utilizes two such approaches.
In \citep{Gammell2014}, Informed RRT* \mbox{(I-RRT*)} attempts to reduce the sampling space by providing a conservative estimate of the area that will contain the optimal solution.
In \citep{Nasir2013}, Smart RRT* \mbox{(S-RRT*)} provides an alternative sampling heuristic as well as a path refinement procedure to avoid waiting for the probabilistic sampling to straighten the path.
We then combine ideas from \citep{Gammell2014} and \citep{Nasir2013} to develop the novel Smart and Informed RRT* \mbox{(SI-RRT*)} which provides greedy refinement of the solution without as many parameters to tune as \mbox{S-RRT*}.

This paper develops a Fillet-based RRT* \mbox{(FB-RRT*)} algorithm for curvature constrained path planning.
Similar to \citep{Yang2014,spline_rrt*}, a fillet approach is used to locally connect points.
Contributions to \citep{Yang2014, spline_rrt*} include the generalization of the fillet structure for RRT planning, a relaxation of connection assumptions that increases flexibility in growing the tree, and a newly developed rewiring procedure to ensure continuity and cost improvement in the resulting path.
Established sampling and path refinement procedures are also extended to the fillet structure.
A minor contribution of this work is the combination of two sampling heuristics \citep{Gammell2014, Nasir2013} within the fillet framework.


The remainder of the paper proceeds as follows.
In Section \ref{section:backround}, the basics of RRT and RRT* are introduced.
Section \ref{section:fillet} then introduces the fillet approach to local planning.
Section \ref{sec:fillet-rrt-star} develops procedures for incorporating the fillet approach into RRT* \Revision{with a brief description of reverse fillet considerations given in Appendix \ref{sec:reverse_fillet}.}
Section \ref{section:improving_convergence} then develops the smart-and-informed sampling and path refinement procedures and presents the Fillet-based RRT* algorithm.
Results are presented in Section \ref{sec:results} using the Open Motion Planning Library (OMPL) \citep{ompl} to benchmark the performance of RRT* using a straight-line motion primitive, an arc-based fillet, a B\'{e}zier curve fillet, Dubin's paths, and various sampling techniques.
Concluding remarks are given in Section \ref{sec:conclusion}.

  \section{The Rapidly-exploring Random Tree}
  \label{section:backround}
%

RRT-based algorithms are commonly broken into a series of generalized space sampling and tree growing procedures.
RRT's variants (including RRT*) refine and augment these procedures.
This section defines several basic procedures, giving them context within RRT and RRT*.
The procedures in this section are found in \citep{Karaman2011, bezier_curves} with variations in notation.
They are included for the sake of completeness in presenting the Fillet-Based RRT* (FB-RRT*) formulation and sampling heuristics in Sections \ref{section:fillet} through \ref{section:improving_convergence}.

\subsection{Notation}
RRT-based algorithms iteratively construct a rooted, out-branching tree to find a path through the state space.
The tree is an acyclic directed graph denoted as $T=\{V,E\}$, where $V$ is the set of nodes or vertices within the tree and $E \subset V \times V$ denotes the set of edges between vertices.
If an edge points from $v_1$ to $v_2$, $v_2$ is referred to as the child of $v_1$ and $v_1$ as the parent of $v_2$.
The root has no parent while all other vertices have exactly one parent.
Each vertex can have multiple children.
Each vertex within $V$ corresponds to a state in the $d$-dimensional state space denoted as $X \subset \mathbb{R}^d$.
Note that we will assume that $X \subset \mathbb{R}^2$, although that is certainly not the case for general RRT formulations.
The tree is initialized with solely the root node, i.e. $V = \{x_{r}\}$, $E = \varnothing$.
Nodes are added to the tree to find a path to the target set, $X_{t} \subset X$. The path must avoid the space blocked by obstacles, $X_{obs} \subset X$, staying within the free space, $X_{free} = X \setminus X_{obs}$.

Paths through the state space are written as an ordered subset of $X$.
In RRT*, the paths are assigned a cost, typically the path length.
Allowing $P(X)$ to denote the power set of $X$, the cost is a mapping $c:P(X) \rightarrow \mathbb{R}_+$.
The closed ball of radius $r \in \mathbb{R}_+$ centered at $x \in \mathbb{R}^d$ is denoted as $\mathcal{B}_{x,r} = \left\{y \in X : \left\|y-x\right\| \leq r \right\}$.
$X_{s} \subset X_{free}$ is the set of states from which the additional nodes will be sampled.
Additional notation is summarized in Table \ref{tab:notation}.

\begin{table*}[htp]
\caption{Notation used throughout the paper}
\label{tab:notation}
  \centering
  {\footnotesize
  \setlength{\tabcolsep}{1pt}
  \newcolumntype{L}[1]{>{\raggedright\let\newline\\\arraybackslash\hspace{0pt}}p{#1}}
  \newlength{\fullwidth}
  \newlength{\namewidthone}
  \newlength{\namewidthtwo}
  \setlength{\fullwidth}{\textwidth - 8\tabcolsep}
  \setlength{\namewidthone}{\widthof{$c\!\left(X_{path}\right)$}}
  \setlength{\namewidthtwo}{\widthof{$x_{nearest}$}}
  \begin{tabularx}{\textwidth}{L{\namewidthone} L{(0.5 \fullwidth) - \namewidthone} L{\namewidthtwo} L{(0.5 \fullwidth) - \namewidthtwo}}
    \toprule
    \textbf{Name} & \textbf{Description} & \textbf{Name} & \textbf{Description} \\
    \toprule
    \multicolumn{4}{l}{General sets} \\
    $\mathbb{Z}$   & The set of all real-valued integers                &
    $\mathbb{R}_+$ & The set of all positive real numbers        \\
    $\mathbb{R}$   & The set of all real-valued numbers                 &
    $\mathbb{R}^d$ & The set of all real $d$-dimensional vectors \\
    \hline
    \multicolumn{4}{l}{Configuration space} \\
    $X$        & The full configuration space, $X \subset \mathbb{R}^2$ &
    $X_{obs}$  & Obstacle filled configuration space                                                                                 \\
    $X_{free}$ & Obstacle free configuration space                                                                                   &
    $X_{t}$    & The target set                                                                                                      \\
    $X_{s}$    & The sampling set                                                                                                    &
    $X_{b}$    & The beacon set                                                                                                      \\
    $X_{path}$ & An ordered set of states                                                                                            &
    $X_{near}$ & A set of nodes near a given node                                                                                    \\
    $X_i$      & The set that, if sampled, will improve the path                                                                     &
    $X_{sol}$  & The solution path                                                                        \\
    $x$        & An element of $X$, a 2D position                                                                                    &
    $\psi$     & Used to denote orientation                                                                                          \\
    \hline
    \multicolumn{4}{l}{Special sets} \\
    $\mathcal{E}_{x_a,x_b}$ & A subset of $X$ in an ellipse with focal points $x_a$ and $x_b$            &
    $\cont{y}$              & The set of all $y \in \mathbb{Z}_+$ times continuously differentiable paths \\
    $\mathcal{B}_{x,r}$     & The closed ball of radius $r$ and centered at $x$                           \\
    \hline
    \multicolumn{4}{l}{Operators} \\
    $R(\theta)$                & The right-handed rotation matrix parameterized by angle $\theta$                                        &
    $a \divby b$               & Tests if $a$ is divisible by $b$ (i.e., true if $a \mbox{ } mod \mbox{ } b == 0$) \\
    $c\!\left(X_{path}\right)$ & Cost of a path through the state space                                                                  &
    $\left\|x\right\|$         & The 2-norm of $x$                                                                                       \\
    $\land$, $\lor$            & Logical \enquote{and} and \enquote{or} operators                                                        &
    $\binom{n}{k}$             & N choose k, i.e. $\binom{n}{k} = \frac{n!}{k!(n-k)!}$                                                   \\
    \hline
    \multicolumn{4}{l}{Tree notation} \\
    $V$   & A set of vertices                             &
    $T$   & An acyclic directed graph, $T = \left\{V,E\right\}$ \\
    $E$   & A set of edges that connect vertices          &
    $x_r$ & The root node                                 \\
    \hline
    \multicolumn{4}{l}{Planning Parameters} \\
    $\alpha$ & The max number of neighbors to consider      &
    $\eta$   & The max distance to travel when steering a point \\
    $\rho$   & The radius for searching for nearest neighbors               &
    $b_t$    & Determines how frequently $X_t$ will be sampled              \\
    $\gamma$ & The angular displacement between two vectors                 &
    $b_b$    & Determines how frequently $X_b$ will be sampled              \\
    \hline
    \multicolumn{4}{l}{Tree Search Variables} \\
    $x_p$         & A parent node                                           &
    $x_c$         & A child node                                            \\
    $x_{gp}$      & A parent of a parent (i.e., grandparent node)           &
    $x_{gc}$      & A child node of a child (i.e., grandchild node)         \\
    $x_n$         & A new node to be added to the tree                      &
    $x_{rand}$    & A randomly sampled state                                \\
    $c_y$         & The cost of the node $x_y$                              &
    $x_{nearest}$ & The point that is closest to $x_n$ in $V$               \\
    $x_{best}$    & The last node in the best $X_{sol}$ set found so far    &
    $x_{near}$    & A point that is close to $x_n$, $x_{near} \in X_{near}$ \\
    \hline
    \multicolumn{4}{l}{Fillets} \\
    $\kappa_{max}$             & The maximum path curvature allowed                                                                                                &
    $s$                        & A spatial indexing                                                                                                                \\
    $x_1$, $x_2$, $x_3$        & The point where a fillet begins, the center point, and the ending point                                  &
    $s_0$, $s_1$, $s_2$, $s_3$ & The index of the fillet at $x_1$, $x_s$, $x_e$, and $x_3$ respectively                                                              \\
    $x_s$, $x_e$               & Points where fillet curve starts and ends                                                                                         &
    $\F_i$                     & Abbreviation for the full fillet length                                                                                           \\
    $\Psi(s)$                  & The spatially indexed curve of the fillet                                                                                         &
    $\Psi_i$                   & The path length of the curve centered at $x_i$                                                                                    \\
    $b_i$                      & Distance from $x_{i-1}$ to the beginning of the fillet curve centered at $x_i$                                                    &
    $e_i$                      & Distance from $x_{i+1}$ to the end of the fillet curve centered at $x_i$                                                          \\
    $d(\gamma)$                & The distance from $x_s$ or $x_e$ to $x_2$ given the angular displacement between fillet lines                                     &
    $m_{i,j}$                  & Distance on the line $\overline{x_i x_j}$ that isn't replaced by a fillet curve, i.e. $\left\|x_i - x_j\right\| - d(\gamma_i) - d(\gamma_j)$ \\
    \hline
    \multicolumn{4}{l}{Arcs} \\
    $r$      & The radius of the circle made from executing $\kappa_{max}$ &
    $\zeta$  & Distinguishes clockwise and counterclockwise arcs           \\
    $\theta$ & An angular variable ranging from $[0,\gamma)$            \\
    \hline
    \multicolumn{4}{l}{B\'ezier fillets} \\
    $B_{n,i}(\tau)$ & A Berstein polynomial of degree $n$ and iteration $i$ &
    $\tau$          & A path parameterization index ranging from $0$ to $1$ \\
    $P_n(\tau)$     & A B\'ezier curve of degree $n$                        &
    $p_i$           & Control points for the B\'ezier curve                 \\
    $\nu_{1-4}$     & Constant scalars                                      &
    $h$,$g$,$k$ & Weights that are dependent on $d(\gamma)$                 \\
    \hline
    \multicolumn{4}{l}{Vectors and lines} \\
    $u_{ab}$                  & The unit vector formed from $(x_b-x_a)/\left\|x_b-x_a\right\|$ &
    $\overline{x_ax_b}$       & The line that intersects points $x_a$ and $x_b$     \\
    $\overrightarrow{x_ax_b}$ & The vector that goes from $x_a$ to $x_b$            &
    $x_{a,i}$                 & Vector $a$'s $i$th element                          \\
    \bottomrule
  \end{tabularx}
  }
\end{table*}

\subsection{Common Sampling-based Planning Procedures}
The literature on sample-based planning defines planning algorithms using a number of procedures.
We now define several generic procedures that can be found in \citep{Karaman2011} with notation changed to match the sequel.

\begin{procedure}[$T \leftarrow Initialize(x_{r})$]
  Returns an initialized tree with $x_{r} \in X$ as the root node and no edges, i.e. $T \leftarrow \{V,E\}$, $V \leftarrow \{x_{r}\}$, and $E \leftarrow \varnothing$.
\end{procedure}

\begin{procedure}[$x_{rand} \leftarrow Sample(X_{s})$]
  \label{procedure:sample}
  Returns a random state from the set $X_{s} \subset X$.
\end{procedure}

\begin{procedure}[$x_{nearest} \leftarrow Nearest(x_{rand},T)$]
  Finds the nearest vertex in the set $V$ to the state $x_{rand} \in X$ using the 2-norm as a distance.
\end{procedure}

\begin{procedure}[$X_{near} \leftarrow Near_{\rho,\alpha}(x_{rand},T)$]
  Finds the nearest $\alpha \in \mathbb{Z}_+$ vertices that are within a given radius\footnote{In this work $\rho$ is held constant, but many RRT* based algorithms vary $\rho$ \citep{Karaman2011}.}, $\rho \in \mathbb{R}_+$, of the point $x_{rand} \in X$. 
\end{procedure}
The constant $\alpha$ is used to prevent too many connections from being attempted in a given iteration \citep{Lavalle2006}.

\begin{procedure}[$x_{n} \leftarrow Steer_\eta(x,y)$]
  \label{procedure:steer}
  Returns a point that is within a predefined distance $\eta \in \mathbb{R}_+$ from $x \in X$ in the direction of $y \in X$, i.e.
  \begin{equation*}
    Steer_\eta(x,y) = argmin_{\{z \in X : \|z - x \| \leq \eta\}} \|z - y\|
  \end{equation*}
\end{procedure}
The Steer function prevents long edges from being added to RRT search trees.
This is important because it reduces the expected extension length at each iteration and likewise reduces the likelihood that a given iteration will fail to expand the search tree due to its edge being blocked by an obstacle \citep{lan2015}.

\begin{procedure}[$T \leftarrow InsertNode(x_{n},x_{p},T)$]
  Adds the node $x_{n} \in X_{free}$ to the tree with $x_{p} \in V$ as the new node's parent, i.e. $V \leftarrow V \cup \{x_{n} \}; E \leftarrow E \cup \{(x_{p},x_{n}) \}$.
\end{procedure}

\begin{procedure}[$X_{sol} \leftarrow Solution(x_v,T)$]
  Finds the path through $T$, $X_{sol} \subset V$, that leads from the root node to $x_v$.
\end{procedure}

\begin{procedure}[$X_{path} \leftarrow Path(x_{start},x_{end})$]
  Builds an ordered set of states that connect the state $x_{start} \in X$ to $x_{end} \in X$ without considering obstacles.
\end{procedure}
Note that $Solution$ is used to search the tree while $Path$ is used to search $X$ in an attempt to grow the tree.

\begin{procedure}[$bool \leftarrow CollisionFree(X_{path})$]
  Returns true if and only if $X_{path}$ is obstacle free, i.e. $X_{path} \subset X_{free}$.
\end{procedure}

\begin{procedure}[$x_{p} \leftarrow Parent(x_{c},T)$]
  Returns the parent node of $x_{c} \in V$ in the tree $T$, or $\varnothing$ if $x_{c}$ is the root node.
\end{procedure}

\begin{procedure}[$X_{children} \leftarrow Children(x_{p},T)$]
  Returns every node from the set $V$ in $T$ that has $x_{p}$ as its parent.
\end{procedure}

\begin{procedure}[$c_v \leftarrow Cost(x_{v},T)$]
  Returns the cost of $x_{v} \in V$.
  The cost of a vertex is defined as the path length traveled from $x_{r}$ to $x_{v}$ along the tree, i.e.
  \begin{equation*}
    Cost(x_{v},T) = c(Solution(x_{v},T))
  \end{equation*}
\end{procedure}

\begin{procedure}[$c_n \leftarrow CostToCome(x_{n},x_{p},T)$]
  Calculates the cost of $x_{n} \in X$ if it were connected to the tree through $x_{p} \in V$, returning an infinite cost if the path is not obstacle free.
  It is defined in Algorithm \ref{alg:cost_to_come}.
\end{procedure}
\begin{algorithm}[t]
  \caption{$c_n \leftarrow CostToCome(x_{n},x_{p},T)$}
  \label{alg:cost_to_come}
  \begin{algorithmic}[1]
    \State $X_{path} \leftarrow Path(x_{p},x_{n})$ \label{alg:cost_to_come:path}
    \If{$CollisionFree(X_{path})$} \label{alg:cost_to_come:collision_check}
      \State \Return $Cost(x_{p},T) + c(X_{path})$ \label{alg:cost_to_come:cost} \Comment{Path length calculation}
    \Else
      \State \Return $\infty$
    \EndIf
  \end{algorithmic}
\end{algorithm}

\subsection{RRT}
\label{section:rrt}
RRT quickly searches $X_{free}$ to find a feasible (not optimal), obstacle free solution and can be used with complex motion primitives while maintaining probabilistic completeness \citep{nrr}.
The RRT algorithm is composed of two main steps that are repeatedly performed until a solution is found.
The first is taking a biased sample from $X$.
The second is growing the search tree toward the random sample using the $Extend$ procedure.
These two procedures are now stated.

\begin{procedure}[$x_{rand} \leftarrow Biased\text{-}Sample(i, b_t, X_t)$]
  Returns a random point at iteration $i \in \mathbb{Z}_+$ given the sampling bias, $b_t \in \mathbb{Z}_+$, and the target set, $X_t$, as described in Algorithm \ref{alg:biased_sampling}.
\end{procedure}
\begin{algorithm}[tb]
  \caption{$x_{rand} \leftarrow Biased\text{-}Sample(i, b_t, X_t)$}
  \label{alg:biased_sampling}
  \begin{algorithmic}[1]
    \If{$i \divby b_{t}$}\Comment{Bias towards target set}
      \State $x_{rand} \leftarrow Sample(X_{t})$
    \Else \Comment{Otherwise sample a random point}
      \State $x_{rand} \leftarrow Sample(X)$
    \EndIf
    \State \Return $x_{rand}$
  \end{algorithmic}
\end{algorithm}
The sample is biased towards the target set by selecting the sample from $X_t$ every $b_t$ iterations.
$b_t$ is a design parameter affecting exploration and exploitation.
A small $b_t$ will attempt to connect the tree to the target set more frequently.

\begin{procedure}[$\{x_{n},x_{p},c_n\} \leftarrow Extend(x_{rand},T)$]
  Given a sample, $x_{rand} \in X$, and tree, $T = \{V,E\}$, the $Extend$ procedure finds the closest vertex to $x_{rand}$ that is already in $V$ and checks if a valid extension can be made from the tree towards $x_{rand}$, see Algorithm \ref{alg:extend}.
\end{procedure}
The $Extend$ procedure is illustrated in Figure \ref{fig:extend}.
Note that RRT does not make use of the extension cost $c_n$; it is included for use in RRT*.
\begin{algorithm}[t]
  \caption{$\{x_{n},x_{p},c_n\} \leftarrow Extend(x_{rand},T)$}
  \label{alg:extend}
  \begin{algorithmic}[1]
    \State $x_{nearest} \leftarrow Nearest(x_{rand},T)$ \label{alg:extend:nearest}
    \State $x_{n} \leftarrow Steer_\eta(x_{nearest},x_{rand})$ \label{alg:extend:steer} \Comment{Steer towards the nearest point}
    \State $c_{n} \leftarrow CostToCome(x_{n},x_{nearest},T)$ \Comment{Evaluate cost of resulting path}
    \If{$\infty \neq c_{n}$} \label{alg:extend:collision_check}
      \State \Return $\{x_{n},x_{nearest},c_{n}\}$ \label{alg:extend:return}
    \Else
      \State \Return $\{\varnothing,\varnothing,\infty\}$
    \EndIf
  \end{algorithmic}
\end{algorithm}
\tikzstyle{node}    =[circle,draw=black!100,thick]
\tikzstyle{new-node}=[circle,draw=red!100,  thick]
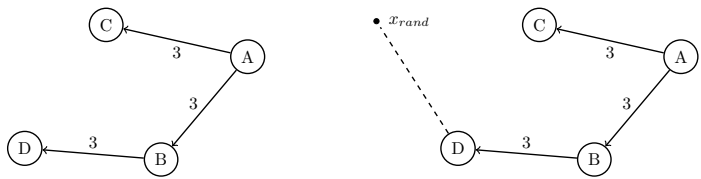
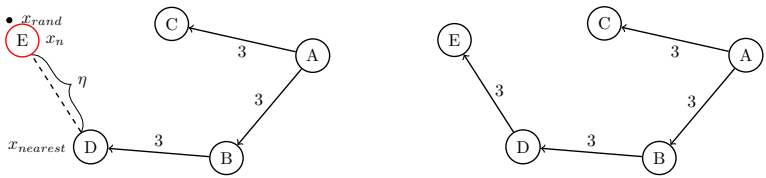
\begin{figure}[t]
  \centering
  \newcommand{\W}{0.8\linewidth}
  \begin{subfigure}[t]{0.47\linewidth}
    \centering
    \resizebox{\W}{!}{
    \begin{tikzpicture}
      \centering

      \node[node](A) at (0,0) {A};

      \node[node](C) at ($(A)!0.26!(-11,2.5)$) {C};
      \draw[->,thick](A) -- (C) node[midway,below] {$3$};

      \node[node](B) at ($(A)!0.35!(-5,-6)$) {B};
      \draw[->,thick](A) -- (B) node[midway,left,yshift=0.1cm] {$3$};

      \node[node](D) at ($(B)!0.38!(-9,-1.5)$) {D};
      \draw[->,thick](B) -- (D) node[midway,above,yshift=0cm] {$3$};

      \node[white](Rand) at ($(D)!0.33!(-9.5,6)$) {};
    \end{tikzpicture}
    }
    \caption{The tree before $x_{rand}$ is sampled and Algorithm \ref{alg:extend} starts.}
    \label{fig:extend:sub1}
  \end{subfigure}
  \begin{subfigure}[t]{0.47\linewidth}
    \centering
    \resizebox{\W}{!}{
    \begin{tikzpicture}
      \centering

      \node[node](A) at (0,0) {A};

      \node[node](C) at ($(A)!0.26!(-11,2.5)$) {C};
      \draw[->,thick](A) -- (C) node[midway,below] {$3$};

      \node[node](B) at ($(A)!0.35!(-5,-6)$) {B};
      \draw[->,thick](A) -- (B) node[midway,left,yshift=0.1cm] {$3$};

      \node[node](D) at ($(B)!0.38!(-9,-1.5)$) {D};
      \draw[->,thick](B) -- (D) node[midway,above,yshift=0cm] {$3$};

      \node[black,label=right:$x_{rand}$](Rand) at ($(D)!0.33!(-9.5,6)$) {};
      \filldraw[black,label=above:$x_{rand}$] (Rand) circle(1.5pt);
      \draw[dashed,thick](D) -- (Rand);
    \end{tikzpicture}
    }
    \caption{$x_{rand}$ is sampled and node D is found to be the closest node to $x_{rand}$.}
    \label{fig:extend:sub2}
  \end{subfigure}
  \begin{subfigure}[t]{0.47\linewidth}
    \centering
    \resizebox{\W}{!}{
    \begin{tikzpicture}
      \centering

      \node[node](A) at (0,0) {A};

      \node[node](C) at ($(A)!0.26!(-11,2.5)$) {C};
      \draw[->,thick](A) -- (C) node[midway,below] {$3$};

       \node[node](B) at ($(A)!0.35!(-5,-6)$) {B};
       \draw[->,thick](A) -- (B) node[midway,left,yshift=0.15cm] {$3$};

       \node[node,label=left:$x_{nearest}$](D) at ($(B)!0.38!(-9,-1.5)$) {D};
       \draw[->,thick](B) -- (D) node[midway,above,yshift=0cm] {$3$};

       \node[black,label=right:$x_{rand}$](Rand) at ($(D)!0.33!(-9.5,6)$) {};
       \filldraw[black,label=above:$x_{rand}$] (Rand) circle(1.5pt);

      \node[new-node,label=right:$x_{n}$](E) at ($(D)!0.84!(Rand)$) {E};
      \draw[dashed,thick](E) -- (Rand);
      \draw[dashed,thick](D) -- (E);
      \draw[decorate,decoration={brace,amplitude=10pt}]
        (E) -- (D) node [black,midway,xshift=15pt,yshift=7pt]
        {$\eta$};
    \end{tikzpicture}
    }
    \caption{$x_{rand}$ is steered toward node D resulting in the new node, node E.}
    \label{fig:extend:sub3}
  \end{subfigure}
  \begin{subfigure}[t]{0.47\linewidth}
    \centering
    \resizebox{\W}{!}{
    \begin{tikzpicture}
      \centering

      \node[node](A) at (0,0) {A};

       \node[node](C) at ($(A)!0.26!(-11,2.5)$) {C};
       \draw[->,thick](A) -- (C) node[midway,below] {$3$};

       \node[node](B) at ($(A)!0.35!(-5,-6)$) {B};
       \draw[->,thick](A) -- (B) node[midway,left,yshift=0.1cm] {$3$};

       \node[node](D) at ($(B)!0.38!(-9,-1.5)$) {D};
       \draw[->,thick](B) -- (D) node[midway,above,yshift=0cm] {$3$};

       \node[](Rand) at ($(D)!0.33!(-9.5,6)$) {};

       \node[node](E) at ($(D)!0.84!(Rand)$) {E};
       \draw[->,thick](D) -- (E) node[midway,right,yshift=0.05cm] {$3$};
     \end{tikzpicture}
     }
     \caption{After checking for obstacles, node E is added to the tree.}
    \label{fig:extend:sub4}
  \end{subfigure}
  \caption{An illustration of the $Extend$ procedure.}
  \label{fig:extend}
\end{figure}

\begin{algorithm}[t]
  \caption{$X_{sol} \leftarrow RRT(x_{r},X_{t},b_{t})$}
  \label{alg:rrt}
  \begin{algorithmic}[1]
    \State $T \leftarrow Initialize(x_{r})$
    \For{$i = 1,\cdots,\inf$}
      \State $x_{rand} \leftarrow Biased\text{-}Sample(i, b_t, X_t)$ \Comment{Biased configuration sampling}
      \State $\{x_{n},x_{p},c_n\} \leftarrow Extend(x_{rand},T)$ \label{alg:rrt:extend} \Comment{Extend tree towards new point}
      \If{$x_{n} \neq \varnothing$}
        \State $T \leftarrow InsertNode(x_{n},x_{p},T)$ \label{alg:rrt:insert} \Comment{Add point to tree}
        \If{$x_{n} \in X_{t}$} \label{alg:rrt:if_finished}
          \State \Return $Solution(x_{n},T)$ \label{alg:rrt:return} \Comment{Return solution if found}
        \EndIf
      \EndIf
    \EndFor
  \end{algorithmic}
\end{algorithm}

The RRT algorithm can now be described.
First, a random point is sampled from the configuration space.
If the tree can be extended, the new point is added to the tree.
If the new point is in the target set then RRT returns a solution, as shown in Algorithm \ref{alg:rrt}.
RRT is known to quickly find solutions for complex problems as it naturally explores unexplored areas of the state space, a property called the Voronoi property \citep{rrt_connect}.
When using the vertices of the tree to create a Voronoi diagram, unexplored regions correspond to larger Voronoi cells.
The probability that a Voronoi cell is sampled is proportional to the size of that cell.
Thus, the RRT tree naturally extends towards regions that have not yet been explored, avoiding problems with local minima and nonconvex obstacles.

\subsection{RRT*}
\label{section:rrt*}
RRT* is an extension of RRT that adds probabilistic guarantees for asymptotic optimality to the probabilistic completeness guarantees of RRT \citep{Karaman2011}.
RRT* does so by performing local optimizations on the edges in the tree whenever a new node is added.
As the number of iterations goes to infinity, the repeated local optimization transforms the tree into a set of globally optimal paths from the root node to every reachable point in the obstacle free configuration space.

RRT* includes two significant changes to RRT, both of which concern the neighborhood set of the node being added to the tree, i.e. $X_{near} = Near_{\rho,\alpha}(x_{n},T)$.
The first modification is replacing the $Extend$ procedure with $Extend^*$.
As illustrated in Figure \ref{fig:optimal_extend}, the $Extend^*$ procedure selects a parent from $V$ within a specified distance of the new point that minimizes the cost of the new node.

The second modification happens after $x_{n}$ is added to the tree in a new procedure called $Rewire$.
Each node in a local neighborhood is tested to see if its cost would be improved by going through the new node instead of its current parent node.
If so, the edges are changed so that $x_n$ becomes the node's new parent as illustrated in Figure \ref{fig:rewire}.
\begin{procedure}[$\{x_{n},x_{p}\} \leftarrow Extend^*(x_{rand},T)$]
  Given a tree, $T = \{V,E\}$, the $Extend^*$ procedure finds the best ``local'' connection for extending the tree in the direction of $x_{rand} \in X$.
  It returns a new point to be added to the tree, $x_n$, and the parent, $x_p \in V$, as defined in Algorithm \ref{alg:optimal_extend}.
\end{procedure}
\begin{algorithm}[t]
  \caption{$\{x_{n},x_{p}\} \leftarrow Extend^*\!(x_{rand},T)$ \label{def:optimal_extend}}
  \label{alg:optimal_extend}
  \begin{algorithmic}[1]
    \State $\{x_{n},x_{p},c_{min}\} \leftarrow Extend(x_{rand},T)$ \label{alg:optimal_extend:extend}
    \If{$x_{n} \neq \varnothing$} \label{alg:optimal_extend:collistion_free} \Comment{If an extension is possible}
      \State $X_{near} \leftarrow Near_{\rho,\alpha}(x_{n},T)$ \label{alg:optimal_extend:near}
      \ForAll{$x_{near} \in X_{near}$} \label{alg:optimal_extend:it_over_near} \Comment{Check for a lower cost connection}
        \State $c_{tmp} \leftarrow CostToCome(x_{n},x_{near},T)$ \label{alg:optimal_extend:cost_to_come}
        \If{$c_{tmp} < c_{min}$} \label{alg:optimal_extend:if_better}
          \State $x_{p} \leftarrow x_{near}$ \label{alg:optimal_extend:update_best_vertex}
          \State $c_{min} \leftarrow c_{tmp}$ \label{alg:optimal_extend:update_best_cost}
        \EndIf
      \EndFor
      \State \Return $\{x_{n},x_{p}\}$ \label{alg:optimal_extend:return}
    \EndIf
    \State \Return $\{\varnothing,\varnothing\}$;
  \end{algorithmic}
\end{algorithm}
\tikzstyle{node}     =[circle,draw=black!100,thick]
\tikzstyle{new-node} =[circle,draw=red!100,  thick]
\tikzstyle{near-node}=[circle,draw=green!100,thick]
\begin{figure}[t]
  \centering
  \newcommand{\W}{0.8\linewidth}
  \begin{subfigure}[t]{0.47\linewidth}
    \centering
    \resizebox{\W}{!}{
    \begin{tikzpicture}
      \centering

       \node[node](A) at (0,0) {A};

       \node[node](C) at ($(A)!0.26!(-11,2.5)$) {C};
       \draw[->,thick](A) -- (C) node[midway,below] {$3$};

       \node[node](B) at ($(A)!0.35!(-5,-6)$) {B};
       \draw[->,thick](A) -- (B) node[midway,left,yshift=0.1cm] {$3$};

       \node[node](D) at ($(B)!0.38!(-9,-1.5)$) {D};
       \draw[->,thick](B) -- (D) node[midway,above,yshift=0cm] {$3$};

       \node[](Rand) at ($(D)!0.33!(-9.5,6)$) {};

       \node[node](E) at ($(D)!0.84!(Rand)$) {E};
       \draw[->,thick](D) -- (E) node[midway,right,yshift=0.05cm] {$3$};
     \end{tikzpicture}
     }
     \caption{Tree after $Extend$ finishes.}
     \label{fig:optimal_extend:sub1}
  \end{subfigure}
  \begin{subfigure}[t]{0.47\linewidth}
    \centering
    \resizebox{\W}{!}{
    \begin{tikzpicture}
      \centering

      \node[node](A) at (0,0) {A};

      \node[near-node](C) at ($(A)!0.26!(-11,2.5)$) {C};
      \draw[->,thick](A) -- (C) node[midway,below] {$3$};

      \node[node](B) at ($(A)!0.35!(-5,-6)$) {B};
      \draw[->,thick](A) -- (B) node[midway,left,yshift=0.1cm] {$3$};

      \node[node,label=left:$x_{nearest}$](D) at ($(B)!0.38!(-9,-1.5)$) {D};
      \draw[->,thick](B) -- (D) node[midway,above,yshift=0cm] {$3$};

      \node[](Rand) at ($(D)!0.33!(-9.5,6)$) {};
      \node[new-node,label=above:$x_{n}$](E) at ($(D)!0.84!(Rand)$) {E};
      \draw[dashed,thick](D) -- (E) node[midway,left,yshift=-0.1cm] {$3$};

      \draw[color=blue] ([xshift=3.5cm]E) arc (0:16:3.5);
      \draw[color=blue] ([xshift=3.5cm]E) arc (0:-55:3.5);
      \node[](temp) at ([xshift=3.25cm,yshift=-1.65cm]E) {};
      \draw[decorate,decoration={brace,amplitude=10pt}]
        (E) -- (temp) node [black,midway,xshift=6pt,yshift=15pt]
        {$\rho$};
    \end{tikzpicture}
    }
    \caption{$x_{n}$'s neighborhood set is found to be node C.}
    \label{fig:optimal_extend:sub2}
  \end{subfigure}
  \begin{subfigure}[t]{0.47\linewidth}
    \centering
    \resizebox{\W}{!}{
    \begin{tikzpicture}
      \centering

      \node[node](A) at (0,0) {A};

      \node[near-node,label=below:$X_{near}$](C) at ($(A)!0.26!(-11,2.5)$) {C};
      \draw[->,thick](A) -- (C) node[midway,below] {$3$};

      \node[node](B) at ($(A)!0.35!(-5,-6)$) {B};
      \draw[->,thick](A) -- (B) node[midway,left,yshift=0.1cm] {$3$};

      \node[node,label=left:$x_{nearest}$](D) at ($(B)!0.38!(-9,-1.5)$) {D};
      \draw[->,thick](B) -- (D) node[midway,above,yshift=0cm] {$3$};

      \node[](Rand) at ($(D)!0.33!(-9.5,6)$) {};
      \node[new-node,label=above:$x_{n}$](E) at ($(D)!0.84!(Rand)$) {E};
      \draw[dashed,thick](D) -- (E) node[midway,left,yshift=-0.1cm] {$3$};
      \draw[dashed,thick](C) -- (E) node[midway,above,yshift=0cm] {$3.2$};
    \end{tikzpicture}
    }
    \caption{Connecting through nodes C and D result in total costs of~$6.2$ and $9$ respectively.}
    \label{fig:optimal_extend:sub3}
  \end{subfigure}
  \begin{subfigure}[t]{0.47\linewidth}
    \centering
    \resizebox{\W}{!}{
    \begin{tikzpicture}
      \centering

      \node[node](A) at (0,0) {A};

      \node[node](C) at ($(A)!0.26!(-11,2.5)$) {C};
      \draw[->,thick](A) -- (C) node[midway,below] {$3$};

      \node[node](B) at ($(A)!0.35!(-5,-6)$) {B};
      \draw[->,thick](A) -- (B) node[midway,left,yshift=0.1cm] {$3$};

      \node[node](D) at ($(B)!0.38!(-9,-1.5)$) {D};
      \draw[->,thick](B) -- (D) node[midway,above,yshift=0cm] {$3$};

      \node[](Rand) at ($(D)!0.33!(-9.5,6)$) {};
      \node[node](E) at ($(D)!0.84!(Rand)$) {E};
      \draw[->,thick](C) -- (E) node[midway,above,yshift=0cm] {$3.2$};
    \end{tikzpicture}
    }
    \caption{Because connecting through node C yields a lower total cost, node E is connected to node C.}
    \label{fig:optimal_extend:sub4}
  \end{subfigure}
  \caption{An illustration of the $Extend^*$ procedure.}
  \label{fig:optimal_extend}
\end{figure}
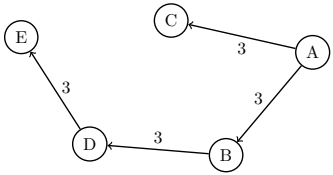
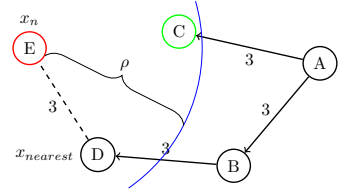
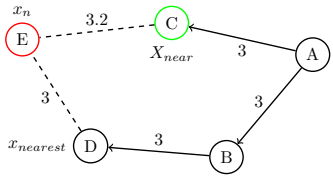
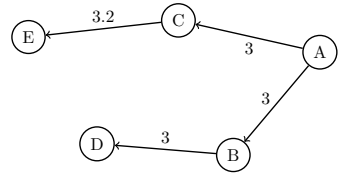

\begin{procedure}[$T \leftarrow Rewire(x_{n},X_{near},T)$]
  Given a tree, $T = \{V,E\}$, with node $x_n \in V$ and set $X_{near} \subset V$, $Rewire$ returns a tree with a modified edge set such that $x_n$ is made the  parent of elements in $X_{near}$ if it results in a lower cost for the elements of $X_{near}$.
  It is defined in Algorithm \ref{alg:rewire}.
\end{procedure}
\begin{algorithm}[t]
  \caption{$T \leftarrow Rewire(x_{n},X_{near},T)$}
  \label{alg:rewire}
  \begin{algorithmic}[1]
    \ForAll{$x_{near} \in X_{near}$} \label{alg:rewire:for}
      \State $c_{near} \leftarrow CostToCome(x_{near},x_{n},T)$
      \LineComment{Check $x_{near}$'s feasibility and that its cost is reduced}
      \If{$c_{near} < Cost(x_{near},T)$} \label{alg:rewire:if_optimal}
        \State $x_{p} \leftarrow Parent(x_{near},T)$ \label{alg:rewire:parent} \Comment{Rewire around $x_{near}$}
        \State $E \leftarrow \left(E \setminus \{x_{p},x_{near}\} \right) \cup \{x_{n},x_{near}\}$ \label{alg:rewire:rewire}
      \EndIf
    \EndFor
    \State \Return $T$
  \end{algorithmic}
\end{algorithm}
\tikzstyle{node}     =[circle,draw=black!100,thick]
\tikzstyle{new-node} =[circle,draw=red!100,  thick]
\tikzstyle{near-node}=[circle,draw=green!100,thick]
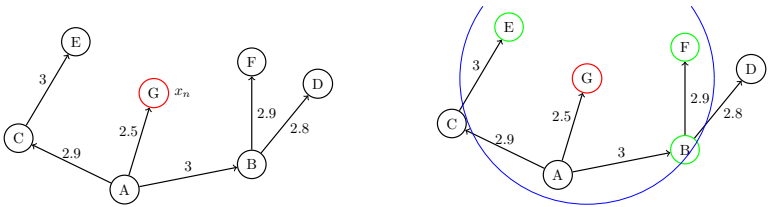
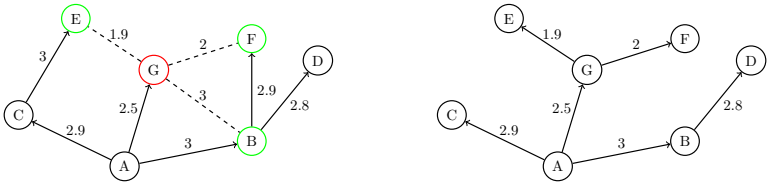
\begin{figure}[t]
  \centering
  \newcommand{\W}{0.8\linewidth}
  \begin{subfigure}[t]{0.47\linewidth}
    \centering
    \resizebox{\W}{!}{
    \begin{tikzpicture}
      \centering

      \node[node](A) at (0,0) {A};

      \node[node](B) at ($(A)!0.3!(10,2)$) {B};
      \draw[->,thick](A) -- (B) node[midway,above] {$3$};

      \node[node](C) at ($(A)!0.5!(-5,2.45)$) {C};
      \draw[->,thick](A) -- (C) node[midway,above] {$2.9$};

      \node[node](D) at ($(B)!0.26!(9,8)$) {D};
      \draw[->,thick](B) -- (D) node[midway,right,yshift=-0.1cm] {$2.8$};

      \node[node](E) at ($(C)!0.18!(5,14)$) {E};
      \draw[->,thick](C) -- (E) node[midway,above,xshift=-0.1cm] {$3$};

      \node[node](F) at ($(B)!0.26!(3,10)$) {F};
      \draw[->,thick](B) -- (F) node[midway,right] {$2.9$};

      \node[new-node,label=right:$x_{n}$](G) at ($(A)!0.23!(3,10)$) {G};
      \draw[->,thick](A) -- (G) node[midway,above,xshift=-0.25cm] {$2.5$};
    \end{tikzpicture}
    }
    \caption{The tree after $Extend^*$ adds node G to the tree.}
    \label{fig:rewire:sub1}
  \end{subfigure}
  \begin{subfigure}[t]{0.47\linewidth}
    \centering
    \resizebox{\W}{!}{
    \begin{tikzpicture}
      \centering

      \node[node](A) at (0,0) {A};

      \node[near-node](B) at ($(A)!0.3!(10,2)$) {B};
      \draw[->,thick](A) -- (B) node[midway,above] {$3$};

      \node[node](C) at ($(A)!0.5!(-5,2.45)$) {C};
      \draw[->,thick](A) -- (C) node[midway,above] {$2.9$};

      \node[node](D) at ($(B)!0.26!(9,8)$) {D};
      \draw[->,thick](B) -- (D) node[midway,right,yshift=-0.1cm] {$2.8$};

      \node[near-node](E) at ($(C)!0.18!(5,14)$) {E};
      \draw[->,thick](C) -- (E) node[midway,above,xshift=-0.1cm] {$3$};

      \node[near-node](F) at ($(B)!0.26!(3,10)$) {F};
      \draw[->,thick](B) -- (F) node[midway,right] {$2.9$};

      \node[new-node](G) at ($(A)!0.23!(3,10)$) {G};
      \draw[->,thick](A) -- (G) node[midway,above,xshift=-0.25cm] {$2.5$};
      \draw[color=blue] ([xshift=3cm]G) arc (0:-215:3);
      \draw[color=blue] ([xshift=3cm]G) arc (0:35:3);
    \end{tikzpicture}
    }
    \caption{Nodes B, E, and F are found to be G's neighborhood set.}
    \label{fig:rewire:sub2}
  \end{subfigure}
  \begin{subfigure}[t]{0.47\linewidth}
    \centering
    \resizebox{\W}{!}{
    \begin{tikzpicture}
      \centering

      \node[node](A) at (0,0) {A};

      \node[near-node](B) at ($(A)!0.3!(10,2)$) {B};
      \draw[->,thick](A) -- (B) node[midway,above] {$3$};

      \node[node](C) at ($(A)!0.5!(-5,2.45)$) {C};
      \draw[->,thick](A) -- (C) node[midway,above,xshift=0.1cm] {$2.9$};

      \node[node](D) at ($(B)!0.26!(9,8)$) {D};
      \draw[->,thick](B) -- (D) node[midway,right,yshift=-0.1cm] {$2.8$};

      \node[near-node](E) at ($(C)!0.18!(5,14)$) {E};
      \draw[->,thick](C) -- (E) node[midway,above,xshift=-0.1cm] {$3$};

      \node[near-node](F) at ($(B)!0.26!(3,10)$) {F};
      \draw[->,thick](B) -- (F) node[midway,right] {$2.9$};

      \node[new-node](G) at ($(A)!0.23!(3,10)$) {G};
      \draw[->,thick](A) -- (G) node[midway,above,xshift=-0.25cm] {$2.5$};

      \draw[dashed,thick](G) -- (E) node[midway,above,xshift=0.1cm] {$1.9$};
      \draw[dashed,thick](G) -- (F) node[midway,above] {$2$};
      \draw[dashed,thick](G) -- (B) node[midway,above] {$3$};
    \end{tikzpicture}
    }
    \caption{The original costs of nodes B, E, and F are $3$, $5.9$, and $5.9$ and their potential costs are $5.5$, $4.4$, and $4.5$ respectively.}
    \label{fig:rewire:sub3}
  \end{subfigure}
  \begin{subfigure}[t]{0.47\linewidth}
    \centering
    \resizebox{\W}{!}{
    \begin{tikzpicture}
      \centering

      \node[node](A) at (0,0) {A};

      \node[node](B) at ($(A)!0.3!(10,2)$) {B};
      \draw[->,thick](A) -- (B) node[midway,above] {$3$};

      \node[node](C) at ($(A)!0.5!(-5,2.45)$) {C};
      \draw[->,thick](A) -- (C) node[midway,above,xshift=0.1cm] {$2.9$};

      \node[node](D) at ($(B)!0.26!(9,8)$) {D};
      \draw[->,thick](B) -- (D) node[midway,right,yshift=-0.1cm] {$2.8$};

      \node[node](G) at ($(A)!0.23!(3,10)$) {G};
      \draw[->,thick](A) -- (G) node[midway,above,xshift=-0.25cm] {$2.5$};

      \node[node](E) at ($(C)!0.18!(5,14)$) {E};
      \draw[->,thick](G) -- (E) node[midway,above,xshift=0.1cm] {$1.9$};

      \node[node](F) at ($(B)!0.26!(3,10)$) {F};
      \draw[->,thick](G) -- (F) node[midway,above] {$2$};

    \end{tikzpicture}
    }
    \caption{Node B is not rewired because its cost would not be lowered by the operation, however, nodes E and F are rewired.}
    \label{fig:rewire:sub4}
  \end{subfigure}
  \caption{An illustration of the $Rewire$ procedure.}
  \label{fig:rewire}
\end{figure}

\begin{algorithm}[t]
  \caption{$X_{sol} \leftarrow RRT^*\!(x_{r},X_{t},b_{t},n)$}
  \label{alg:rrt*}
  \begin{algorithmic}[1]
    \State $T \leftarrow Initialize(x_{r})$
    \State $x_{best} \leftarrow \varnothing$ \label{alg:rrt*:init_x_best}
    \For{$i = 1,\ldots,n$} \label{alg:rrt*:for}
      \State $x_{rand} \leftarrow Biased\text{-}Sample(i, b_t, X_t)$ \label{alg:rrt*:end_choose_x_rand} \Comment{Biased configuration sampling}
      \State $\{x_{n},x_{p}\} \leftarrow Extend^*\!(x_{rand},T)$ \label{alg:rrt*:optimal_extend} \Comment{Extend tree towards new point}
      \If{$x_{n} \neq \varnothing$}
        \State  $T \leftarrow InsertNode(x_{n},x_{p},T)$ \Comment{Add point to tree}
        \State $T \leftarrow Rewire(x_{n},Near_{\rho,\alpha}(x_{n},T),T)$ \label{alg:rrt*:rewire} \Comment{Rewire edges around new point}
        \If{$x_{n} \in X_{t} \land \left(x_{best} = \varnothing \lor Cost(x_{n},T) < Cost(x_{best},T)\right)$}
          \State $x_{best} \leftarrow x_{n}$ \label{alg:rrt*:store_target} \Comment{Update best path found}
        \EndIf
      \EndIf
    \EndFor
    \If{$x_{best} \neq \varnothing$ \label{alg:rrt*:solution_found}} \Comment{Return best solution found}
      \State \Return $Solution(x_{best},T)$ \label{alg:rrt*:return}
    \Else
      \State \Return $\{\varnothing\}$
    \EndIf
  \end{algorithmic}
\end{algorithm}
\par
The RRT* algorithm is shown in Algorithm \ref{alg:rrt*}.
Note that RRT and RRT* are very similar with the main difference being the addition of the $Extend^*$ and $Rewire$ procedures.
Additionally, RRT* runs for a specific number of iterations, $n \in \mathbb{N}_+$, instead of stopping when the first solution is found.
\par
RRT* is both probabilistically complete and asymptotically optimal \citep{Karaman2011}.
However, RRT* tends to converge slowly because of the Voronoi property.
The Voronoi property helps RRT-based algorithms find valid solutions by encouraging exploration.
As the solution improves in RRT*, the Voronoi regions around the solution get smaller, resulting in a diminishing probability that a given sample will improve the solution \citep{Akgun2011}.

  \section{The Fillet Approach for Local Planning}
  \label{section:fillet}
  In many cases, planned paths must obey nonholonomic constraints, e.g. \citep{cui2018,lan2015,Lavalle2006}.
The $Extend$ and $Path$ procedures can be modified to use basic atomic motions that satisfy such constraints during RRT-based planning, enabling RRT to be used with virtually any set of dynamics.
RRT* variants, however, have no additional benefit if the underlying dynamics or primitive motions do not allow the connection of any two states using a single edge in open space \citep{Li2016}.
The reason being that the $Rewire$ procedure cannot be performed if the nodes in the neighborhood set cannot be exactly connected to each other.



This constraint is detrimental when planning with motion primitives that enforce dynamic path constraints, such as maximum curvature.
A common technique for considering maximum curvature constraints is to use Dubin's paths.
Dubin's paths connect orientated points with the shortest path while considering maximum curvature constraints \citep{Lavalle2006}.
The issue with using Dubin's paths in a sample-based path planner is that a poor choice in the orientation of the points along the solution can cause a significant increase in path length, as shown in Figure \ref{fig:fillet_vs_dubins}.
Furthermore, the orientation that minimizes overall path length changes as the solution converges to optimality.

Instead of attempting to connect two oriented points, fillets connect two line segments (defined with three unoriented points) with a curve transitioning smoothly between them, as shown in Figure \ref{fig:general_fillet}.
Without the addition of orientation, fillets naturally allow incremental improvements to the solution.
The result is a path that is continuous in position and orientation.
Additional path qualities may be achieved depending on the choice of fillet.
%
This section will introduce the general fillet concept and requirements for creating a path using fillets.
Two fillets are then defined, one using an arc and one using B\'ezier curves.
Section \ref{sec:fillet-rrt-star} utilizes these fillets as motion primitives in RRT-based algorithms.

\begin{figure}[t]
\centering
\begin{tikzpicture}[scale=0.3, rotate=90]
  \coordinate (x1) at (0,0);
  \coordinate (x2) at (0,10);
  \coordinate (x3) at (10,10);
  \coordinate (x4) at (10,15);
  \newcount\r;
  \r = 3;

  \node[white]() at (0,-0.01) {};

  \draw[black] (x1) -- (x2);
  \draw[black] (x2) -- (x3);
  \draw[black] (x3) -- (x4);

  \node[label={above:{\scriptsize $x_s$}},inner sep=1pt]() at (x1) {};
  \node[label={below:{\scriptsize $x_e$}},inner sep=1pt]() at (x4) {};

  \draw[red,thick] (x1) -- (x2);
  \draw[red,thick] (x2) arc (-0:-115:-\r) coordinate(arc1_end) {};
  \draw[red,thick] (x3) arc (90:60:-\r)  coordinate(arc2_end) {};
  \draw[red,thick] (arc1_end) -- (arc2_end);

  \draw[red,thick] (x3) arc (270:-25:-\r) coordinate(arc1_end) {};
  \draw[red,thick] (x4) arc (-0: -25:\r)  coordinate(arc2_end) {};
  \draw[red,thick] (arc1_end) -- (arc2_end);

  \draw[blue,thick] (x1) --
                    (0, 7) arc (0:-90:-\r) --
                    (7, 10) arc (90:180:-\r) --
                    (x4) {};

  \filldraw[black] (x1) circle(3pt);
  \filldraw[black] (x2) circle(3pt);
  \filldraw[black] (x3) circle(3pt);
  \filldraw[black] (x4) circle(3pt);

  \draw[black,->,line width=0.25mm] (x1) -- (0,1);
  \draw[black,->,line width=0.25mm] (x2) -- (0,11);
  \draw[black,->,line width=0.25mm] (x3) -- (11,10);
  \draw[black,->,line width=0.25mm] (x4) -- (10,16);

\end{tikzpicture}
  \caption{Given the set of nodes that start with $x_s$ and end with $x_e$ their respective orientations are denoted with arrows pointing from them.
           The blue path is the path that is made by the arc-fillet path generation and the red path is the path that is made by Dubin's paths.}
  \label{fig:fillet_vs_dubins}
\end{figure}
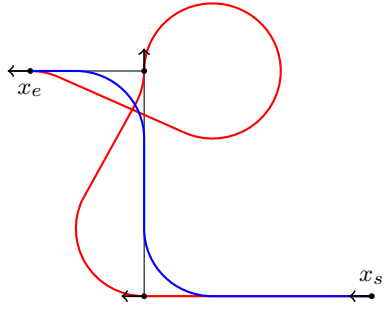

\subsection{General Fillets} \label{sec:general_fillets}
Given three input points $x_1,x_2,x_3 \in X$, a fillet connects $x_1$ to $x_3$ with the combination of two straight-line segments and a curve.
The line segments constitute portions of the lines $\overline{x_1 x_2}$ and $\overline{x_2x_3}$.
The curve intersects line $\overline{x_1x_2}$ at point $x_s$ and line $\overline{x_2x_3}$ at $x_e$.
The resulting fillet moves in a straight line from $x_1$ to $x_s$, along the curve from $x_s$ to $x_e$, and then along the straight line from $x_e$ to $x_3$, as depicted in Figure \ref{fig:general_fillet}.
The major differentiator between the different fillets is the definition of the curve portion, which affects the placement of $x_s$ and $x_e$.

This work assumes symmetric fillets, resulting in an equivalent distance between the node that the fillet is centered at and the two ends of the fillet curve, $x_s$ and $x_e$.
This distance is a function of the fillet curve type as well as the change in orientation between $\overrightarrow{x_1x_2}$ and $\overrightarrow{x_2x_3}$.

The position along the fillet can be described using a spatial index $s$.
Allow $s_0$ to be where the fillet meets $x_1$, $s_1$ to be the index where the fillet's curve begins, $s_2$ to be the index of where the fillet's curve ends, and $s_3$ to be the index of where the fillet reaches $x_3$.
Allow $\Psi(s)$ to be the fillet's curve such that $\Psi(0) = x_s$ and $\Psi(s_2-s_1) = x_e$.
Also, define the unit vector from $x_i$ to $x_j$ as $u_{ij}$ (see Table \ref{tab:notation}).
The position along the fillet can be written in a piecewise form as\
\begin{equation}
\label{eq:fillet_equation}
x(s) = \begin{cases}
	x_1 + s u_{12}      & s_0 \leq s \leq s_1 \\
	\Psi(s-s_1)         & s_1 <    s \leq s_2 \\
	x_e + (s-s_2)u_{23} & s_2 <    s \leq s_3
\end{cases}.
\end{equation}

Defining $\gamma_i \in [0,\pi)$ as the angle measured from $\overrightarrow{x_{i-1}x_i}$ to $\overrightarrow{x_i x_{i+1}}$, the fillet distance for $x_1,x_2,x_3$ can be defined as $d(\gamma_2) = \left\|x_s-x_2\right\| = \left\|x_e-x_2\right\|$.
Once $d(\gamma_2)$ is found, the start and end points of the curve can be written as
\begin{equation}
\label{eq:start_end_arc}
\begin{split}
x_s &= x_2 + d(\gamma_2) u_{21} \\
x_e &= x_2 + d(\gamma_2) u_{23}
\end{split}.
\end{equation}

\subsection{Fillet Paths}
A smooth path to a destination node can be created using a sequence of points where fillets are formed from point triplets and then combined, as shown in Figure \ref{fig:general_fillet}.
Without loss of generality, it is assumed that the path starts at node $1$ and moves to node $n$ using the sequence $x_1, x_2, ..., x_n$.
The path is thus made from $n$ nodes, using $n-2$ fillets to arrive at $x_n$.
There are two major concerns when formulating the path.
The first is the path continuity; not every sequence of points can be combined using fillets to create a continuous path.
The second major consideration is path length; the $Extend^*$ and $Rewire$ procedures depend upon path length for local optimizations.

\subsubsection*{Path Continuity}
The first key to using the fillets for planning purposes is to ensure that the path resulting from joining multiple fillets is continuous.
There are two conditions to ensure feasibility: one to ensure that the fillet curve ends before the final point in the fillet and one condition to ensure that all fillets end before the next one begins.
Assuming $x_i$ is the middle node, these conditions can be expressed as
\begin{equation}
  \label{equ:conditions}
    \begin{split}
      d(\gamma_i)                   &\leq \| x_i - x_{i+1} \| \\
      d(\gamma_{i-1}) + d(\gamma_i) &\leq \| x_i - x_{i-1} \|
    \end{split}.
\end{equation}

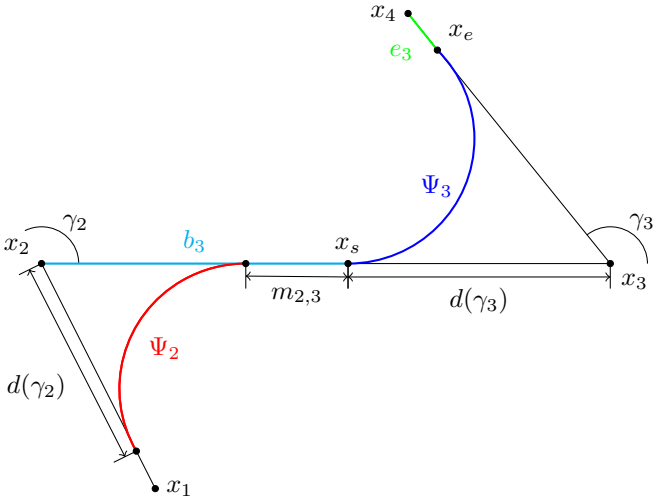
\begin{figure}[t]
\centering
\resizebox{0.75\linewidth}{!}{
\begin{tikzpicture}[scale=0.85]
  \newcount\r;
  \r = 2;

  \node[label={170:  $x_2$},inner sep=0pt](x2) at (0,0) {};
  \node[label={right:$x_1$},inner sep=0pt](x1) at ($(2,-4)! 0.1 !(x2)$) {};
  \node[label={-30:  $x_3$},inner sep=0pt](x3) at (9,0) {};
  \node[label={left: $x_4$},inner sep=0pt](x4) at ($(5,5)! 0.2 ! (x3)$) {};

  \draw[black](x1) -- (x2) -- (x3) -- (x4);

  \coordinate (arc0_start) at ($(2,-4)!0.25!(x2)$);
  \draw[red,thick] (arc0_start) arc (30:-90:-\r) coordinate(arc0_end) {} node[midway,label={-30:$\Psi_2$},inner sep=0pt] {};

  \node[label={above:$x_s$},inner sep=0pt](arc1_start) at ($(-1,0)!0.585!(x3)$) {};
  \draw[blue,thick,label={left:$\Psi_2$}] (arc1_start) arc (-90:45:\r) coordinate(arc1_end) {} node[midway,label={left:$\Psi_3$},inner sep=0pt] {};
  \node[label={30:$x_e$},inner sep=0pt]() at (arc1_end) {};

  \coordinate (temp) at ($(x1)!1.1!(x2)$);
  \draw pic[draw,angle radius=0.5cm,"$\gamma_2$" shift={(3mm,3mm)}] {angle=x3--x2--temp};

  \coordinate (temp) at ($(x2)!1.1!(x3)$);
  \draw pic[draw,angle radius=0.5cm,"$\gamma_3$" shift={(3mm,3mm)}] {angle=temp--x3--x4};

  \draw[black] (arc0_start) -- ++($(x2) !0.4cm! ($(x2)!{sin(-90)}!90:(x1)$)$) coordinate[midway] (brace_start);
  \draw[black] (x2)         -- ++($(x2) !0.4cm! ($(x2)!{sin(-90)}!90:(x1)$)$) coordinate[midway] (brace_end);
  \draw[<->,black] (brace_start) -- (brace_end) node[black,midway,label={-150:$d(\gamma_2)$},inner sep=0pt] {};

  \draw[black] (arc1_start) -- ++(0,-0.4) coordinate[midway] (brace_start);
  \draw[black] (x3)         -- ++(0,-0.4) coordinate[midway] (brace_end);
  \draw[<->,black] (brace_start) -- (brace_end) node[black,midway,label={below:$d(\gamma_3)$},inner sep=0pt] {};

  \draw[black] (arc0_end)   -- ++(0,-0.4) coordinate[midway] (brace_start);
  \draw[black] (arc1_start) -- ++(0,-0.4) coordinate[midway] (brace_end);
  \draw[<->,black] (brace_start) -- (brace_end) node[black,midway,label={below:$m_{2,3}$},inner sep=0pt] {};

  \draw[green,thick] (x4) -- (arc1_end) node[midway,label={-100:$e_3$},inner sep=0pt] {};
  \draw[cyan,thick] (arc1_start) -- (x2) node[midway,label={above:$b_3$},inner sep=0pt] {};

  \draw[red,thick] (arc0_start) arc (30:-90:-\r);

  \foreach \point in {arc1_start,arc1_end,arc0_end,arc0_start,x1,x2,x3,x4}
  {
    \filldraw[black] (\point) circle(1.4pt);
  }
\end{tikzpicture}\hfill
}
\caption{The fillet generated to connect $x_2$ and $x_4$ is shown in blue with the fillet that comes before it shown in red.
         Note that $d(\gamma_2) + d(\gamma_3) \leq \left\|x_2-x_3\right\|$, providing a continuous path.}
\label{fig:general_fillet}
\end{figure}


\subsubsection*{Length of Fillet Paths}
In RRT*'s $Extend^*$ and $Rewire$ procedures, local optimizations are made to the search tree to find the shortest path. These procedures are defined for straight-line paths where the path length can be calculated as the distance between nodes in the path. This is not the case for paths created from fillets. Thus, the path length to a particular node and the length relation to other nodes are now evaluated.

Each fillet consists of two line segments and a fillet curve. The length of a single fillet with $x_i$ as the middle node of the fillet can be expressed as
\begin{equation}
  \label{eq:fillet_definition}
	\F_i = b_{i} + \Psi_i + e_{i}
\end{equation}
where $b_{i}$, $\Psi_i$, and $e_{i}$ are the beginning component length, the curve length, and the end component length (see Figure \ref{fig:general_fillet}).
Given \eqref{equ:conditions}, the two straight-line lengths can be expressed as
\begin{equation}
  \label{eq:b_e}
  \begin{split}
    b_i &= \left\|x_{i-1} - x_{i}\right\| - d(\gamma_i) \\
    e_i &= \left\|x_{i+1} - x_i\right\|   - d(\gamma_i)
  \end{split}.
\end{equation}
This definition of fillet length allows for the expression of a recursive relationship for calculating the path length in the following lemma.
\begin{lemma}
  \label{lem:recursive_length}
  Assume an ordered sequence of nodes is used to create a path using fillets with \eqref{equ:conditions} satisfied for every intermediate node.
  Given a resulting path length of $c_i$ to arrive at node $x_i$, $i \geq 3$, the path length to arrive at node $x_{i+1}$ can be expressed as
  \begin{equation}
    \label{eq:recursive_fillet_length}
    c_{i+1} = c_i + \F_i - \left\|x_i - x_{i-1}\right\|.
  \end{equation}
\end{lemma}
\begin{proof}
  The path length to node $x_i$ along a sequence of curves and lines can be written as a summation of individual parts.
  The path to $x_i$ contains $i-2$ curves of length $\Psi_2$ through $\Psi_{i-1}$.
  Let $m_{j,j+1}$ be the length of the straight-line segment that connects curve $j$ to curve $j+1$, i.e.
  \begin{equation}
    \label{eq:m}
    m_{j,j+1} = \left\|x_j - x_{j+1}\right\| - d(\gamma_j) - d(\gamma_{j+1}),
  \end{equation}
  which is positive assuming \eqref{equ:conditions} is satisfied for all fillets.
  The length of the path to arrive at $x_i$ is
  \begin{equation}
    \label{eq:summed_length}
    c_i = b_2 + \sum_{k=2}^{i-1}\Psi_k + \sum_{k=2}^{i-2}m_{k,k+1} + e_{i-1}.
  \end{equation}
  Note that the path connecting to node $x_{i+1}$ could be expressed similarly with a length of
  \begin{equation*}
    c_{i+1} = b_2 + \sum_{k=2}^{i}\Psi_k + \sum_{k=2}^{i-1}m_{k,k+1} + e_{i},
  \end{equation*}
  which can be written in terms of $c_i$ as
  \begin{equation}
    \label{eq:c_i+1}
    c_{i+1} = c_i + \Psi_i + m_{i-1,i} + e_{i} - e_{i-1}.
  \end{equation}
  Given \eqref{eq:b_e} and \eqref{eq:m}, \eqref{eq:c_i+1} becomes
  \begin{equation}
    \label{eq:c_plus_1}
    c_{i+1} = c_i + b_i + \Psi_i + e_{i} - \left\|x_i - x_{i-1}\right\|.
  \end{equation}
  Given the definition of $\F_i$ in \eqref{eq:fillet_definition}, \eqref{eq:c_plus_1} simplifies to \eqref{eq:recursive_fillet_length}.
\end{proof}
A few properties can now be stated using the recursive relationship in Lemma \ref{lem:recursive_length}, beginning with the relationship between the path length to a node and the path length to one of its descendants.
\begin{corollary}
  \label{cor:no_descendant}
  The path length to a node is not dependent upon the choice of any nodes that come after it.
\end{corollary}
\begin{proof}
  This can be seen by examining the summation form of the path length in \eqref{eq:summed_length} and noting that none of the variables depend upon any node after node $x_i$.
\end{proof}

The recursive path length calculation using fillets depends on the parent as well as the grandparent node.
This is different from the straight-line motion primitive where the path length to a node can be calculated using knowledge of solely the parent node.
This leads to the following lemma about path length, which has significant implications for rewiring a tree connected with fillets.
\begin{lemma}
  \label{lem:rewiring_tests}
  Given a node $x_i$ with a child node $x_{i+1}$ and multiple possible parent nodes, choosing the parent for $x_i$ to minimize $c_i$ \underline{\textit{will not}} necessarily result in the smallest possible value for $c_{i+1}$.
\end{lemma}
\begin{proof}
  The proof is given through a simple example where the shorter path to $x_i$ results in a longer path to $x_{i+1}$.
  Consider Figure \ref{fig:connection_alternatives}.
  Let the path from $x_r$ to $x_6$ have the same length as the path from $x_r$ to $x_1$ and the same be true of $x_5$ and $x_2$, i.e.
  \begin{equation*}
    \begin{split}
      Cost\left(x_1,T\right) = Cost\left(x_6,T\right) \\
      Cost\left(x_2,T\right) = Cost\left(x_5,T\right)
    \end{split}.
  \end{equation*}
  Path A, shown in red, is the shortest path to $x_3$ with a path length of $1.9$.
  Path A also results in a path length to $x_4$ of $2.9$.
  On path B, shown in blue, let the fillet that connects $x_5$ to $x_4$ be an arc-fillet with $r = 0.5$, $d(\pi/2) = 1/2$, and a resulting arc length of $\pi/4$.
  Starting at $x_6$ and following path B results in a path length to $x_3$ of $2$, which is greater than the path length to get to $x_3$ on path A.
  To minimize the path length to $x_3$ path A is chosen.
  However, the length of the path from $x_6$ to $x_4$ is $1 + 1/2 + \pi/4 + 1/2 \approx 2.785$.
  Thus, minimizing path length to $x_3$ does not minimize path length to $x_4$.
\end{proof}
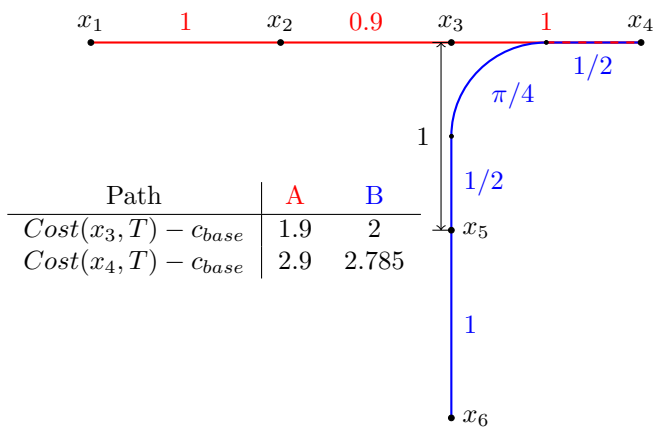
\begin{figure}[t]
  \centering
    \begin{tikzpicture}[scale=2.5]
      \node[label={above:$x_1$},inner sep=0pt](x1) at (1,3) {};
      \node[label={above:$x_2$},inner sep=0pt](x2) at (2,3) {};
      \node[label={above:$x_3$},inner sep=0pt](x3) at (2.9,3) {};
      \node[label={above:$x_4$},inner sep=0pt](x4) at (3.9,3) {};
      \node[label={right:$x_5$},inner sep=0pt](x5) at (2.9,2) {};
      \node[label={right:$x_6$},inner sep=0pt](x6) at (2.9,1) {};

      \draw[red,thick] (x1) -- (x2) node[midway,inner sep=0pt,label={above:$1$}] {}
                            -- (x3) node[midway,inner sep=0pt,label={above:$0.9$}] {}
                            -- (x4) node[midway,inner sep=0pt,label={above:$1$}] {};

      \draw[blue,thick] (x6) -- (x5)          node[midway,inner sep=0pt,label={right:$1$}] {}
                             -- ++(0,0.5)     node[midway,inner sep=0pt,label={right:$1/2$}] {} coordinate(xs) {}
                             arc (180:90:0.5) node[midway,inner sep=0pt,label={-45:$\pi/4$}] {} coordinate(xe) {}
                             -- (x4)          node[midway,inner sep=0pt,label={below:$1/2$}] {};

      \draw[red,dashed,thick] (xe) -- (x4);

      \filldraw[black] (x1) circle(0.015);
      \filldraw[black] (x2) circle(0.015);
      \filldraw[black] (x3) circle(0.015);
      \filldraw[black] (x4) circle(0.015);
      \filldraw[black] (x5) circle(0.015);
      \filldraw[black] (x6) circle(0.015);
      \filldraw[black] (xs) circle(0.01);
      \filldraw[black] (xe) circle(0.01);

      \draw[black] (x3) -- ++(-0.1,0) coordinate[midway] (brace_start);
      \draw[black] (x5) -- ++(-0.1,0) coordinate[midway] (brace_end);
      \draw[<->,black] (brace_start) -- (brace_end) node[black,midway,label={left:$1$},inner sep=0pt] {};

      \node at (1.65,2)
        {
          \begin{tabular}{c | c c}
            Path                     & \textcolor{red}{A} & \textcolor{blue}{B} \\
            \hline
            $Cost(x_3,T) - c_{base}$ & $1.9$              & $2$ \\
            $Cost(x_4,T) - c_{base}$ & $2.9$              & $2.785$
          \end{tabular}
        };
    \end{tikzpicture}\hfill
  \caption{Let $c_{base} = Cost(x_1,T) = Cost(x_6,T)$ and $Cost(x_2,T) = Cost(x_5,T)$.
           Path A, shown in red, yields a shorter path length for $x_3$ than path B, shown in blue.
           However, path B yields a shorter path length to $x_4$ than path A.}
  \label{fig:connection_alternatives}
\end{figure}
Therefore, in a $Rewire$ procedure, it is not sufficient to solely check the path to the node being rewired, but the path lengths to descendants must also be evaluated. As the tree grows larger, checking all descendants would be overly cumbersome. The following corollary establishes that the only descendants that need to be checked are the children nodes.
\begin{corollary}
  \label{cor:unchanged}
  If a node's parent and grandparent remain unchanged, but the tree is rewired such that the cost of the node's parent is lowered, then the cost of the node will be lowered by the same amount as its parent.
\end{corollary}
\begin{proof}
  If a node's parent and grandparent remain the same then an examination of \eqref{eq:recursive_fillet_length} shows that the only portion of its path length that will change is the length to the parent's node, $c_i$, since both $\F_i$ and $\left\|x_i - x_{i-1}\right\|$ remain unchanged.
  Thus, the same change in cost seen by a parent will be reflected in the cost of the node in question if the grandparent node remains unchanged.
\end{proof}


\subsection{The Arc Fillet}
\label{ssec:arc_fillet}
The arc-fillet formulates the curve portion of the fillet using a circle of radius $r \in \mathbb{R}_+$ that is tangential to the two line segments at $x_s$ and $x_e$, see Figure \ref{fig:arc_fillet}.
To respect curvature constraints, the radius can be chosen as the inverse of the maximum curvature, i.e. $r = 1/\kappa_{max}$.
The distance between the curve intersection points and the intermediary point of the fillet is expressed in the following Lemma:

\begin{lemma}
  The arc-fillet distance, $d(\gamma)$, for a circular curve of radius $r$ is
  \begin{equation}
  \label{eq:d_gamma_arc}
  d(\gamma) = \frac{r\bigl(1 - \cos(\gamma)\bigr)}{\sin(\gamma)}.
  \end{equation}
\end{lemma}
\begin{proof}
  Assume that the coordinate frame $f$ is defined such that $x_s$ is at the origin and the $x$-axis is pointing along $\overrightarrow{x_1x_2}$.
  Using the pre-subscript $f$ to denote a point expressed in frame $f$ then ${_fx_s} = \begin{bmatrix} 0 & 0 \end{bmatrix}^T$.
  As $d(\gamma)$ represents the distance along $\overline{x_1x_2}$ from $x_s$ to $x_2$, the value of $x_2$ in frame $f$ is
  \begin{equation}
    \label{eq:f_x_2}
    {_fx_2} = \begin{bmatrix}
      d(\gamma) \\
      0
    \end{bmatrix}.
  \end{equation}
  As the orientation of the path must be tangential to the circle at $x_s$ for continuity, the center point of the circle will lie upon the $y$-axis at a distance $r$ from $x_s$, i.e.,
  \begin{equation*}
    {_fx_c} = \begin{bmatrix}
      0 \\
      r \zeta
    \end{bmatrix}\mbox{, } \zeta \in \{-1, 1\},
  \end{equation*}
  where $\zeta = 1$ corresponds to a counter-clockwise arc and $\zeta =-1$ to a clockwise arc.
  Parameterizing the arc by its tangent angle, the arc can be expressed in frame $f$ as
  \begin{equation*}
    {_f\Psi_{arc}}(\theta) =
      r
      \begin{bmatrix}
        \sin(\theta) \\
        \zeta \left(1 - \cos(\theta)\right)
      \end{bmatrix}
  \end{equation*}
  where ${_fx_s} = {_f\Psi_{arc}}(0)$.
  The parameter $\theta \in [0,\gamma)$ represents the orientation of the path in frame $f$.
  The tangent of ${_f\Psi_{arc}}\left(\gamma\right)$ should be parallel to $\overline{x_2x_3}$.
  The task is to solve for $d(\gamma)$ such that ${_fx_e} = {_f\Psi_{arc}}(\gamma)$.

  The unit vector from ${_fx_3}$ to ${_fx_2}$ can be written in terms of $\gamma$ as $-\begin{bmatrix} \cos(\gamma) & \sin(\gamma) \end{bmatrix}^T$.
  The point ${_fx_2}$ can be expressed as a combination of this unit vector and ${_fx_e}$ as
  \begin{equation}
  \label{eq:x_2_version2}
  \begin{split}
    {_fx_2} &= {_fx_e} - d(\gamma)
      \begin{bmatrix}
        \cos(\gamma) \\
        \sin(\gamma)
      \end{bmatrix} \\
    &=
      r
      \begin{bmatrix}
        \sin(\gamma) \\
        \zeta \left(1-\cos(\gamma)\right)
      \end{bmatrix}
    - d(\gamma)
      \begin{bmatrix}
        \cos(\gamma) \\
        \sin(\gamma)
      \end{bmatrix}
  \end{split}
  \end{equation}
  Equation \eqref{eq:d_gamma_arc} is found by setting \eqref{eq:f_x_2} equal to \eqref{eq:x_2_version2} and solving for $d(\gamma)$.
\end{proof}
\begin{remark}
  It is important to note the singularities of $d(\gamma)$ and their corresponding significance.
  A change of direction of $\gamma=0$ corresponds to executing a straight line.
  A value of $\gamma = \pm \pi$ would correspond to reversing direction.
\end{remark}

One advantage of the arc-fillet is that the fillet curve can be directly expressed using the spatial index $s$.
The orientation of $\overrightarrow{x_1x_2}$ can be written as
$$
\psi = atan2(u_{12, 2}, u_{12,1}).
$$
The point along the fillet curve can then be expressed in the inertial frame as
\begin{align}
  \label{eq:arc_equation}
  \Psi_{arc}(s) &= R(\psi) \cdot {_f\Psi_{arc}}\left(\frac{s}{r}\right) + x_s\mbox{,} &
  R(\psi) &=
    \begin{bmatrix}
      \cos(\psi) & -\sin(\psi) \\
      \sin(\psi) & \cos(\psi)
    \end{bmatrix}\mbox{,}
\end{align}
where $\theta$ has been replaced with $\frac{s}{r}$ as the arc-length of a circle is $s = \theta r$.
Note that $\Psi(0) = x_s$ and $\Psi(\gamma r) = x_e$, as expected.

Finally, the spatial switching indices from \eqref{eq:fillet_equation} are
\begin{equation*}
\begin{split}
	s_1 &= s_0 + \left\|x_s - x_1\right\| \\
	s_2 &= s_1 + \gamma r \\
	s_3 &= s_2 + \left\|x_3 - x_e\right\|
\end{split}.
\end{equation*}

\tikzsetnextfilename{arc_fillet}
\begin{figure}[t]
  \centering
    \begin{tikzpicture}[scale=0.7]
      \node[label={above:$x_1$},inner sep=0pt](x1) at (1.5,  0)  {};
      \node[label={-30:  $x_2$},inner sep=0pt](x2) at (10, 0)  {};
      \node[label={left: $x_3$},inner sep=0pt](x3) at ($(4,7)!0.15!(x2)$) {};
      \draw[black](x1) -- (x2);
      \draw[black](x2) -- (x3);

      \coordinate (arc_start) at ($(x2)!0.65!(0,0)$);
      \draw[red,thick] (arc_start) arc (-90:36:3cm) coordinate(arc_end) {};

      \node[label={80:$x_s$},inner sep=0pt]() at (arc_start) {};

      \node[label={30:$x_e$},inner sep=0pt]() at (arc_end) {};

      \coordinate (xc) at ($(arc_start) !3cm! ($(arc_start)!{sin(-90)}!90:(x1)$)$);

      \node[label={above:$x_c$},inner sep=0pt] at (xc) {};

      \coordinate (temp) at (500,0);
      \draw pic[draw,angle radius=0.5cm,"$\gamma$" shift={(3mm,3mm)}] {angle=temp--x2--x3};

      \draw[thick,cyan]  (x1) -- (arc_start) node[midway,inner sep=0pt,label={below:$b_2$}] {};
      \draw[thick,green] (x3) -- (arc_end)   node[midway,inner sep=0pt,label={30:   $e_2$}] {};

      \draw[black] (arc_start) -- ++(-0.4cm,0) coordinate[midway] (brace_start);
      \draw[black] (xc)        -- ++(-0.4cm,0) coordinate[midway] (brace_end);
      \draw[<->,black] (brace_start) -- (brace_end) node[black,midway,label={left:$r$},inner sep=0pt] {};

      \draw[black] (arc_start) -- ++(0,-0.4cm) coordinate[midway] (brace_start);
      \draw[black] (x2)        -- ++(0,-0.4cm) coordinate[midway] (brace_end);
      \draw[<->,black] (brace_start) -- (brace_end) node[black,midway,label={below:$d(\gamma)$},inner sep=0pt] {};

      \filldraw[black] (x1)        circle(1.5pt);
      \filldraw[black] (x2)        circle(1.5pt);
      \filldraw[black] (x3)        circle(1.5pt);
      \filldraw[black] (xc)        circle(1.5pt);
      \filldraw[black] (arc_end)   circle(1.5pt);
      \filldraw[black] (arc_start) circle(1.5pt);
    \end{tikzpicture}
  \caption{The arc-fillet generated to connect $x_1$ and $x_3$.}
  \label{fig:arc_fillet}
\end{figure}
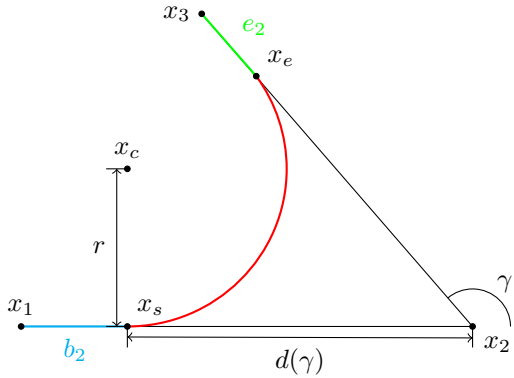

The development of the arc-fillet enables a definition of a new procedure for creating a path between three points:
\begin{procedure}[$X_{fillet} \leftarrow ArcFillet\left(x_0,x_1,x_2,x_3\right)$]
  The $ArcFillet$ procedure uses \eqref{eq:start_end_arc}, \eqref{eq:d_gamma_arc}, and \eqref{eq:arc_equation} to make an arc that connects $\overline{x_1x_2}$ to $\overline{x_2x_3}$.
  Feasibility checks are made with respect to the arc that connects $\overline{x_0x_1}$ to $\overline{x_1x_2}$ using \eqref{eq:d_gamma_arc} to define $d(\gamma)$ in \eqref{equ:conditions}.
  If the checks fail the null path is returned.
\end{procedure}

\begin{remark}
  The arc-fillet is made from straight-line segments and arcs assuming forward motion as is a Dubin's path.
  As a result, a path made using arc-fillets will have the same continuity properties as a Dubin's path, i.e. $\cont{1}$ in position and orientation with curvature assuming instantaneous changes between zero and maximum curvature \citep{Lavalle2006}.
\end{remark}

\subsection{The B\'ezier Fillet}
\label{section:bezier_curve_gen}
The arc-fillet assumes an instantaneous change in curvature which may be inappropriate for some scenarios that require higher levels of smoothness.
In this section, B\'ezier curves are used to generate the fillet, resulting in paths that are $\cont{2}$ continuous in position and orientation, and $\cont{1}$ continuous in curvature.
A detailed description of the curve is left to \citep{Yang2014} and the references therein.
The definition of the curve is shown as follows for the sake of completeness.

A B\'ezier curve connects two points and is defined as,
\begin{equation}
  \label{equ:bezier_curve}
  P_n(\tau) = \sum^{n}_{i=0} p_i B_{n,i}(\tau),
\end{equation}
where $p_i \in \mathbb{R}^2$ are control points, $n \in \mathbb{Z}_+$ is the degree of the polynomial, $\tau$ is a path parameterization index such that $0 \leq \tau \leq 1$, $P_n(0) = p_0$, and $P_n(1) = p_n$.
Note that there is no direct relationship between changes in $\tau$ and path length \citep{Gravesen1997}.
The functions $B_{n,i}(\tau)$ are Berstein polynomials defined as
\begin{equation}
  \label{equ:berstein_polynomials}
  B_{n,i}(\tau) = \binom{i}{n} \tau^i (1 - \tau)^{n-i}.
\end{equation}

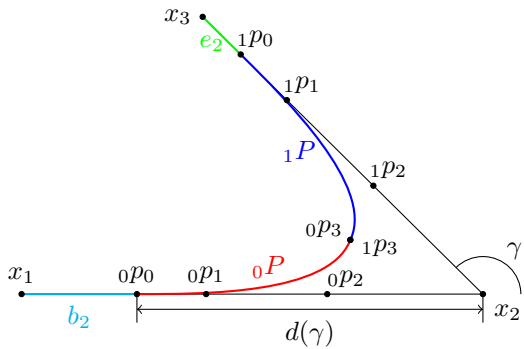
\begin{figure}[t]
\centering
\begin{tikzpicture}[scale=0.076]
  \newcount\Xonex;
  \newcount\Xoney;
  \newcount\Xtwox;
  \newcount\Xtwoy;
  \newcount\Xthreex;
  \newcount\Xthreey;
  \newcount\Bzerox;
  \newcount\Bzeroy;
  \newcount\Bonex;
  \newcount\Boney;
  \newcount\Btwox;
  \newcount\Btwoy;
  \newcount\Bthreex;
  \newcount\Bthreey;
  \newcount\Ezerox;
  \newcount\Ezeroy;
  \newcount\Eonex;
  \newcount\Eoney;
  \newcount\Etwox;
  \newcount\Etwoy;
  \newcount\Ethreex;
  \newcount\Ethreey;


  \Xonex = 20;
  \Xoney = 0;
  \Xtwox = 100;
  \Xtwoy = 0;
  \Xthreex = 40;
  \Xthreey = 60;
  \Bzerox = 40;
  \Bzeroy = 0;
  \Bonex = 52.0366;
  \Boney = 0;
  \Btwox = 73;
  \Btwoy = 0;
  \Bthreex = 77;
  \Bthreey = 10;
  \Ezerox = 58;
  \Ezeroy = 42;
  \Eonex = 66;
  \Eoney = 34;
  \Etwox = 81;
  \Etwoy = 19;
  \Ethreex = \Bthreex;
  \Ethreey = \Bthreey;

  \node[label={above:$x_1$},inner sep=0pt](x1) at (\Xonex,\Xoney) {};
  \node[label={-30:  $x_2$},inner sep=0pt](x2) at (\Xtwox,\Xtwoy) {};
  \node[label={left: $x_3$},inner sep=0pt](x3) at ($(x2)!0.81!(\Xthreex,\Xthreey)$) {};

  \node[label={above:$_0p_0$},inner sep=0pt](b0) at (\Bzerox,\Bzeroy) {};
  \node[inner sep=0pt](e0) at (\Ezerox,\Ezeroy) {};
  \node[label={above:$_1p_0$},xshift=0.2cm,yshift=-0.15cm]() at (e0) {};
  \node[label={above:$_0p_1$},inner sep=0pt](b1) at (\Bonex,\Boney) {};
  \node[inner sep=0pt](e1) at (\Eonex,\Eoney) {};
  \node[label={above:$_1p_1$},xshift=0.2cm,yshift=-0.15cm]() at (e1) {};
  \node[inner sep=0pt](b2) at (\Btwox,\Btwoy) {};
  \node[label={above:$_0p_2$},xshift=0.25cm,yshift=-0.15cm]() at (b2) {};
  \node[inner sep=0pt](e2) at (\Etwox,\Etwoy) {};
  \node[label={above:$_1p_2$},xshift=0.2cm,yshift=-0.15cm]() at (e2) {};
  \node[inner sep=0pt](b3) at ($(b2)!0.5!(e2)$) {};
  \node[label={left:$_0p_3$},xshift=0.15cm,yshift=0.15cm]() at (b3) {};
  \node[inner sep=0pt](e3) at (b3) {};
  \node[label={right:$_1p_3$},xshift=-0.1cm,yshift=-0.1cm]() at (b3) {};

  \draw[black](x1) -- (x2);
  \draw[black](x3) -- (x2);

  \draw[thick,red]  (b0) .. controls (b1) and (b2) .. (b3) {} node[midway,above,xshift=0.1cm] {$_0P$};
  \draw[thick,blue] (e0) .. controls (e1) and (e2) .. (e3) {} node[midway,left,xshift=0.05cm,yshift=-0.1cm] {$_1P$};
  \draw[thick,cyan] (x1) -- (b0);
  \node[label={below:\textcolor{cyan}{$b_2$}},xshift=0.0cm,yshift=0.1cm]() at ($(x1)!0.5!(b0)$) {};
  \draw[thick,green] (x3) -- (e0) node[midway,below,xshift=-0.13cm,yshift=0.13cm] {$e_2$};

  \filldraw[black] (x1) circle(13pt);
  \filldraw[black] (x2) circle(13pt);
  \filldraw[black] (x3) circle(13pt);

  \filldraw[black] (b0) circle(13pt);
  \filldraw[black] (b1) circle(13pt);
  \filldraw[black] (b2) circle(13pt);
  \filldraw[black] (b3) circle(13pt);
  \filldraw[black] (e0) circle(13pt);
  \filldraw[black] (e1) circle(13pt);
  \filldraw[black] (e2) circle(13pt);
  \filldraw[black] (e3) circle(13pt);

  \coordinate (temp) at (500,0);

  \draw pic[draw,angle radius=0.5cm,"$\gamma$" shift={(3mm,3mm)}] {angle=temp--x2--x3};

  \draw[black] (b0) -- ++(0,-5cm) coordinate[midway] (brace_start);
  \draw[black] (x2) -- ++(0,-5cm) coordinate[midway] (brace_end);
  \draw[<->,black] (brace_start) -- (brace_end) node[black,midway,label={below:$d(\gamma)$},inner sep=0pt] {};
\end{tikzpicture}\hfill
\caption{The fillet generated to connect $x_1$ and $x_3$.
         Note that sets $_0p_i$ and $_1p_i$ are control points that replace $p_i$ in \eqref{equ:bezier_curve} and in analogy to Figure \ref{fig:arc_fillet}, $_0p_0 = x_s$ and $_1p_0 = x_e$.}
\label{fig:bezier_curve}
\end{figure}

\cite{Yang2014} combines two cubic B\'{e}zier curves to generate the curve of the fillet.
In Figure \ref{fig:bezier_curve}, the curves are denoted as $_0P$ and $_1P$ with the connecting lines denoted as $b_2$ and $e_2$.
The fillet distance can now be stated.

\begin{lemma}
  The distance between the intermediary point, $x_2$, and the curve's start and end points can be expressed as
  \begin{equation}
    \label{equ:d_gamma}
    d(\gamma) = \frac{\nu_4 \sin\left(\frac{\gamma}{2}\right)}{\kappa_{max} \cos^2\left(\frac{\gamma}{2}\right)},
  \end{equation}
  where
  \begin{equation*}
  \begin{tabular}{c c c c}
    $\nu_1 = 7.2364$ & $\nu_2 = \frac{2}{5}\left(\sqrt{6} - 1\right)$ & $\nu_3 = \frac{\nu_2 + 4}{\nu_1 + 6}$ & $\nu_4 = \frac{\left(\nu_2 + 4\right)^2}{54 \nu_3}$
  \end{tabular}.
  \end{equation*}
\end{lemma}
\begin{proof}
See Section 2.1 of \citep{Yang2014}.
\end{proof}

To define the fillet curve, four control points are needed for each of the two B\'{e}zier curves.
The control points for $_0P$ and $_1P$ are denoted as $_0p_i$ and $_1p_i$, respectively, and can be expressed as
\begin{equation}
  \begin{tabular}{l l}
    $_0p_0 = \;x_2     + d \cdot u_{12}$ & $_1p_0 = \;x_2     + d \cdot u_{32}$ \\
    $_0p_1 = {_0p_0} - g \cdot u_{12}$ & $_1p_1 = {_1p_0} - g \cdot u_{32}$ \\
    $_0p_2 = {_0p_1} - h \cdot u_{12}$ & $_1p_2 = {_1p_1} - h \cdot u_{32}$ \\
    $_0p_3 = {_0p_2} + k \cdot u_d$    & $_1p_3 = {_1p_2} - k \cdot u_d$
  \end{tabular},
\end{equation}
\noindent
where $u_d$ is the unit vector pointing from ${_0p_2}$ to ${_1p_2}$ and the weights, $h$, $g$, and $k$ are defined as
\begin{equation}
\begin{tabular}{c c c}
  $h = \nu_3 d$ & $g = \nu_2 \nu_3 d$ & $k = \frac{6 \nu_3 \cos\left(\frac{\gamma}{2}\right)}{\nu_2 + 4}d$
\end{tabular}.
\end{equation}



A procedure that differs from what is found in \citep{Yang2014} is now defined to generate curves.
In \citep{Yang2014}, a maximum curve angle, $\gamma_{max}$, is employed with an associated distance $d_{min} = d(\gamma_{max})$.
To ensure subsequent fillet curves do not overlap, connecting points are forced to be $2d_{min}$ apart.
We found the $2d_{min}$ node separation to be overly restrictive as small path refinements are not allowed under such a constraint.
These small path refinements prove necessary, especially around curves in the obstacles.
This limitation significantly reduces the ability to rewire, during which small refinements to the tree are common.

The underlying desired constraint enforced with the $2d_{min}$ spacing is path continuity. The conditions in \eqref{equ:conditions} give the path generation process more flexibility. An example is given in Figure \ref{fig:spline_reach} showing the points that could be considered under both the conditions in \eqref{equ:conditions} and under a $\gamma_{max}$ and $2d_{min}$ constraint. The curve generating procedure can now be stated.


\begin{procedure}[$X_{fillet} \leftarrow BezierFillet(x_0,x_1,x_2,x_3)$]
  $BezierFillet$ uses the cubic B\'ezier spline to generate a $\cont{2}$ continuous curve from $x_1$ to $x_3$ using \eqref{equ:bezier_curve} through \eqref{equ:d_gamma}.
  Feasibility checks are made with respect to the curve that starts at $x_0$ and goes to $x_2$ as is described by the combination of \eqref{equ:conditions} and \eqref{equ:d_gamma}.
  If the checks fail the null path is returned.
\end{procedure}

\subsection{A Comparison of Arc and B\'ezier Fillets}
Arc and B\'ezier fillets provide different advantages and disadvantages.
A big advantage of the arc-fillet is its simplicity and speed.
As can be seen in Table \ref{fig:edge_benchmark}, an arc-fillet can be generated in about half the time it takes to make a B\'ezier-fillet.
Another benefit of the arc-fillet is that it is less constrained than the B\'ezier-fillet, resulting in a larger reachability set for connecting points, as shown in Figure \ref{fig:fillet_reach}.
This fact is critical to RRT because it directly affects exploration and convergence.
\begin{table}[t]
  \centering
  \caption{Three points, $x_1,x_2,x_3$, were randomly sampled 1 million times with each $x,y$ component bounded between 0 and 10 at each sample.
           For Dubin's paths, the orientation of a point was set to be tangential to the vector pointing to it from its parent.
           The average length of the resulting paths and time it took to find them is presented with their respective standard deviations.}
  \label{fig:edge_benchmark}
  \begin{tabular}{|c|c|c|}
    \hline
    Motion Primitive & Length ($m$) & Computation Time ($\mu s$) \\
    \hline
    Straight-line    & $20.845 \pm 7.417$ & $\,\,\, 3.371 \pm \,\,\, 1.542$ \\
    \hline
    Dubin's path     & $23.797 \pm 7.729$ & $19.087 \pm \,\,\, 8.670$ \\
    \hline
    Arc-fillet       & $18.838 \pm 6.811$ & $\,\,\, 7.965 \pm \,\,\, 2.762$ \\
    \hline
    B\'ezier-fillet  & $18.909 \pm 6.242$ & $15.568 \pm 12.517$ \\
    \hline
  \end{tabular}
\end{table}

\begin{figure*}[ht]
  \centering
  \begin{raggedright}
    \begin{subfigure}[t]{0.48\linewidth}
      \centering
      \resizebox{0.9 \linewidth}{!}{
        \begin{tikzpicture}
          \begin{axis}[
            ymin=-0.75,
            ymax=6.5,
            xmin=-4.5,
            xmax=4.5,
            ylabel near ticks,
            xlabel near ticks,
            ylabel={$Y(m)$},
            xlabel={$X(m)$},
            axis on top=true,
            ]

            \addplot[color=green, fill] (5,7) -- (5,-2) -- (-5,-2) -- (-5,7) -- (5,7);

            \addplot[color=yellow, fill] table [x=x,y=y, col sep=comma] {sim_data/reachability/other_bezier_constraint_a.csv};
            \addplot[color=yellow, fill] table [x=x,y=y, col sep=comma] {sim_data/reachability/other_bezier_constraint_b.csv} -- (5,-2) -- (-5,-2)\closedcycle;

            \addplot[color=red, fill] table [x=x,y=y, col sep=comma] {sim_data/reachability/bezier_constraint_a.csv}\closedcycle;
            \addplot[color=red, fill] table [x=x,y=y, col sep=comma] {sim_data/reachability/bezier_constraint_b.csv} -- (5,-2) -- (-5,-2)\closedcycle;

            \node[label={right:{$x_1$}},inner sep=0pt](x1) at (0,0) {};
            \node[label={345:{$x_2$}},inner sep=0pt](x2) at (0,3) {};

            \draw[black,thick] (x1) -- (x2);
            \filldraw[black] (x1) circle [radius=0.75pt];
            \filldraw[black] (x2) circle [radius=0.75pt];
          \end{axis}
        \end{tikzpicture}
      }
      \caption{Comparison of the constraints in \citep{Yang2014} to those found in \eqref{equ:conditions} for B\'ezier curve fillets.
               The red shows the area that both sets of feasibility conditions deem invalid for $x_3$.
               The yellow is area that only \eqref{equ:conditions} deems valid.
               The green is area that both sets of conditions deem valid.}
      \label{fig:spline_reach}
    \end{subfigure}
  \end{raggedright}
  \begin{raggedleft}
    \begin{subfigure}[t]{0.48\linewidth}
      \centering
      \resizebox{0.9 \linewidth}{!}{
        \begin{tikzpicture}
          \begin{axis}[
            ymin=-0.75,
            ymax=6.5,
            xmin=-4.5,
            xmax=4.5,
            ylabel near ticks,
            xlabel near ticks,
            ylabel={$Y(m)$},
            xlabel={$X(m)$},
            axis on top=true,
            ]

            \addplot[color=green, fill] (5,7) -- (5,-2) -- (-5,-2) -- (-5,7) -- (5,7);

            \addplot[color=yellow, fill] table [x=x,y=y, col sep=comma] {sim_data/reachability/bezier_constraint_a.csv};
            \addplot[color=yellow, fill] table [x=x,y=y, col sep=comma] {sim_data/reachability/bezier_constraint_b.csv} -- (5,-2) -- (-5,-2)\closedcycle;

            \addplot[color=red, fill] table [x=x,y=y, col sep=comma] {sim_data/reachability/arc_constraint_a.csv}\closedcycle;
            \addplot[color=red, fill] table [x=x,y=y, col sep=comma] {sim_data/reachability/arc_constraint_b.csv};

            \node[label={right:{$x_1$}},inner sep=0pt](x1) at (0,0) {};
            \node[label={345:{$x_2$}},inner sep=0pt](x2) at (0,3) {};

            \draw[black,thick] (x1) -- (x2);
            \filldraw[black] (x1) circle [radius=0.75pt];
            \filldraw[black] (x2) circle [radius=0.75pt];
          \end{axis}
        \end{tikzpicture}
      }
      \caption{Visualizes the reachability regions of the arc-fillet and B\'ezier-fillet generation.
               The red shows the area that both fillets cannot reach.
               The yellow is area that arc-fillets can reach but B\'ezier-fillets cannot.
               The green is area that both fillets can reach.}
      \label{fig:fillet_reach}
    \end{subfigure}
  \end{raggedleft}
  \caption{Comparisons of different constraints on the position of $x_3$ if a fillet was made starting from $x_1$ through $x_2$ and to $x_3$.
    On the left, the constraints in \citep{Yang2014} are compared to those found in \eqref{equ:conditions} for the B\'ezier curve fillets.
    On the right, the difference in the definition of $d(\gamma)$ between arc-fillets and B\'ezier-fillets is expressed in terms of where \eqref{equ:conditions} is satisfied.
    Both figures assume $x_1$ is at the origin, $x_2 = \protect\begin{bmatrix} 0 & 3 \protect\end{bmatrix}^T$, $\kappa_{max} = 2m^{-1}$, $d(\gamma_1) = 0$, $d_{min} = 1.5m$, and $\gamma_{max} = 0.624\pi$ radians.
  }
    \label{fig:reachability}
\end{figure*}

The major advantage of the B\'ezier-fillet is the smoothness of the resulting path.
The arc-fillet, like Dubin's paths, guarantees only $\cont{1}$ continuity of pose.
The B\'ezir-fillet guaranties $\cont{2}$ continuity of pose and $\cont{1}$ continuity of curvature, as is shown in Figure \ref{fig:fillet_curvature}.
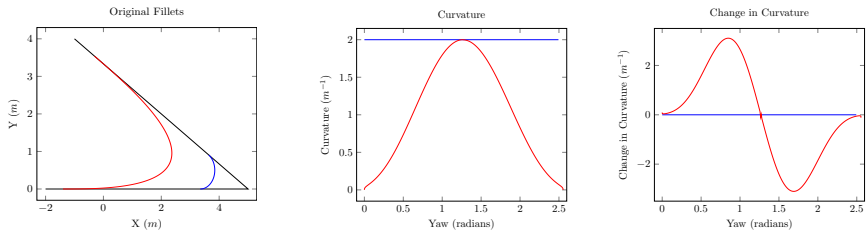
\begin{figure*}[ht]
  \centering
  \begin{subfigure}[t]{0.32\linewidth}
    \resizebox{0.9 \linewidth}{3.05cm}{
    \begin{tikzpicture}
      \begin{axis}[
          ymin=-0.3,
          ymax=4.3,
          xmin=-2.3,
          xmax=5.3,
          ylabel={Y ($m$)},
          xlabel={X ($m$)},
          title={Original Fillets},
          ylabel near ticks,
          xlabel near ticks,
          legend style={nodes={scale=1, transform shape},at={(1,1)},anchor=north east}
      ]
        \addplot[black] plot coordinates { (-2,0) (5,0) (-1,4) };
        \addplot[color=blue,thick] table [x=X,y=Y, col sep=comma] {sim_data/arc_fillet.csv};
        \addplot[color=red, thick] table [x=X,y=Y, col sep=comma] {sim_data/bezier_fillet.csv};
      \end{axis}
    \end{tikzpicture}
    }
    \caption{Fillets generated between the points $\begin{bmatrix} -2 & 0 \end{bmatrix}^T$\!, $\begin{bmatrix} 5 & 0 \end{bmatrix}^T$\!, and $\begin{bmatrix} -1 & 4 \end{bmatrix}^T$\!.}
  \end{subfigure}
  \begin{subfigure}[t]{0.32\linewidth}
    \centering
    \resizebox{0.9 \linewidth}{!}{
    \begin{tikzpicture}
      \begin{axis}[
          ymin=-0.15,
          ymax=2.15,
          xmin=-0.1,
          xmax=2.6,
          ylabel={Curvature ($m^{-1}$)},
          xlabel={Yaw (radians)},
          title={Curvature},
          ylabel near ticks,
          xlabel near ticks,
          legend style={nodes={scale=1, transform shape},at={(1,1)},anchor=north east}
      ]
        \addplot[color=blue,thick] table [x=YAW, y=Curvature, col sep=comma] {sim_data/arc_fillet.csv};
        \addplot[color=red, thick] table [x=YAW, y=Curvature, col sep=comma] {sim_data/bezier_fillet.csv};
      \end{axis}
    \end{tikzpicture}
    }
    \caption{The curvature of each fillet over the fillet curve.}
  \end{subfigure}
  \begin{subfigure}[t]{0.32\linewidth}
    \centering
    \resizebox{0.9 \linewidth}{!}{
    \begin{tikzpicture}
      \begin{axis}[
          ymin=-3.5,
          ymax=3.5,
          xmin=-0.1,
          xmax=2.6,
          ylabel={Change in Curvature ($m^{-1}$)},
          xlabel={Yaw (radians)},
          title={Change in Curvature},
          ylabel near ticks,
          xlabel near ticks,
          legend style={nodes={scale=1, transform shape},at={(1,1)},anchor=north east}
      ]
        \addplot[color=blue,thick] table [x=YAW, y=dCurvature, col sep=comma] {sim_data/arc_fillet.csv};
        \addplot[color=red, thick] table [x=YAW, y=dCurvature, col sep=comma] {sim_data/bezier_fillet.csv};
      \end{axis}
    \end{tikzpicture}
    }
    \caption{The first derivative of the curvature of each fillet type.}
  \end{subfigure}
  \caption{An example of the arc and B\'{e}zier fillets with the corresponding curvature and curvature rate. The arc-fillet is shown in blue and the B\'ezier-fillet is shown in red. Note that the arc has zero curvature change as there is an instantaneous jump from zero to maximum curvature.}
  \label{fig:fillet_curvature}
\end{figure*}



  \section{Fillet-based RRT*}
  \label{sec:fillet-rrt-star}
%

Fillet-based variants of the RRT and RRT* algorithms are proposed in this section.
The overall structure of the fillet-based variants is the same as their standard counterparts.
This section will only cover the procedures that change, namely: $Initialize$, $CostToCome$, $Extend$, $Extend^*\!$, and $Rewire$.
To construct the fillet-based variants, each of these procedures is redefined with a procedure of the same name but with the prefix \enquote{$FB$}, i.e. $Extend$ becomes $FB\text{-}Extend$.
The first four procedures have small changes to their standard counterparts and are presented first.
The $FB\text{-}Rewire$ procedure is then discussed in detail.
This section ends with a discussion of FB-RRT*'s advantages.
The FB-RRT* algorithm is not presented until Section \ref{section:improving_convergence} where a modified sampling procedure is discussed.

\subsection{Procedures with Minor Changes}
\label{sec:procedures_with_minor_changes}
The first procedure to be modified is the $Initialize$ procedure.
As defined, the $Initialize$ procedure has no consideration of the vehicle orientation, which must be respected to generate an executable path for the vehicle.
$FB\text{-}Initialize$ differs from $Initialize$ in that it initializes the search tree to have an edge of length $d_{init} \in \mathbb{R}_+$ extending from the root, $x_{r}$, to a point in the direction of the initial orientation of the robot, $\psi_{r} \in [-\pi,\pi)$.
Recall that connecting to a new point using a fillet requires both a parent and a grandparent node.
If $x_{r}$ is returned from $Nearest$ or $Near$ any attempt to make an edge with $x_{r}$ is ignored as it has no parent, ensuring that the starting orientation is respected\footnote{In our implementation, $x_{r}$ is left out of the search in $Nearest$ and $Near$.}.
The new initialization procedure is defined as follows.
\begin{procedure}[$T \leftarrow FB\text{-}Initialize_{d_{init}}(x_{r},\psi_{r})$]
  Returns a tree with two nodes and an edge of length $d_{init}$ based on the root node, $x_r$, and initial orientation, $\psi_r$, as defined in Algorithm \ref{alg:fb_initialize}.
\end{procedure}

\begin{algorithm}[t]
  \caption{\hbox{$T \leftarrow FB\text{-}Initialize_{d_{init}}(x_{r},\psi_{r})$}}
  \label{alg:fb_initialize}
  \begin{algorithmic}[1]
    \State $x_{n} \leftarrow x_{r} + d_{init} \begin{bmatrix} \cos\left(\psi_{r}\right) & \sin\left(\psi_{r}\right) \end{bmatrix}^T$
    \State $V \leftarrow \{x_{r},x_n\}$
    \State $E \leftarrow \{(x_{r},x_n)\}$
    \State $T \leftarrow \{V,E\}$
    \State \Return $T$
  \end{algorithmic}
\end{algorithm}

The $CostToCome$ procedure is redefined to accommodate the new fillet-based path generation.
Recalling Lemma \ref{lem:recursive_length}, this depends upon both the parent and grandparent nodes of the node in question.
The new procedure is given as follows.
\begin{procedure}[$c_{n} \leftarrow FB\text{-}CostToCome(x_{n},x_{p},x_{gp},T)$]
  Calculates the cost of $x_{n}$ if it is connected to the tree, T, through the parent, $x_{p}$, and the grandparent, $x_{gp}$, as in Algorithm \ref{alg:fillet_cost_to_come}.
  Note that the $Fillet$ procedure in Algorithm \ref{alg:fillet_cost_to_come} can be replaced by $ArcFillet$ or $BezierFillet$ depending upon the fillet being used.
\end{procedure}

\begin{algorithm}[t]
  \caption{\hbox{$c_n \leftarrow FB\text{-}CostToCome(x_{n},x_{p},x_{gp},T)$}}
  \label{alg:fillet_cost_to_come}
  \begin{algorithmic}[1]
    \State $X_{fillet} \leftarrow Fillet(\Revision{Parent(x_{gp},T)},x_{gp},x_{p},x_{n})$
    \If{$X_{fillet} \neq \varnothing \And CollisionFree(X_{fillet})$}
      \State \Return $Cost(x_{p},T) + c(X_{fillet}) - \left\|x_p - x_{gp}\right\|$ \Comment{Path length calculation}
    \Else
      \State \Return $\infty$
    \EndIf
  \end{algorithmic}
\end{algorithm}

The $Extend$ procedure is updated in two substantial ways.
The first is the use of $FB\text{-}CostToCome$.
The second is the use of a \enquote{node orientation} when performing the nearest neighbor search. By incorporating a sense of \enquote{node orientation}, infeasible sharp turns can be avoided in the nearest neighbor searching.
There is no actual \enquote{orientation} of a node as the nodes are 2D points that guide the creation of the path.
However, we can bias the nearest neighbor search to penalize turns by noting that a fillet ending at a node will be oriented with the line extending from the node's parent to the node.
After the fillet curve has been executed, the robot will be aligned with that orientation.
Given node $x_i$ and its parent $x_{i-1}$, the orientation of $x_i$ is defined as $\psi_i = atan2(u_{i-1\mbox{ }i,2},u_{i-1\mbox{ }i,1})$.
The nearest neighbor search is performed over $\begin{bmatrix} x_{i,1} & x_{i,2} & \cos(\psi_i) & \sin(\psi_i) \end{bmatrix}^T$ instead of $\begin{bmatrix} x_{i,1} & x_{i,2} \end{bmatrix}^T$.
The inclusion of a pseudo ``node orientation'' combined with the relaxed continuity constraints, depicted in Figure \ref{fig:spline_reach}, enables us to forgo the \mbox{$k$-nearest-neighbor} search that \cite{spline_rrt*} uses to aid in the success of $Extend$.
The updated procedure is now defined.

\begin{procedure}[$\{x_{n},x_{p},c_n\} \leftarrow FB\text{-}Extend(x_{rand},T)$]
  Given $x_{rand} \in X$ and a tree $T$, the $FB\text{-}Extend$ procedure finds the closest vertex to $x_{rand}$ in terms of a combined position and orientation metric and attempts to extend the tree in the direction of $x_{rand}$, as defined in Algorithm \ref{alg:fillet_extend}.
\end{procedure}

\begin{algorithm}[t]
  \caption{\hbox{$\{x_{n},x_{p},c_n\} \leftarrow FB\text{-}Extend(x_{rand},T)$}}
  \label{alg:fillet_extend}
  \begin{algorithmic}[1]
    \State $x_{nearest} \leftarrow Nearest(x_{rand},T)$
    \State $x_{n} \leftarrow Steer_\eta(x_{nearest},x_{rand})$ \Comment{Steer towards the nearest point}
    \State $x_{gp} \leftarrow Parent(x_{nearest},T)$
    \State $c_{n} \leftarrow$ $FB\text{-}CostToCome(x_{n},x_{nearest},x_{gp},T)$ \Comment{Evaluate cost of path}
    \If{$\infty \neq c_{n}$}
      \State \Return $\{x_{n},x_{nearest},c_{n}\}$
    \EndIf
    \State \Return $\{\varnothing,\varnothing,\inf\}$
  \end{algorithmic}
\end{algorithm}

The $FB\text{-}Extend^*$ procedure is identical to $Extend^*$ except for the use of $FB\text{-}Extend$ and $FB\text{-}CostToCome$.
The $FB\text{-}Extend$ and $FB\text{-}Extend^*$ procedures are illustrated in Figure \ref{fig:fillet_extend}.
\Revision{
Note that the numbers shown in Figure \ref{fig:fillet_extend} are the edge costs of the edges they are near.
The same is true of Figures \ref{fig:extend}, \ref{fig:optimal_extend}, and \ref{fig:rewire} except Figure \ref{fig:fillet_extend} uses the fillet cost calculation described in Lemma \ref{lem:recursive_length}.
}
\begin{procedure}[$\{x_{n},x_{p}\} \leftarrow FB\text{-}Extend^*\!(x_{rand},T)$]
  Given $x_{rand} \in X$ and a tree $T = \{V,E\}$, the $FB\text{-}Extend^*$ uses $FB\text{-}Extend$ to find a node for extending the tree and then finds the locally optimal path in $V$ for connecting to the new point. The procedure is given in Algorithm \ref{alg:fillet_optimal_extend}.
\end{procedure}
\begin{algorithm}[t]
  \caption{\hbox{$\{x_{n},x_{p}\} \leftarrow FB\text{-}Extend^*\!(x_{rand},T)$}}
  \label{alg:fillet_optimal_extend}
  \begin{algorithmic}[1]
    \State $\{x_{n},x_{p},c_{min}\} \leftarrow FB\text{-}Extend(x_{rand},T)$
    \If{$x_{n} \neq \varnothing$} \Comment{If an extension is possible}
      \State $X_{near} \leftarrow Near_{\rho,\alpha}(x_{n},T)$
      \ForAll{$x_{near} \in X_{near}$} \Comment{Check for a lower cost connection}
        \State $x_{gp} \leftarrow Parent(x_{near},T)$
        \State $c_{tmp} \leftarrow FB\text{-}CostToCome(x_{n},x_{near},x_{gp},\!T)$
        \If{$c_{tmp} < c_{min}$}
          \State $x_{p} \leftarrow x_{near}$
          \State $c_{min} \leftarrow c_{tmp}$
        \EndIf
      \EndFor
      \State \Return $\{x_{n},x_{p}\}$
    \EndIf
    \State \Return $\{\varnothing,\varnothing\}$;
  \end{algorithmic}
\end{algorithm}

\input{tikz_pics/fillet_rewire.tex}

\subsection{The Fillet-based Rewire Procedure}
\label{sec:fb_rewire}
In the FB-RRT* framework, care must be taken to ensure both path feasibility and cost improvement when rewiring.
Unlike its straight-line counterpart, it is not sufficient to choose a parent based purely on the path length to the node.
The following lemma presents a set of sufficient conditions to ensure that a rewiring will not be detrimental to the tree.
\begin{lemma}
  \label{lem:fb-rewire}
  Assume that a tree $T=\{V,E\}$ is given such that \eqref{equ:conditions} is satisfied for all consecutive nodes.
  Rewiring $E$ to make $x_n \in V$ the new parent of $x_{near}\in V$ will result in a continuous path with all node costs unchanged or lowered if the following three conditions are met:
  \begin{enumerate}
    \item \label{rewA} The resulting path to $x_{near}$ using $x_n$ as its parent is obstacle free, does not violate \eqref{equ:conditions}, and the cost of $x_{near}$ is improved, see Figure \ref{fig:fillet_rewire:sub3}.
    \item \label{rewB} The resulting path to each child of $x_{near}$ is obstacle free, does not violate \eqref{equ:conditions}, and the cost of the child is not increased, see Figure \ref{fig:fillet_rewire:sub4}.
    \item \label{rewC} The resulting path to each grandchild of $x_{near}$ does not violate \eqref{equ:conditions}, see Figure \ref{fig:fillet_rewire:sub5}.
  \end{enumerate}
\end{lemma}
\begin{proof}
  The only paths that will be affected by changing the parent of $x_{near}$ will be the fillet connecting $x_{near}$ to its grandparent and the fillets connecting $x_{n}$ to the children of $x_{near}$.
  Conditions \ref{rewA}, \ref{rewB}, and \ref{rewC} employ obstacle checking and \eqref{equ:conditions} to ensure that fillet curves do not overlap in the new section of path nor with the preceding or subsequent sections of the path.

  The path cost to $x_{near}$ will be improved due to \ref{rewA}.
  The path costs to the children are not increased per \ref{rewB}.
  As all other parent and grandparent nodes remain unchanged, their respective path costs will not increase due to Corollary \ref{cor:unchanged}.
\end{proof}
The $FB\text{-}Rewire$ procedure is now stated and illustrated in Figure \ref{fig:fillet_rewire}.
\begin{procedure}[$E \leftarrow FB\text{-}Rewire(x_{n},X_{near},T)$]
  Given a tree, $T=\{V,E\}$, with node $x_n \in V$ and set $X_{near} \subset V$, $FB\text{-}Rewire$ returns a modified tree with $E$ changed to have $x_n$ be the parent to elements of $X_{near}$ if conditions in Lemma \ref{lem:fb-rewire} are satisfied. The procedure is given in Algorithm \ref{alg:fillet_rewire}.
\end{procedure}

\begin{algorithm}[t]
  \caption{$T \leftarrow FB\text{-}Rewire(x_{n},X_{near},T)$}
  \label{alg:fillet_rewire}
  \begin{algorithmic}[1]
    \ForAll{$x_{near} \in X_{near}$}
      \State $x_{p} \leftarrow Parent(x_{n},T)$
      \State $c_{near} \leftarrow FB\text{-}CostToCome(x_{near},x_{n},x_{p},T)$
      \LineComment{Check $x_{near}$'s feasibility and that its cost is reduced}
      \If{$c_{near} \ge Cost(x_{near},T)$}
        \State \GoTo Continue
      \EndIf
      \ForAll{$x_{c} \in Children(x_{near},T)$}
        \State $c_c \leftarrow FB\text{-}CostToCome(x_{c},x_{near},x_{n},T)$
        \LineComment{Check the feasibility and cost of $x_{near}$'s children}
        \If{$c_c > Cost(x_{c},T)$}
          \State \GoTo Continue
        \EndIf
        \ForAll{$x_{gc} \in Children(x_{c},T)$}
          \State $X_{fillet} \leftarrow Fillet(\Revision{x_n,x_{near},x_{c},x_{gc}})$
          \If{$\varnothing = X_{fillet}$} \Comment{Check the feasibility of $x_{near}$'s grandchildren}
            \State \GoTo Continue
          \EndIf
        \EndFor
      \EndFor
      \State $x_{p} \leftarrow Parent(x_{near},T)$
      \State $E \leftarrow \left(E \setminus \{x_{p},x_{near}\} \right) \cup \{x_{n},x_{near}\}$ \Comment{Rewire around $x_{near}$}
      \State Continue:\:
    \EndFor
    \State \Return $T$
  \end{algorithmic}
\end{algorithm}

Note that the $Rewire$ procedure described above is different than that in \citep{spline_rrt*}.
In \citep{spline_rrt*}, the neighborhood set of $x_{n}$, $X_{near}$, is checked to ensure that:
\begin{enumerate}
  \item[a.] Connecting $x_{n}$ and $x_{near}$ will not violate their max angle and distance conditions.
  \item[b.] The curve formed between $x_{n}$ and $x_{near}$ is obstacle free.
  \item[c.] The cost of $x_{near}$ will be improved by the rewire operation.
\end{enumerate}
Thus, condition \ref{rewA} is met (with a conservative continuity condition), but conditions \ref{rewB} and \ref{rewC} are not considered.
If \cite{spline_rrt*}'s conditions hold, the children of $x_{near}$ are set to be the children of the parent of $x_{near}$ and the parent of $x_{near}$ is set to be $x_{n}$, see Figure \ref{fig:spline_rewire_counter:sub2}.
It is important to note that the $Rewire$ operation in \citep{spline_rrt*} could result in discontinuous paths due to not checking feasibility for all affected nodes.
There is no guarantee that connecting the parent of $x_{near}$ directly to the children of $x_{near}$ will result in a valid tree.
The angles and distances formed by that connection must first be checked as illustrated in Figure \ref{fig:spline_rewire_counter}.
Moreover, due to not checking costs on all affected nodes, the $Rewire$ operation in \citep{spline_rrt*} may actually increase costs to some nodes as shown in Lemma \ref{lem:rewiring_tests}.

\tikzstyle{node}     =[circle,draw=black!100,thick]
\tikzstyle{new-node} =[circle,draw=red!100,  thick]
\tikzstyle{near-node}=[circle,draw=green!100,thick]
\begin{figure}[h]
  \centering
  \begin{subfigure}[t]{0.47\linewidth}
    \centering
    \resizebox{\linewidth}{!}{
    \begin{tikzpicture}

      \node[node](A) at (0,0) {A};

      \node[node](B) at ($(A)!0.5!(0,-6)$) {B};
      \draw[->,thick](A) -- (B);

      \node[near-node](C) at ($(B)!0.298!(-10,-5)$) {C};
      \draw[->,thick](B) -- (C);

      \node[node](D) at ($(A)!0.3487429!(5,-7)$) {D};
      \draw[->,thick](A) -- (D);

      \node[node](E) at ($(C)!0.465!(-9,-1.1)$) {E};
      \draw[->,thick](C) -- (E);

      \node[new-node](F) at ($(A)!0.344!(-8,-3.5)$) {F};
      \draw[->,thick](A) -- (F);

      \coordinate (temp) at ($(A)!2!(B)$) {};
      \draw pic[draw,angle radius=0.5cm,"$\gamma_a$" shift={(-3mm,-3mm)}] {angle=C--B--temp};

      \draw[color=blue] ([xshift=2.9cm]F) arc (0:35:2.9);
      \draw[color=blue] ([xshift=2.9cm]F) arc (0:-215:2.9);
    \end{tikzpicture}
    }
    \caption{The Node F is being rewired around the preexisting tree, and node C is F's neighborhood set.}
    \label{fig:spline_rewire_counter:sub1}
  \end{subfigure}
  \begin{subfigure}[t]{0.47\linewidth}
    \centering
    \resizebox{\linewidth}{!}{
    \begin{tikzpicture}

      \node[node](A) at (0,0) {A};

      \node[node](B) at ($(A)!0.5!(0,-6)$) {B};
      \draw[->,thick](A) -- (B);

      \node[node](F) at ($(A)!0.344!(-8,-3.5)$) {F};
      \draw[->,thick](A) -- (F);

      \node[node](C) at ($(B)!0.298!(-10,-5)$) {C};
      \draw[->,thick](F) -- (C);

      \node[node](D) at ($(A)!0.3487429!(5,-7)$) {D};
      \draw[->,thick](A) -- (D);

      \node[node](E) at ($(C)!0.465!(-9,-1.1)$) {E};
      \draw[->,thick](B) -- (E);

      \coordinate (temp) at ($(A)!1.2!(B)$) {};
      \draw pic[draw,angle radius=0.5cm,"$\gamma_b$" shift={(-3mm,-3mm)}] {angle=E--B--temp};
    \end{tikzpicture}
    }
    \caption{Node C is rewired to have node F as its parent, and node E becomes a child of node B.}
    \label{fig:spline_rewire_counter:sub2}
  \end{subfigure}
  \caption{After node C has been rewired to node F the angle formed between nodes A, B, and B's child has increased, i.e. $\gamma_b > \gamma_a$.
    Without checking, there is no way to know if $\gamma_b$ is less then the max angle allowed.}
  \label{fig:spline_rewire_counter}
\end{figure}
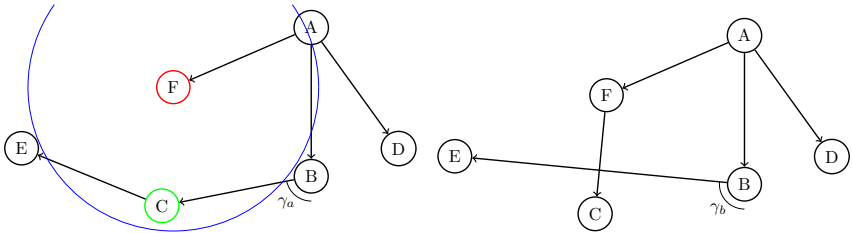

Note that the FB-RRT and FB-RRT* algorithms are not yet stated as additional improvements to the sampling and smoothing are first discussed in the following section.
\Revision{
Also note that reverse motion considerations are discussed in Appendix \ref{sec:reverse_fillet}.
}

  \section{Overcoming the Voronoi Property}
  \label{section:improving_convergence}
  Nonholonomic constraints exacerbate the slow convergence of RRT*.
This section introduces two refinements to \mbox{FB-RRT*} designed to reduce convergence time.
There are multiple variants of RRT* that aim to improve convergence by improving sampling \citep{Akgun2011, Gammell2014, Kuwata2009, Nasir2013}.
Two such variants are used as a basis for design herein: Informed RRT* \mbox{(I-RRT*)} \citep{Gammell2014} and Smart RRT* \mbox{(S-RRT*)} \citep{Nasir2013}.
These algorithms are identical to RRT* before the first solution is found.
However, after the first path to the goal is found, they use information about the solution to guide sampling toward space that will improve the final solution.
In addition, \mbox{S-RRT*} introduces a path smoothing procedure that continuously refines the solution path, further improving convergence.
This section first presents \mbox{I-RRT*} and \mbox{S-RRT*}.
These techniques are then combined in the Fillet-based Smart and Informed RRT* \mbox{(FB-SI-RRT*)} formulation.

\subsection{Informed RRT*}
\label{section:informed_rrt*}
Informed RRT* \mbox{(I-RRT*)} is an extension of RRT* that aims to reduce the amount of time spent sampling space that will not improve the final path.
It is observed that any time spent sampling outside of the set that will improve the path is wasted time. This set is referred to as the ``informed'' set and denoted as $X_{i} \subset X_{free}$.
Naturally, it would be best to directly sample $X_{i}$.
However, calculating $X_{i}$ can be difficult, if not impossible.

A conservative approximation of $X_{i}$, denoted as $X_{i}^\prime$, is defined in \citep{Gammell2014} based upon an approximation of the problem's optimal cost and the best cost found, $c_{best}$. The optimal cost is denoted as $c_{min}$ with its approximation denoted as $c_{min}^\prime$.
The approximation for $c_{min}^\prime$ is calculated as the distance between the root node and final state in the best-found path, i.e. $c_{min}^\prime = \| x_{r} - x_{t} \|$.
$X_i^\prime$ is defined as
\begin{equation}
  X_{i}^\prime = \mathcal{E}_{x_{r},x_t}
\end{equation}
where $\mathcal{E}_{x_{r},x_t}$ is the open set of states in an ellipse with the focal points set at $x_{r}$ and $x_{t}$.
The length of the major axis of the ellipse is $c_{best}$ and the length of the minor axis is $\sqrt{c_{best}^2 - c_{min}^{\prime2}}$, see Figure \ref{fig:ellipse}.
With this definition of $X_{i}^\prime$, it is guaranteed that the best solution found so far is entirely inside the informed set.

\begin{figure}[t]
  \centering
  \begin{tikzpicture}[scale=0.7]
    \draw[color=black,very thick] (0,0) ellipse (4cm and 2cm);

    \node[black,label={above:{$x_{r}$}},inner sep=0pt] (xi) at (-3, 0) {};
    \node[black,label={above:{$x_{t}$}},inner sep=0pt] (xt) at ( 3, 0) {};

    \filldraw[black] (xi) circle(1.5pt);
    \filldraw[black] (xt) circle(1.5pt);

    \draw[black] (xi) -- ++(0,-0.5) coordinate[midway] (brace_start);
    \draw[black] (xt) -- ++(0,-0.5) coordinate[midway] (brace_end);
    \draw[<->,black] (brace_start) -- (brace_end) node[black,midway,label={below:{$c_{min}^\prime$}},inner sep=0pt] {};

    \draw[black] (-4,2) -- ++(0,0.5) coordinate[midway] (brace_start);
    \draw[black] ( 4,2) -- ++(0,0.5) coordinate[midway] (brace_end);
    \draw[<->,black] (brace_start) -- (brace_end) node[black,midway,label={above:{$c_{best}$}},inner sep=0pt] {};

    \draw[black] (-4,-2) -- ++(-0.5,0) coordinate[midway] (brace_start);
    \draw[black] (-4, 2) -- ++(-0.5,0) coordinate[midway] (brace_end);
    \draw[<->,black] (brace_start) -- (brace_end) node[black,midway,label={[rotate=90]above:{$\sqrt{c_{best}^2 - c_{min}^{\prime2}}$}},inner sep=0pt] {};
  \end{tikzpicture}
  \caption{The ellipse that defines $X_i^\prime$.}
  \label{fig:ellipse}
\end{figure}
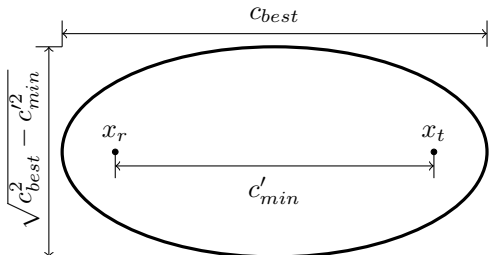
\mbox{I-RRT*} samples solely within $X_{i}^\prime$ once the first solution is found, reducing the configuration space sampled and improving the probability that a given sample will improve the solution.
Note that \mbox{I-RRT*} is identical to RRT* with the exception of using $X_{i}^\prime$ instead of $X_s$ for sampling new points after the first solution is found. 
As long as $c_{min}^\prime \le c_{min}$, it is shown in \citep{Gammell2014} that \mbox{I-RRT*} maintains asymptotic optimality.

In an environment with small convex obstacles, I-RRT* is shown to converge much faster than RRT* \citep{Gammell2014}.
If there are larger nonconvex obstacles then I-RRT* degrades in performance to RRT*.
This is because the optimal cost, $c_{min}$, is much more than $c_{min}^\prime$, causing the informed subset to be very large.

\subsection{Smart RRT*}
\label{section:rrt*_smart}
Smart RRT* \mbox{(S-RRT*)} proposes two alternative approaches to improving convergence \citep{Nasir2013}.
First, similar to \mbox{I-RRT*}, the sampling set is reduced once a solution is found.
Second, \mbox{S-RRT*} proposes a path smoothing procedure that is used as an integral part of the algorithm to further improve convergence.

The sampling set is reduced by biasing sampling around the nodes that form the shortest found path. The idea being that improvements near the path can help to refine the chosen route around obstacles.
These nodes are referred to as the beacon set, $X_b$, and can be defined as
\begin{equation*}
  X_{b} = Solution(x_{t},T)
\end{equation*}
where $x_{t} \in X_{t}$ is the end of the shortest path to the target set.
Subsequent iterations then bias the sampling towards the union of balls of radius $r_b$ around each beacon, i.e.
\begin{equation*}
  X_s = \bigcup_{x_b \in X_b} \mathcal{B}_{x_b,r_b}
\end{equation*}
A major philosophical difference between \mbox{S-RRT*} and \mbox{I-RRT*} is that \mbox{S-RRT*} biases the sampling around the beacon set whereas \mbox{I-RRT*} seeks to reduce the size of the sampling space.
The biasing encourages local path refinement with a continued sampling of $X$ for exploration.

Significant improvements are also made through the addition of a path smoothing procedure, referred to as $OptimizePath$ in \citep{Nasir2013}. $OptimizePath$ seeks to straighten out the path each time a better path is found. This avoids waiting for the sampling to straighten the path, something that becomes decreasingly probable as the number of samples increases due to the Voronoi property.
The $OptimizePath$ procedure does this by performing rewire operations on each beacon with every other beacon, as detailed in Algorithm \ref{alg:optimize_path}.
Note that the $RemoveBetween$ procedure removes all of the nodes in $X_b$ that lie between $x_{near}$ and $x_{ittr}$ not including $x_{near}$ and $x_{ittr}$.
This removes nodes from the solution/beacon set that are not needed.
Which straightens out the solution and reduces the number of beacons that have to be sampled.
\begin{algorithm}[t]
  \caption{$X_{b}^\prime \leftarrow OptimizePath(X_{b},T)$}
  \label{alg:optimize_path}
  \begin{algorithmic}[1]
    \State $X_{b}^\prime \leftarrow \varnothing$ \label{alg:optimize_path:init_output}
    \ForAll{$x_{ittr} \in X_{b}$} \label{alg:optimize_path:first_for}
      \ForAll{$x_{near} \in X_{b} \cap X_{b}^\prime$} \label{alg:optimize_path:second_for} \Comment{Beacons that have not been used}
        \State $T \leftarrow Rewire(x_{ittr},\{x_{near}\},T)$ \label{ln:opt_path_rewire}
        \If{$Parent(x_{near},T) = x_{ittr}$} \Comment{Remove unneeded beacons}
          \State $X_b \leftarrow RemoveBetween(x_{near},x_{ittr},X_b)$
        \EndIf
      \EndFor \label{alg:optimize_path:end_second_for}
      \State $X_{b}^\prime \leftarrow X_{b}^\prime \cup \{x_{ittr}\}$
    \EndFor
    \State \Return $X_{b}^\prime$
  \end{algorithmic}
\end{algorithm}

Before the first solution is found, \mbox{S-RRT*} performs identically to RRT*.
However, after the first solution is found, the convergence of \mbox{S-RRT*} is much faster than RRT*, especially for straight-line motion primitives.
It is important to note that \mbox{S-RRT*} has a tendency to spend more time converging on local minima than RRT*.
In fact, if only the space near the beacon set were to be sampled after the first solution was found, then \mbox{S-RRT*} would lose its asymptotic optimality because it would only refine the first path found.
We found that the beacon radius could be relatively small using straight-line paths, but needed to be increased significantly for curvature constrained paths.

\subsection{Smart and Informed Sampling}
\mbox{I-RRT*} performs well in small path planning problems where the obstacles are convex and $c_{min}^\prime$ is a good approximation of $c_{min}$.
\mbox{S-RRT*} converges impressively fast, especially when planning using straight-line paths.
Both techniques add parameters that need to be determined -- an estimate of the best cost for informed and the beacon size for smart sampling.
This work develops smart-and-informed sampling that combines \mbox{S-RRT*'s} fast convergence with an adaptive sample set similar to \mbox{I-RRT*}.
It is designed specifically for the nonholonomic motion primitives to provide a sampling heuristic that does not require fine tuning of additional parameters.

As is the case with both of its predecessors, the sampling will be identical to RRT* until the first solution is found.
At that point, the $OptimizePath$ procedure, Algorithm \ref{alg:optimize_path}, is called on the initial solution.
Instead of using a constant radius around each beacon \mbox{(S-RRT*)}, or an adaptive set based upon the optimality of the entire path \mbox{(I-RRT*)}, the beacon set, $X_{b} \subset V$, is used to generate ellipses around each adjacent pair of beacons along the path that leads from $x_{r}$ to $X_{t}$, as illustrated in Figure \ref{fig:beacons}.
The axes of the ellipse are adapted based on the local optimality of the path, producing a larger sampling space when path refinement is needed and a smaller space when the local path is near optimal.
Similar to the sampling in \mbox{S-RRT*}, the sampling after the beacons are found is biased towards the set formed by these ellipses,
\begin{equation*}
  \mathcal{E}_{X_b} = \bigcup_{i=0}^{\mid X_b \mid-1} \mathcal{E}_{x_{b,i},x_{b,i+1}} \subset X.
\end{equation*}
This sampling bias encourages path refinement.
The full configuration space is still sampled periodically for sake of exploration.

\newcommand{\makeEllipseForArcs}[4]{%
  \draw[#4,fill=#4,fill opacity=0.1]%
    let \p1=($(#3)-(#2)$),%
        \p2=($(#2)-(#1)$),%
        \n1={atan2(\y1,\x1)},%
        \n2={mod(\n1 - atan2(\y2,\x2) + (4*acos(0)),2*acos(0))},%
        \n3={veclen(\y1,\x1)},%
        \n4={1.5 * ((\n2 * acos(0)) / 180)},%
        \n5={1.5 * (1 - cos(\n2))/sin(\n2)},%
        \n6={\n3 + \n4 - \n5},%
        \n7={\fpeval{sqrt((\n6 * \n6 * (7227/254)) - (\n3 * \n3 * (7227/254)))} * 1pt/1cm}%
    in [rotate around={\n1:(#3)}] ($(#2)!0.5!(#3)$) ellipse (0.5 * \n6 and 0.5 * \n7);%
}

\begin{figure}[t]
  \centering
    \begin{tikzpicture}[scale=0.4]
      \coordinate (x0) at (11,-7);
      \coordinate (x1) at (9,-7);
      \coordinate (x2) at (6, 6);
      \coordinate (x3) at (-7, 5);
      \coordinate (x4) at (-9,-5);

      \makeEllipseForArcs{x0}{x1}{x2}{blue}
      \makeEllipseForArcs{x1}{x2}{x3}{green}
      \makeEllipseForArcs{x2}{x3}{x4}{orange}

      \node[label={below:$x_{b,0}$},inner sep=0pt] at (x0) {};
      \node[label={225:  $x_{b,1}$},inner sep=0pt] at (x1) {};
      \node[label={45:   $x_{b,2}$},inner sep=0pt] at (x2) {};
      \node[label={135:  $x_{b,3}$},inner sep=0pt] at (x3) {};
      \node[label={-45:  $x_{b,4}$},inner sep=0pt] at (x4) {};

      \draw[black] (x0) -- (x1) -- (x2) -- (x3) -- (x4);

      \draw[red,   thick] (x0) -- (x1);
      \draw[blue,  thick,rounded corners=0.75cm] ($(x1) !1.5cm! (x0)$) -- (x1) -- (x2);
      \draw[green, thick,rounded corners=0.75cm] ($(x2) !1.5cm! (x1)$) -- (x2) -- (x3);
      \draw[orange,thick,rounded corners=0.75cm] ($(x3) !1.5cm! (x2)$) -- (x3) -- (x4);

      \foreach \point in {x0,x1,x2,x3,x4}
      {
        \filldraw[black] (\point) circle(0.1);
      }

      \draw[black,line width=0.25cm] (-5,-7.5) -- (5,-7.5) -- (5,4) -- (-5,4) -- cycle;

    \end{tikzpicture}\hfill
  \caption{The four nodes, $x_{b,0}$ through $x_{b,4}$, are connected with arc-fillets and each sampling ellipse, $\mathcal{E}_{x_{b,0},x_{b,1}}$, $\mathcal{E}_{x_{b,1},x_{b,2}}$, $\mathcal{E}_{x_{b,2},x_{b,3}}$, and $\mathcal{E}_{x_{b,3},x_{b,4}}$, is shown in red, blue, green, and orange respectively.
           Note that the volume of $\mathcal{E}_{x_{b,0},x_{b,1}}$ is $0$, because $c_{best,0} = c_{min,0}$.}
  \label{fig:beacons}
\end{figure}
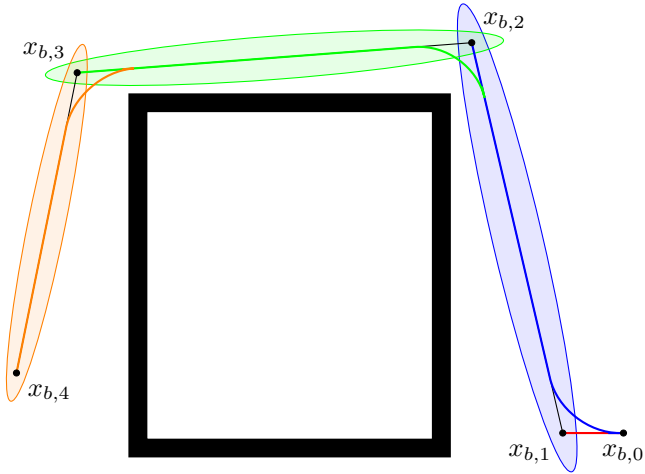

For each beacon ellipse, $c_{min,i}^\prime$ is defined as the distance between the two adjacent beacons that act as the focal points for that ellipse
\begin{equation*}
  c_{min,i}^\prime = \| x_{b,i+1} - x_{b,i} \|
\end{equation*}
and $c_{best,i}$ is the cost differential across the two beacons:
\begin{equation*}
  c_{best,i} = Cost(x_{b,i+1}) - Cost(x_{b,i})
\end{equation*}
Note that for straight-line primitives $c_{best,i} = c_{min,i}^\prime$, causing each ellipse to degenerate to the line between $x_{b,i}$ and $x_{b,i+1}$.
While sampling along this line can be beneficial, we show an example where it slows convergence because it neglects the exploration half of the exploration-exploitation paradigm.
With the fillets, the ellipse rarely degenerates to a straight-line.
Even when it does the $OptimizePath$ procedure removes the redundant intermediary points.
The smart-and-informed sampling procedure is now presented.
\begin{procedure}[$x_{rand} \leftarrow SI\text{-}Sample(i, b_t, b_b, X_b, X_t)$]
  Returns a random point at iteration $i \in \mathbb{Z}_+$ given the sampling biases $b_t \in \mathbb{Z}_+$ and $b_b \in \mathbb{Z}_+$, the beacon set $X_b$, and the target set $X_t$ as described in Algorithm \ref{alg:si_sampling}.
\end{procedure}
The $SI\text{-}Sample$ procedure includes a biasing towards the target set and a random sampling of the configuration space for sake of exploration.
Additionally, as is the case with \mbox{S-RRT*}'s sampling, some percentage of the samples must be drawn from the full configuration space to maintain asymptotic optimality.
This is because SI-RRT* makes no guaranties about its ellipses containing the optimal solution as \mbox{I-RRT*} does.
Sampling within the ellipses of \mbox{SI-RRT*} and the beacons of \mbox{S-RRT*} has the effect of biasing the search to locally refine the current best path.
The sampling of the beacon ellipses works in conjunction with the $OptimizePath$ procedure to seek improvements to the current best path found.
The $OptimizePath$ procedure works to straighten paths while the sampling of the beacon ellipses works to see if local perturbations will improve path length.
\begin{algorithm}[t]
  \caption{$x_{rand} \leftarrow\!SI\text{-}Sample(i,\!b_t,\!b_b,\!X_b,\!X_t)$}
  \label{alg:si_sampling}
  \begin{algorithmic}[1]
    \If{$i \divby b_{t}$} \Comment{Bias towards target set}
      \State $x_{rand} \leftarrow Sample(X_{t})$
    \ElsIf{$X_{b} \neq \varnothing \land i \divby b_{b}$} \Comment{Bias towards beacon ellipses}
      \State $x_{rand} \leftarrow Sample(\mathcal{E}_{X_b})$
    \Else \Comment{Otherwise sample a random point}
      \State $x_{rand} \leftarrow Sample(X)$
    \EndIf
    \State \Return $x_{rand}$
  \end{algorithmic}
\end{algorithm}

\subsection{Fillet-based Smart and Informed RRT*}
With the $SI\text{-}Sample$ procedure in hand, all of the components are in place for presenting the Fillet-Based Smart and Informed RRT* \mbox{(FB-SI-RRT*)} algorithm.
The \mbox{FB-SI-RRT*} algorithm is defined in Algorithm \ref{alg:si_rrt*}.
The straight-line counterpart can be expressed by removing the $FB$ prefix in all of the procedures.
Furthermore, the fillet generation procedure used in $FB\text{-}CostToCome$ can be replaced with either the arc or B\'ezier fillets.

The \mbox{FB-SI-RRT*} algorithm is very similar to the traditional RRT* algorithm as defined in Algorithm \ref{alg:rrt*}.
The major differences are the use of the fillet-based procedures in place of the straight-line counterparts and the $SI\text{-}Sample$ procedure in place of the $Biased\text{-}Sample$.
The path optimization procedure from \mbox{S-RRT*} is also included in \mbox{SI-RRT*} but is absent in traditional RRT*.

The \mbox{FB-SI-RRT*} algorithm begins on line \ref{ln:fb_initialize} by initializing the tree to consider the initial orientation of the vehicle and setting $X_b$ and $c_b$ to the empty set.
As with RRT*, a prefixed number of iterations is specified for refining the path.
Each iteration begins with the sampling of a new point using the $SI\text{-}Sample$ procedure.
This balances biasing towards the target set to find the goal, biasing towards the beacon set to refine the best path found, and sampling from the general obstacle free space for exploration.
The sampled point is then used within $FB\text{-}Extend^*$ in line \ref{ln:fb_extend} to grow the tree in the direction of the new sample while respecting the fillet continuity constraints.
If a connection to the tree is found, a new node is then inserted into the tree on line \ref{ln:add_node}.
The edge set is then rewired around the new node on line \ref{ln:fb_rewire} using the $FB\text{-}Rewire$ to consider the fillet continuity and cost improvement requirements.
If the first path or a shorter path to the target has been found, then the beacon set is updated in line \ref{ln:fb_beacons}.
If $c_b$ has changed, and thus $X_b$ has changed, a $FB\text{-}OptimizePath$ procedure is used on line \ref{ln:fb_opt_path} in an attempt to refine the beacon set.
Note that the $FB\text{-}OptimizePath$ has not been defined, but it can be expressed by changing line \ref{ln:opt_path_rewire} of Algorithm \ref{alg:optimize_path} to use the $FB\text{-}Rewire$ procedure.

\begin{algorithm}[t]
  \caption{\mbox{$X_{sol}\!\leftarrow\!FB\text{-}SI\text{-}RRT^*\!(x_{r},\!\psi_{r},\!X_{t},\!b_{t},\!b_{b},\!n)$}}
  \label{alg:si_rrt*}
  \begin{algorithmic}[1]
    \State $T \leftarrow FB\text{-}Initialize_{d_{init}}(x_{r},\psi_{r})$ \label{ln:fb_initialize}
    \State $X_{b} \leftarrow \varnothing$
    \State $c_b \leftarrow \varnothing$
    \For{$i = 1,\ldots,n$}
      \State $x_{rand} \leftarrow SI\text{-}Sample(i, b_t, b_b, X_b, X_t)$ \label{ln:si_sampling} \Comment{Smart-and-informed sampling}
      \State $\{x_{n},x_{p},c_n\} \leftarrow FB\text{-}Extend^*\!(x_{rand},T)$ \label{ln:fb_extend} \Comment{Extend tree towards new point}
      \If{$x_{n} \neq \varnothing$}
        \State $T \leftarrow InsertNode(x_{n},x_{p},T)$ \label{ln:add_node} \Comment{Add point to tree}
        \State $T \leftarrow FB\text{-}Rewire(x_{n},Near_{\rho,\alpha}(x_{n},T),T)$ \label{ln:fb_rewire} \Comment{Rewire edges around $x_{n}$}
        \If{$x_{n}\!\in\!X_{t} \land \left(X_{b} = \varnothing \lor Cost(x_{n},T) < c(X_{b})\right)$}
          \State $X_{b} \leftarrow Solution(x_{n},T)$ \label{ln:fb_beacons}  \Comment{Update best path found}
        \EndIf
        \If{$X_{b} \neq \varnothing \land c_b \neq c(X_{b})$} \Comment{If the best path has changed}
          \State $X_{b} \leftarrow FB\text{-}OptimizePath(X_{b},T)$ \label{ln:fb_opt_path}
          \State $c_b \leftarrow c(X_b)$
        \EndIf
      \EndIf
    \EndFor
    \State \Return $X_{b}$
  \end{algorithmic}
\end{algorithm}

Both an informed variation and smart variation to the \mbox{FB-RRT}* algorithm can be created with small variations to Algorithm \ref{alg:si_rrt*}. The informed algorithm can be formed by replacing line \ref{ln:si_sampling} with informed sampling and removing line \ref{ln:fb_opt_path}. The smart algorithm can be formed by replacing line \ref{ln:si_sampling} with smart sampling.

  \section{Examples}
  \label{sec:results}

A series of examples are now shown to demonstrate the Fillet-based RRT* approach.
Three environments were chosen to illustrate the performance of various RRT*-based planners.
Within each environment, 16 series of simulations were conducted to illustrate and evaluate the sampling and motion primitive variations.
\mbox{RRT*}, \mbox{I-RRT*}, \mbox{S-RRT*}, and \mbox{SI-RRT*} planned using straight-line, arc-fillet, B\'ezier-fillet, and Dubin's path motion primitives.

Note that a comparison between motion primitive types is not meant to show that one motion primitive is better than another.
Planning with straight-line paths is going to have better convergence characteristics due to their simplicity and lack of dynamic constraints.
In fact, the only primitives that can be directly compared in this fashion are arc-fillets and the Dubin's paths as both assume the same dynamic constraints.
The straight-line primitive is included to show a best-case scenario, providing a pseudo cost of considering the additional dynamic constraints.
This section proceeds with details on the simulations followed by a description of the different environments.
A comparison is made between the fillet formulation presented herein and that of \citep{spline_rrt*} followed by an example that justifies the need to consider curvature constraints while path planning.
Finally, the section ends with a discussion of the results.

\subsection{Simulation Details}
The obstacles are represented with an occupancy grid with each pixel corresponding to one square millimeter.
When performing obstacle collision checks clearance of $0.5m$ is required on all sides.
As orientation is well defined in the case of Dubin's paths and fillets, the points that are $0.5m$ to the right and left of the paths are checked.
In the case of straight-line paths, orientation is not well-defined at the nodes so points are checked $0.5m$ in every cardinal direction.
Paths are generated and checked at one-centimeter resolution.

The steering constant, $\eta$, and neighbor search radius, $\rho$, are both $3m$.
The max number of neighbors to search, $\alpha$, is $100$.
The check target period, $b_t$, is $50$ and the target set, $X_{t}$, is a circle of radius $0.1m$.
The Dubin's radius is set to $0.5m$ and likewise, the maximum curvature constraint imposed on the fillet generation is $2m^{-1}$.
The root node to first node distance, $d_{init}$, is $1m$. Note that these values are somewhat aggressive for some curvature constrained applications, but they allow the straight-line primitive to provide a tighter bound on the possible performance characteristics of the curvature constrained planners.
For \mbox{S-RRT*}, the beacon radius is $3m$ unless otherwise stated.
Note that this is not necessarily the best choice of beacon radius, as shown in Figure \ref{fig:s_rrt_beacon_size}.
However, smaller radius values are very detrimental to the Dubin's path results.
The beacon bias, $b_b$, is $3$ for both \mbox{S-RRT*} and \mbox{SI-RRT*}.

Results are gathered using OMPL \citep{ompl}.
Simulation code can be found in our open-source repository \url{https://gitlab.com/utahstate/robotics/fillet-rrt-star}.
As the sampling is random, each simulation series consists of 100 individual simulations with the average results being presented.
The results were gathered on an AMD Ryzen\texttrademark Threadripper\texttrademark 2990WX processor.
Convergence plots were made by fitting a 10\textsuperscript{th} order polynomial using a least-squares fitting algorithm as described in \citep{Venables2002}.

A least-squares approach is used as the sampling times for path length are not uniform across all simulations and not all simulations find the initial path at the same time.
Note that while the path length for any one run will be monotonically decreasing with time, the least squares fitted plot does not always have the same monotonic property.
The reason is that a particular run may not find a solution until well after other runs and the initial solution it finds may be much larger than the current solution of the other runs, effectively causing the average to increase at the time the run first produces path length data.

\subsection{Environments}
The three environments shown in Figure \ref{fig:all_worlds} were chosen to present and evaluate the performance of differing RRT* approaches.
The environments are referred to as the Spiral world, the Cluttered world, and the Maze world.

The Spiral world is made up of one narrow passage that twists around the starting point.
The world is $40m$ by $40m$, $x_r = \begin{bmatrix} 0 & 0 \end{bmatrix}^T$, and the center of $X_t$ is at $\begin{bmatrix} -15 & -15 \end{bmatrix}^T$.
The only path from $x_r$ to $X_t$ is through a narrow hallway forming a \enquote{bug trap} like set of obstacles.
The Spiral world tests planners' abilities to find a way out of the confined starting area and then converge through all of the passageways.
The \enquote{bug trap} like design makes it difficult for Dubin's paths based planners to find an initial solution, and \mbox{I-RRT*}'s cost heuristic is a poor estimate of $c_{min}$ in this environment. The environment is well suited for smart sampling as there is only one path and refinements are beneficial at each turn.

The Cluttered world is composed of many overlapping circular obstacles.
The world is $100m$ by $100m$, $x_r = \begin{bmatrix} -40 & -40 \end{bmatrix}^T$, and the center of $X_t$ is at $\begin{bmatrix} 40 & 40 \end{bmatrix}^T$.
The abundance of small obstacles results in many small local minima, but there are still large open areas for exploration.
The Cluttered world tests the planners' ability to break out of local minima.
\mbox{I-RRT*} based sampling is well-suited in the environment as the \mbox{I-RRT*} heuristic is a good estimate of the optimal path length.

The Maze world features a series of narrow passages and dead ends.
The world is $50m$ by $50m$, $x_r = \begin{bmatrix} -11 & -22.5 \end{bmatrix}^T$, and the center of $X_t$ is at $\begin{bmatrix} 2.5 & 12.5 \end{bmatrix}^T$.
The Maze world has fewer local minima than the Cluttered world and consists of long narrow corridors.
This world tests the planners' ability to find high-quality initial solutions quickly and then converge past those initial solutions.
While the local minima can be detrimental to beacon-based sampling, the optimal cost heuristic in \mbox{I-RRT*}'s sampling is poor in this case, proving detrimental to I-RRT*'s convergence.

\begin{sidewaysfigure}[p]
  \newcommand{\W}{0.325\linewidth}
  \begin{subfigure}[t]{\W}
    \centering
    \resizebox{\linewidth}{!}{
      \begin{tikzpicture}
        \begin{axis}[
          enlargelimits=false,
          axis lines=none,
          xmin=-20,ymin=-20,xmax=20,ymax=20,
          scale only axis,
          axis equal image,
          axis equal=true,
          ]
          \addplot graphics [xmin=-20,ymin=-20,xmax=20,ymax=20] {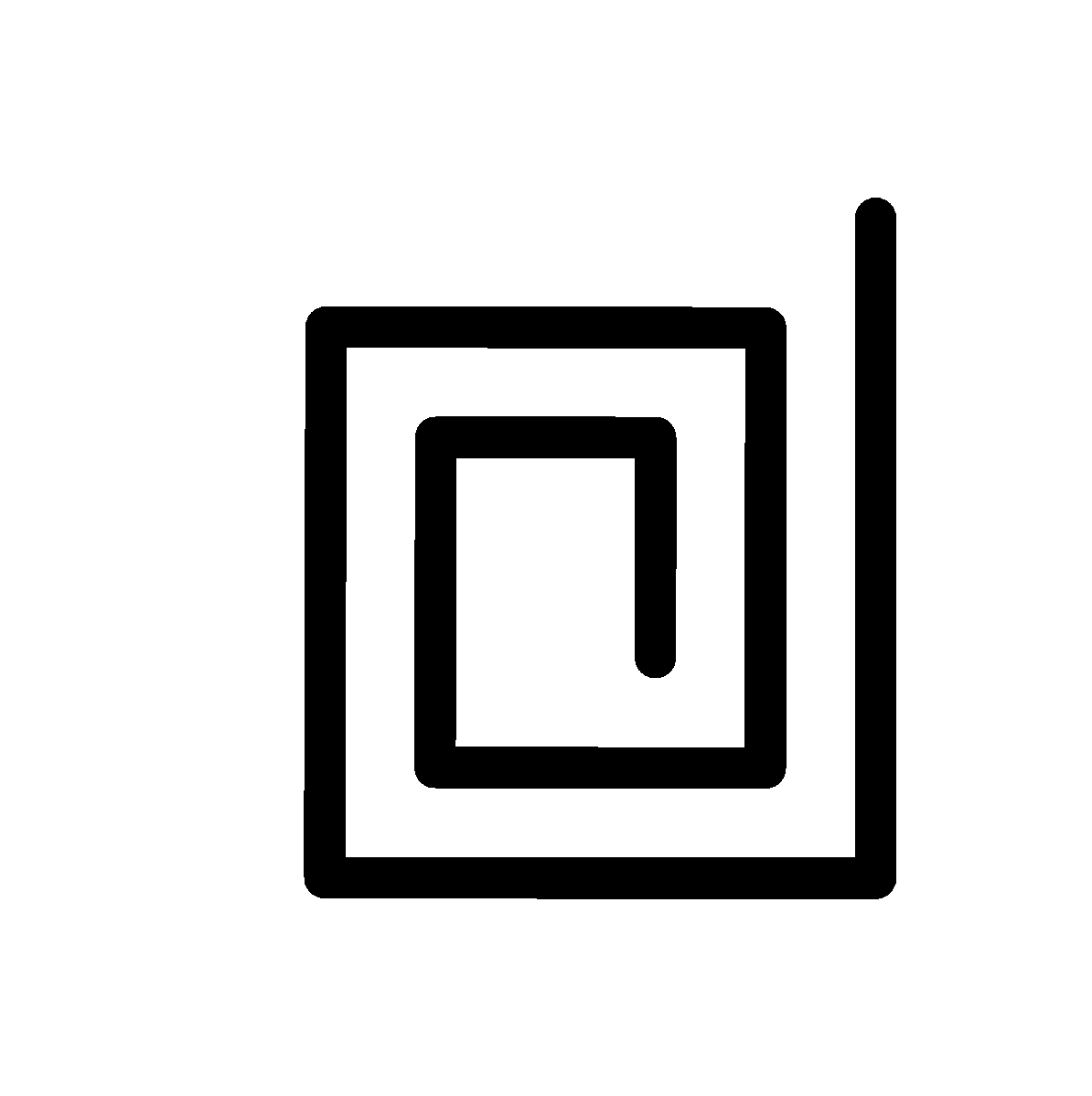};

          \addplot[color=green] table [x=x,y=y, col sep=comma] {pics/spiral/straight_line.csv};
          \addplot[color=red]   table [x=x,y=y, col sep=comma] {pics/spiral/dubins.csv};
          \addplot[color=blue]  table [x=x,y=y, col sep=comma] {pics/spiral/arc_fillet.csv};
          \addplot[color=brown] table [x=x,y=y, col sep=comma] {pics/spiral/bezier_fillet.csv};

          \node[label={left: {\small$x_r$}},inner sep=0pt] () at (0,0) {};
          \node[label={right:{\small$X_t$}},inner sep=0pt] () at (-15,-15) {};
          \draw[thick] (-20,-20) -- (-20,20) -- (20,20) -- (20,-20) -- cycle;
        \end{axis}
      \end{tikzpicture}
    }
  \end{subfigure}
  \begin{subfigure}[t]{\W}
    \centering
    \resizebox{\linewidth}{!}{
      \begin{tikzpicture}
        \begin{axis}[
          enlargelimits=false,
          axis lines=none,
          xmin=-50,ymin=-50,xmax=50,ymax=50,
          scale only axis,
          axis equal image,
          axis equal=true,
          ]
          \addplot graphics [xmin=-50,ymin=-50,xmax=50,ymax=50] {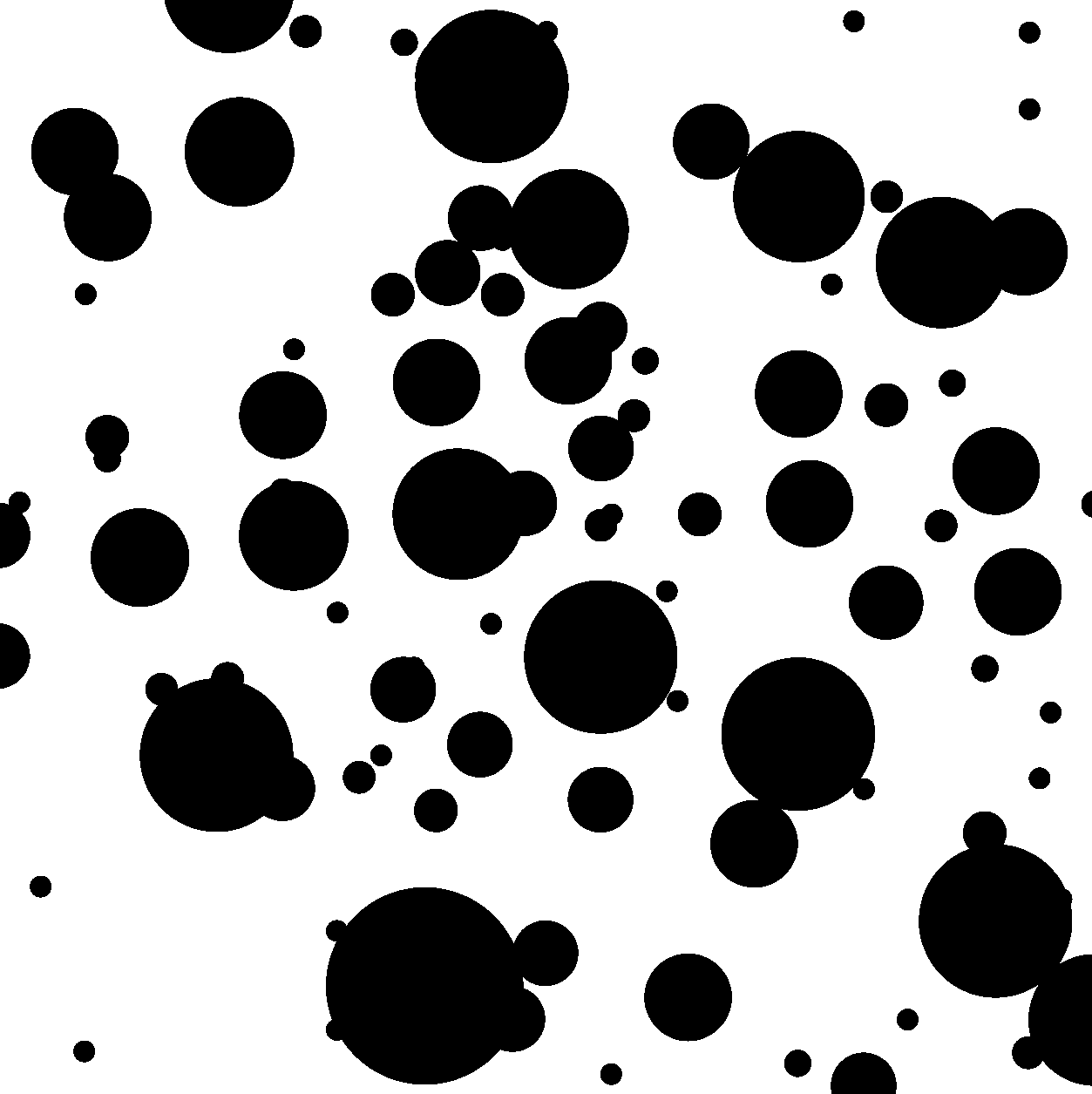};

          \addplot[color=green] table [x=x,y=y, col sep=comma] {pics/cluttered/straight_line.csv};
          \addplot[color=red]   table [x=x,y=y, col sep=comma] {pics/cluttered/dubins.csv};
          \addplot[color=blue]  table [x=x,y=y, col sep=comma] {pics/cluttered/arc_fillet.csv};
          \addplot[color=brown] table [x=x,y=y, col sep=comma] {pics/cluttered/bezier_fillet.csv};

          \node[label={left: {\small$x_r$}},inner sep=0pt] () at (-39.5,-40) {};
          \node[label={above:{\small$X_t$}},inner sep=0pt] () at (40,39.5) {};
          \draw[thick] (-50,-50) -- (-50,50) -- (50,50) -- (50,-50) -- cycle;
        \end{axis}
      \end{tikzpicture}
    }
  \end{subfigure}
  \begin{subfigure}[t]{\W}
    \centering
    \resizebox{\linewidth}{!}{
      \begin{tikzpicture}
        \begin{axis}[
          enlargelimits=false,
          axis lines=none,
          xmin=-25,ymin=-25,xmax=25,ymax=25,
          scale only axis,
          axis equal image,
          axis equal=true,
          ]
          \addplot graphics [xmin=-25,ymin=-25,xmax=25,ymax=25] {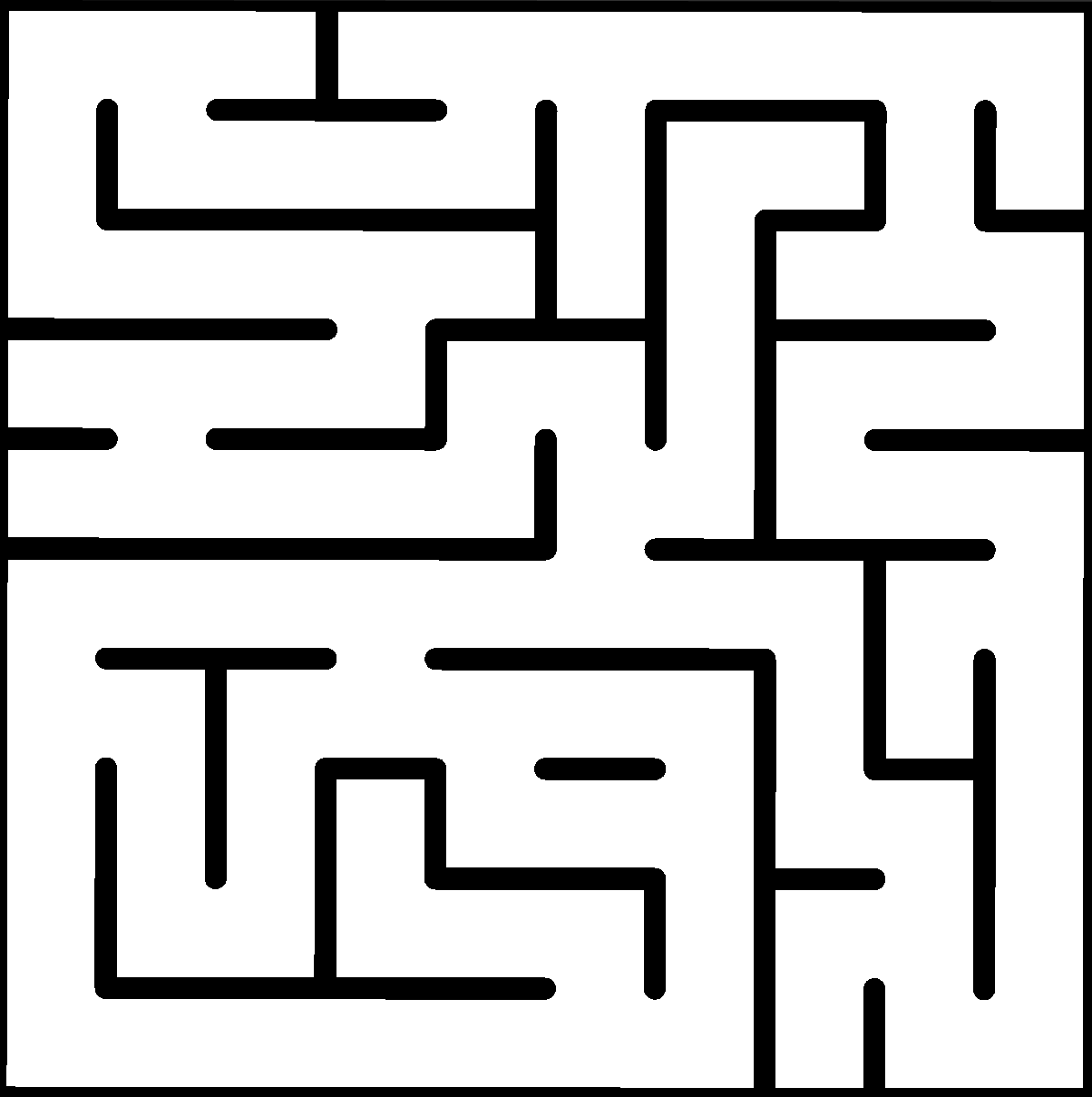};

          \addplot[color=green] table [x=x,y=y, col sep=comma] {pics/maze/straight_line.csv};
          \addplot[color=red]   table [x=x,y=y, col sep=comma] {pics/maze/dubins.csv};
          \addplot[color=blue]  table [x=x,y=y, col sep=comma] {pics/maze/arc_fillet.csv};
          \addplot[color=brown] table [x=x,y=y, col sep=comma] {pics/maze/bezier_fillet.csv};

          \node[label={left: {\small$x_r$}},inner sep=0pt] () at (-10.5,-22.5) {};
          \node[label={below:{\small$X_t$}},inner sep=0pt] () at (2.5,13.5) {};
          \draw[thick] (-25,-25) -- (-25,25) -- (25,25) -- (25,-25) -- cycle;
        \end{axis}
      \end{tikzpicture}
    }
  \end{subfigure}
  \caption{The resulting paths from running RRT* for 30 seconds in the Spiral world (left), Cluttered world (middle), and the Maze world (right).
           Straight-line, Dubin's path, Arc-fillet, and B\'ezeir-fillet paths are shown in green, red, blue, and brown respectively.}
  \label{fig:all_worlds}
\end{sidewaysfigure}

\subsection{Comparison with Previous Work}
\Revision{
Note that this work's \mbox{FB-RRT*} differs from what is given in \citep{Yang2014,spline_rrt*} in the following ways:
\begin{itemize}
  \item The generalization of \citep{Yang2014,spline_rrt*}'s $\gamma_{max}$ and $d_{min}$ based path continuity constraints to the less restrictive form given in \eqref{equ:conditions}.
        See Figure \ref{fig:spline_reach} for an illustration of how \eqref{equ:conditions} is less restrictive than using the $\gamma_{max}$/$d_{min}$ constraints.
  \item The addition of a pseudo ``node orientation'' in the nearest neighbor search heuristically penalizes turns and enables us to forgo the k-nearest-neighbor search that \citep{spline_rrt*} uses.
        See the explanation of $FB\text{-}Extend$ in Section \ref{sec:procedures_with_minor_changes} for more information.
  \item A newly developed rewiring procedure that ensures continuity and cost improvement in the resulting path.
        See Section \ref{sec:fb_rewire} for more information on $FB\text{-}Rewire$.
  \item The generalization of the fillet-based planner structure to make use of any fillet type instead of just B\'ezeir-fillets.
\end{itemize}
}

This section uses convergence plots to compare the formulation of the fillet constraints in this work to the formulation given in \citep{Yang2014,spline_rrt*}.
Specifically, instead of constraining node addition in the tree with \eqref{equ:conditions} we use the constants given in \citep{Yang2014}.
\citep{Yang2014} defines a max node-to-node angle, $\gamma_{max}$, and then uses that angle to define a minimum node-to-node distance, $d_{min} = d\!\left(\gamma_{max}\right)$.
Any nodes that form an angle greater than $\gamma_{max}$ or are closer together than $2d_{min}$ are deemed invalid.
This is a conservative approximation of the constraints defined in \eqref{equ:conditions}\Revision{, see Section \ref{section:bezier_curve_gen} for more details}.
We refer to the version of \mbox{FB-RRT*} that uses \citep{Yang2014}'s constraints as \mbox{SB-RRT*}.

Note that \mbox{SB-RRT*} differs from what is given in \citep{spline_rrt*} because the $FB\text{-}Rewire$ procedure is used to avoid the invalid tree configurations that result from the $Rewire$ procedure in \citep{spline_rrt*}, see Figure \ref{fig:spline_rewire_counter}.
The B\'ezeir-fillet is used for comparison because that is the fillet used in \citep{Yang2014,spline_rrt*}.
The Cluttered world was chosen for this simulation because it is similar to the simulated environment used in \citep{Yang2014}.
The max allowed curvature is kept at $2m^{-1}$, $\gamma_{max} = \frac{\pi}{2} rad.$, and $d_{min} = 0.7938 m$.

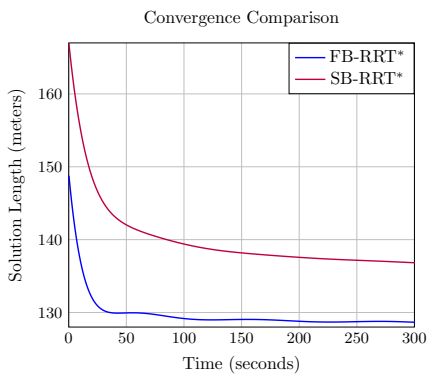
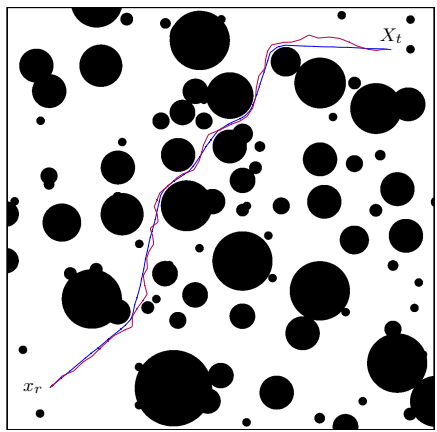
\begin{figure}[h]
  \centering
  \newcommand{\W}{0.49\linewidth}
  \begin{subfigure}[t]{\W}
    \centering
    \resizebox{\linewidth}{!}{
      \begin{tikzpicture}
        \begin{axis}[
          ymin=128,
          ymax=167,
          xmin=0,
          xmax=300,
          ticklabel style={font=\small},
          grid=both,
          title={Convergence Comparison},
          ylabel={Solution Length (meters)},
          xlabel={Time (seconds)},
          ytick style={opacity=0},
          xtick style={opacity=0},
          ylabel near ticks,
          xlabel near ticks,
          legend style={nodes={scale=1, transform shape},at={(1,1)},anchor=north east},
          ]

          \addplot[color=blue,  thick] table [x=x,y=y, col sep=comma] {sim_data/cluttered_old_vs_new/fb_rrt_star.csv};
          \addplot[color=purple,thick] table [x=x,y=y, col sep=comma] {sim_data/cluttered_old_vs_new/old_fb_rrt_star.csv};
          \legend{FB-RRT$^*$, SB-RRT$^*$}
       \end{axis}
      \end{tikzpicture}
    }
    \caption{Convergence of solution cost in the Cluttered world over 5 minutes.
             \mbox{FB-RRT*} is shown in blue and \mbox{SB-RRT*} is shown in purple.}
    \label{fig:sb_rrt_convergence}
  \end{subfigure}
  \begin{subfigure}[t]{\W}
    \centering
    \resizebox{\linewidth}{!}{
      \begin{tikzpicture}
        \begin{axis}[
          enlargelimits=false,
          axis lines=none,
          xmin=-50,ymin=-50,xmax=50,ymax=50,
          scale only axis,
          axis equal image,
          axis equal=true,
          ]
          \addplot graphics [xmin=-50,ymin=-50,xmax=50,ymax=50] {pics/cluttered/cluttered_world.png};

          \addplot[color=blue]   table [x=x,y=y, col sep=comma] {sim_data/cluttered_old_vs_new/fb_rrt_star_path.csv};
          \addplot[color=purple] table [x=x,y=y, col sep=comma] {sim_data/cluttered_old_vs_new/old_fb_rrt_star_path.csv};

          \node[label={left: {\small$x_r$}},inner sep=0pt] () at (-39.5,-40) {};
          \node[label={above:{\small$X_t$}},inner sep=0pt] () at (40,39.5) {};
          \draw[thick] (-50,-50) -- (-50,50) -- (50,50) -- (50,-50) -- cycle;
        \end{axis}
      \end{tikzpicture}
    }
    \caption{Example solutions found in the Cluttered world after 5 minutes.
             \mbox{FB-RRT*} is shown in blue and \mbox{SB-RRT*} is shown in purple.}
    \label{fig:sb_rrt_paths}
  \end{subfigure}
  \caption{A comparison of the solution quality found by \mbox{FB-RRT*} and \mbox{SB-RRT*}.
           \mbox{FB-RRT*} converges faster and to a shorter solution then \mbox{SB-RRT*}.
           \mbox{FB-RRT*} is shown in blue and \mbox{SB-RRT*} is shown in purple.}
  \label{fig:sb_rrt_comp}
\end{figure}

Figure \ref{fig:sb_rrt_convergence} shows the convergence of \mbox{SB-RRT*} and \mbox{FB-RRT*} in the scenario described.
\mbox{FB-RRT*} far outperforms \mbox{SB-RRT*} in both initial convergence speed and the solution to which it settles over time.
\mbox{FB-RRT*} converges \Revision{to a much shorter path} then \mbox{SB-RRT*} as the constraints given in \citep{Yang2014} force each node along the solution to be at least $2 d_{min}$ distance away from each other.
As the solution converges, it becomes difficult to \Revision{shorten the path} further without \Revision{reducing the number of nodes that make up} the solution path.
\Revision{
We emphasize that this simulation only shows the benefit of using the constraints in \eqref{equ:conditions}.
A significant benefit is also received from the updated rewiring procedure that ensures continuous paths are produced.
}

\subsection{Curvature Constrained Paths}
\label{sec:curvature_constrained_paths}
This section provides an example of when considering curvature constraints during path planning is advantageous.
A common approach to curvature constrained path planning is to plan using straight-line motion primitives and then smooth the path after planning; see, for example, Chapter 11 of \citep{Beard2012}.
However, this can lead to invalid paths.

\begin{figure}[h]
  \centering
  \resizebox{0.5\linewidth}{!}{
    \begin{tikzpicture}
      \begin{axis}[
        enlargelimits=false,
        axis lines=none,
        xmin=-15,ymin=-15,xmax=15,ymax=15,
        scale only axis,
        axis equal image,
        axis equal=true,
        ]
        \clip  (15,15) rectangle (-15,-8);

        \addplot graphics [xmin=-15,ymin=-15,xmax=15,ymax=15] {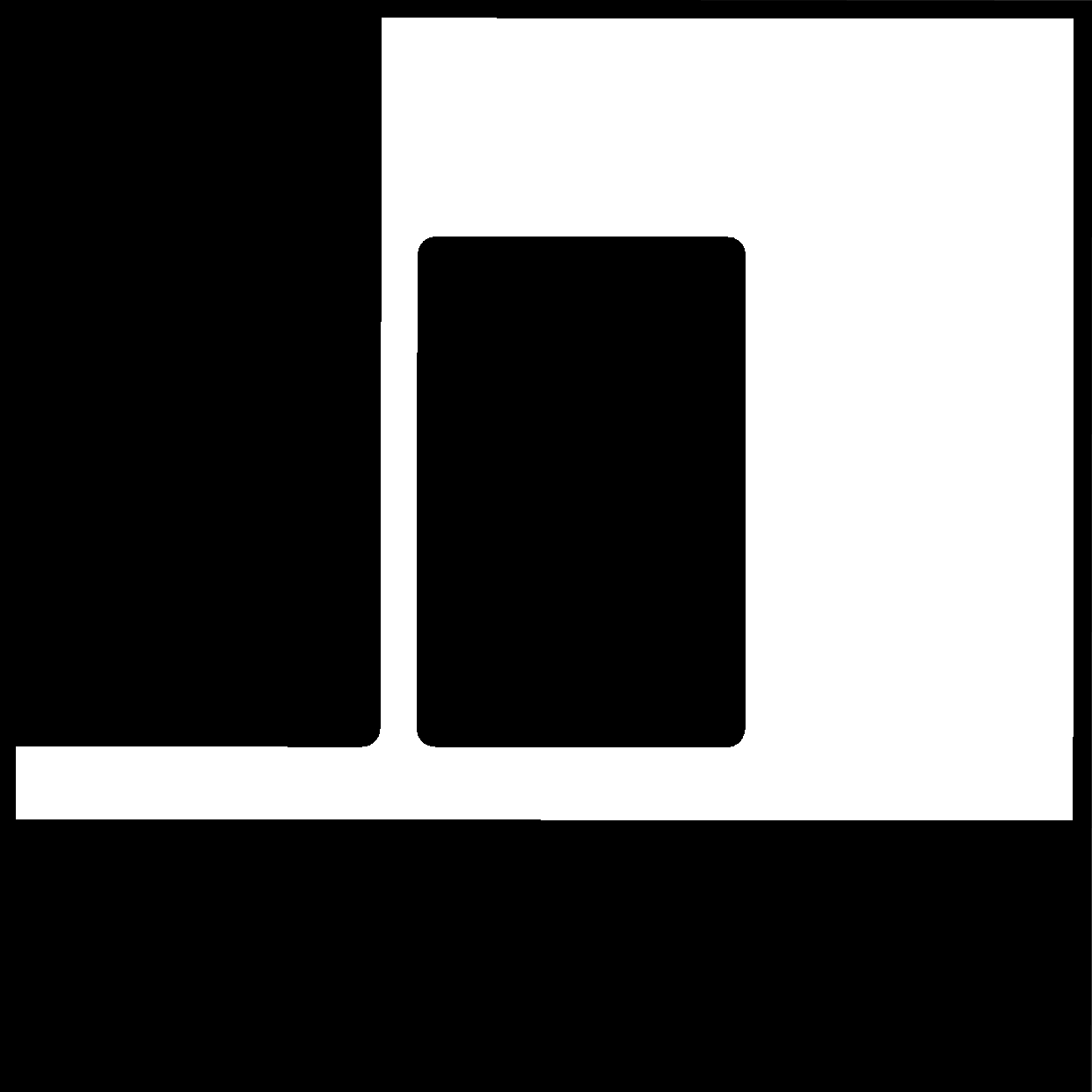};

        \addplot[color=green] table [x=x,y=y, col sep=comma] {pics/curvature_constrained_path/straight_line.csv};
        \addplot[color=blue]  table [x=x,y=y, col sep=comma] {pics/curvature_constrained_path/arc_fillet.csv};

        \node[label={left: {\small$x_r$}},inner sep=0pt] () at (-12.5,-6.1) {};
        \node[label={above:{\small$X_t$}},inner sep=0pt] () at (-3,12) {};

      \end{axis}
    \end{tikzpicture}
  }
  \caption{The resulting paths from running RRT* in an environment where the straight-line path turns sharply down a corridor.
           Straight-line and arc-fillet paths are shown in green and blue respectively.}
  \label{fig:curvature_demo_world}
\end{figure}

One of the scenarios where this would be a problem is shown in Figure \ref{fig:curvature_demo_world}.
Path planning using straight-line primitives converges to the solution shown in green.
This is the shortest path between the start and goal locations while ensuring a minimum distance is maintained between the path and obstacles.

However, a path that goes through this hallway can not satisfy the curvature constraints of the problem without hitting the walls or performing a complex, multi-turn maneuver.
\Revision{
See Appendix \ref{sec:reverse_fillet} for details on generating a multi-turn maneuver with \mbox{FB-RRT*}.
}
If curvature constants are considered during path planning, as is the case for \mbox{FB-RRT*}, then the path planner will return a solution through the second corridor, which is wide enough to make the turn.

\subsection{Results}
The different motion primitives are now compared in terms of the initial solution, the convergence, and the effects of the various sampling techniques.

\subsubsection{Initial Solution Comparison}
Table \ref{tab:initialization_results} shows the performance of each motion primitive in finding the initial solution in terms of the mean time and length of the initial solution for each environment.
The results in Table \ref{tab:initialization_results} apply regardless of the sampling technique used as each sampling approach behaves identically to RRT* before the first solution is found.

The Spiral and Maze worlds take longer on average to find an initial solution than the Cluttered world.
There are a couple of reasons that this may occur.
Both worlds have more narrow passages, resulting in a large number of iterations that fail to extend the tree because their paths become invalidated by obstacles.
Furthermore, the Maze world has many dead-ends, allowing the tree to spend time growing in a direction that will not lead to the target set.

Planners using the arc-fillet primitive find an initial solution in a time comparable to that of the planners using straight-line paths in all of the environments considered, performing slightly faster in the Maze world.
Similarly, the path lengths found are comparable, with the straight-line based paths being slightly shorter in all cases.
The added cost of considering continuity in curvature is seen as the planners using B\'ezier-fillets require roughly twice the time to find an initial solution as their arc-fillet counterparts in all three environments.
Table \ref{tab:init_cost} shows the B\'ezeir-fillet based planners finding a slightly shorter initial solution than their arc-fillet counterparts.
This may be due to the fact that the B\'ezier-fillet planners take more time and iterations to find an initial solution.
During that same time, the planners using arc-fillets are refining their solutions, always having a shorter path than the B\'ezier-fillet planners.

It takes significantly longer for planners using Dubin's paths to find an initial solution in each of the worlds tested here.
The only world where Dubin's paths based planners found an initial solution in a comparable amount of time is the Cluttered world, although the initial path is also significantly longer.
This is, in part, due to loops in the path, similar to that depicted in Figure \ref{fig:fillet_vs_dubins}.
In the Cluttered world, the first path found has many loops and bends, resulting in a longer path length on average.
These loops prove even more detrimental in the Spiral and Maze worlds with their narrow passageways.
The path length does not increase as dramatically for Dubin's paths in the Spiral and Maze worlds as it does in the Cluttered world because there is not as much room for looping paths.
However, there is a significant impact to the initial solution time as planning with Dubin's paths requires over 2 to 9 times as long as B\`ezier-fillets and 5 to 15 times as long as arc-fillets.

It is important to note that the least squares fitting distorts the average transients plots in Figure \ref{fig:transients}.
The average initial solution length for each planner is equivalent for any particular motion primitive.
However, the rapid convergence of some planners cause the least squares solution to appear lower at the initial solution time.

\begin{table}[t]
  \caption{Initial solution results.}
  \label{tab:initialization_results}
  \centering
  \begin{subtable}{\linewidth}
    \centering
      \begin{tabular}{|c|c|c|c|}
        \hline
        Motion Primitive & Spiral              & Cluttered         & Maze                \\ \hline
        Straight-line    & $\,\,\, 4.328 \pm \,\,\, 1.158$ & $0.498 \pm 0.196$ & $11.183 \pm \,\,\, 3.455$ \\ \hline
        Dubin's path     & $96.792       \pm 38.124$       & $1.585 \pm 1.357$ & $40.849 \pm 17.587$ \\ \hline
        Arc-fillet       & $\,\,\, 6.373 \pm \,\,\, 2.124$ & $0.586 \pm 0.527$ & $\,\,\, 7.473 \pm \,\,\, 2.057$ \\ \hline
        B\'ezeir-fillet  & $10.288       \pm \,\,\, 3.292$ & $1.081 \pm 1.966$ & $14.002 \pm \,\,\, 4.04 \,\,\,$ \\ \hline
      \end{tabular}
    \caption{The means and standard deviations of the time in seconds it takes for RRT* to find an initial solution.}
    \label{tab:init_time}
  \end{subtable}
  \begin{subtable}{\linewidth}
    \centering
    \begin{tabular}{|c|c|c|c|}
      \hline
      Motion Primitive & Spiral            & Cluttered          & Maze              \\ \hline
      Straight-line    & $136.05 \pm 3.91$ & $167.42 \pm 12.14$ & $142.25 \pm 7.84$ \\ \hline
      Dubin's path     & $161.68 \pm 7.73$ & $275.07 \pm 39.55$ & $162.65 \pm 9.86$ \\ \hline
      Arc-fillet       & $144.45 \pm 6.47$ & $184.54 \pm 21.81$ & $148.56 \pm 8.4 \,\,\,$ \\ \hline
      B\'ezeir-fillet  & $137.76 \pm 5.45$ & $181.45 \pm 23.76$ & $147.2 \,\,\, \pm 8.03$ \\ \hline
    \end{tabular}
    \caption{The means and standard deviations of the initial path length in meters after RRT* has found an initial solution.}
    \label{tab:init_cost}
  \end{subtable}
\end{table}

\input{tikz_pics/convergence_plots.tex}
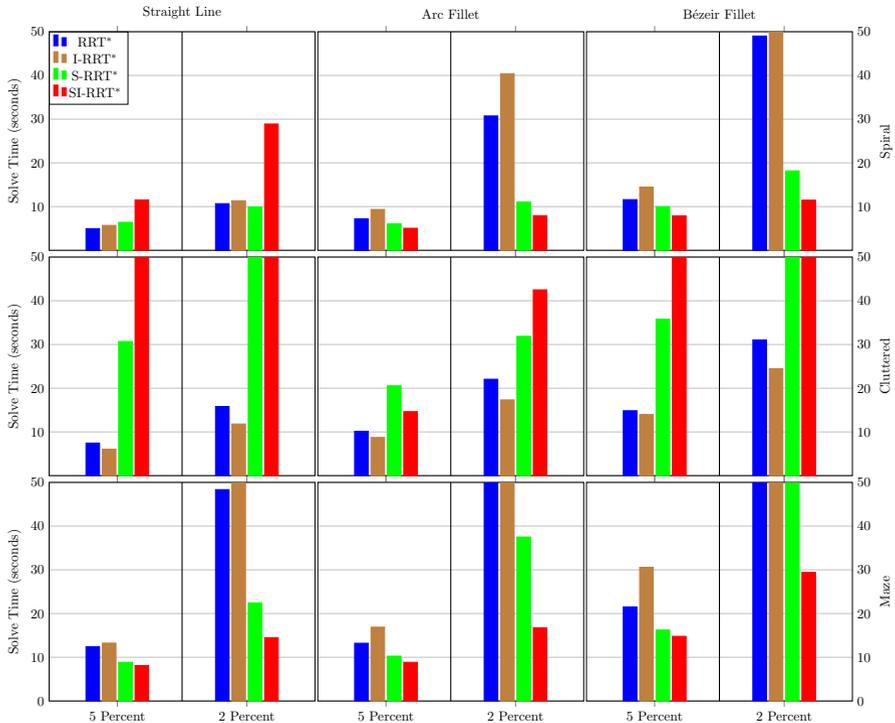
\begin{figure*}[ht]
  \centering
  \resizebox{\linewidth}{!}{
    \begin{tikzpicture}
      \begin{groupplot}[
        ybar=2.5cm,
        xticklabels={,,,},
        group style={
          group size=3 by 3,
          vertical sep=5pt,
          horizontal sep=2pt,},
        ymin=0,
        ymax=50,
        samples=2,
        y tick style={draw=none},
        ymajorgrids = true,
        bar shift auto,
        ylabel near ticks,
        xlabel near ticks,
        xtick=data,
        enlarge x limits={abs=5*\pgfplotbarwidth},
        legend style={nodes={scale=1, transform shape},at={(0,1)},anchor=north west}
        ]

        \nextgroupplot[
          title={Straight Line},
          ylabel={Solve Time (seconds)},
          yticklabels={,,10,20,...,50},
          ]
        \addplot[fill,color=blue]  coordinates {(1, 4.9657) (2, 10.6723)};
        \addplot[fill,color=brown] coordinates {(1, 5.7059) (2, 11.3475)};
        \addplot[fill,color=green] coordinates {(1, 6.4003) (2, 9.8973)};
        \addplot[fill,color=red]   coordinates {(1, 11.5416) (2, 28.8949)};
        \legend{RRT$^*$, I-RRT$^*$, S-RRT$^*$, SI-RRT$^*$}
        \draw[black,thick] (1.5,0) -- (1.5,50);

        \nextgroupplot[
          title={Arc Fillet},
          yticklabels={,,},
          ]
        \addplot[fill,color=blue]  coordinates {(1, 7.1999) (2, 30.736)};
        \addplot[fill,color=brown] coordinates {(1, 9.3522) (2, 40.4054)};
        \addplot[fill,color=green] coordinates {(1, 6.0875) (2, 11.0555)};
        \addplot[fill,color=red]   coordinates {(1, 5.0466) (2, 7.9373)};
        \draw[black,thick] (1.5,0) -- (1.5,50);

        \nextgroupplot[
          title={B\'ezeir Fillet},
          yticklabel pos=right,
          ylabel=Spiral,
          yticklabels={,,10,20,...,50},
          ylabel near ticks,
          ]
        \addplot[fill,color=blue]  coordinates {(1, 11.5964) (2, 49.0211)};
        \addplot[fill,color=brown] coordinates {(1, 14.4602) (2, 60)};
        \addplot[fill,color=green] coordinates {(1, 9.9919) (2, 18.1695)};
        \addplot[fill,color=red]   coordinates {(1, 7.8861) (2, 11.4723)};
        \draw[black,thick] (1.5,0) -- (1.5,50);

        \nextgroupplot[
          ylabel={Solve Time (seconds)},
          yticklabels={,,10,20,...,50},
          ]
        \addplot[fill,color=blue]  coordinates {(1, 7.4537) (2, 15.8235)};
        \addplot[fill,color=brown] coordinates {(1, 6.0504) (2, 11.8192)};
        \addplot[fill,color=green] coordinates {(1, 30.6909) (2, 50.5781)};
        \addplot[fill,color=red]   coordinates {(1, 60) (2, 60)};
        \draw[black,thick] (1.5,0) -- (1.5,50);

        \nextgroupplot[
          yticklabels={,,},
          ]
        \addplot[fill,color=blue]  coordinates {(1, 10.1554) (2, 22.0524)};
        \addplot[fill,color=brown] coordinates {(1, 8.7317) (2, 17.3256)};
        \addplot[fill,color=green] coordinates {(1, 20.5472) (2, 31.8692)};
        \addplot[fill,color=red]   coordinates {(1, 14.6412) (2, 42.502)};
        \draw[black,thick] (1.5,0) -- (1.5,50);

        \nextgroupplot[
          yticklabel pos=right,
          ylabel=Cluttered,
          yticklabels={,,10,20,...,50},
          ylabel near ticks,
          ]
        \addplot[fill,color=blue]  coordinates {(1, 14.8482) (2, 31.0479)};
        \addplot[fill,color=brown] coordinates {(1, 13.9818) (2, 24.4715)};
        \addplot[fill,color=green] coordinates {(1, 35.7861) (2, 60)};
        \addplot[fill,color=red]   coordinates {(1, 60) (2, 60)};
        \draw[black,thick] (1.5,0) -- (1.5,50);

        \nextgroupplot[
          ylabel={Solve Time (seconds)},
          xticklabels={{5 Percent},{2 Percent}},
          ]
        \addplot[fill,color=blue]  coordinates {(1, 12.4019) (2, 48.3512)};
        \addplot[fill,color=brown] coordinates {(1, 13.2362) (2, 60)};
        \addplot[fill,color=green] coordinates {(1, 8.8587) (2, 22.4392)};
        \addplot[fill,color=red]   coordinates {(1, 8.0956) (2, 14.4597)};
        \draw[black,thick] (1.5,0) -- (1.5,50);

        \nextgroupplot[
          yticklabels={,,},
          xticklabels={{5 Percent},{2 Percent}},
          ]
        \addplot[fill,color=blue]  coordinates {(1, 13.2054) (2, 60)};
        \addplot[fill,color=brown] coordinates {(1, 16.8828) (2, 60)};
        \addplot[fill,color=green] coordinates {(1, 10.2612) (2, 37.4922)};
        \addplot[fill,color=red]   coordinates {(1, 8.8486) (2, 16.7486)};
        \draw[black,thick] (1.5,0) -- (1.5,50);

        \nextgroupplot[
          yticklabel pos=right,
          ylabel=Maze,
          ylabel near ticks,
          xticklabels={{5 Percent},{2 Percent}},
          ]
        \addplot[fill,color=blue]  coordinates {(1, 21.5203) (2, 60)};
        \addplot[fill,color=brown] coordinates {(1, 30.5757) (2, 60)};
        \addplot[fill,color=green] coordinates {(1, 16.2296) (2, 60)};
        \addplot[fill,color=red]   coordinates {(1, 14.7872) (2, 29.4243)};
        \draw[black,thick] (1.5,0) -- (1.5,50);
      \end{groupplot}
    \end{tikzpicture}
  }
  \caption{The average time it took each planner to find a solution that was within 5 and 2 percent of the best averaged solution found after 50 seconds.
           Images from left to right show the results of the straight-line, arc-fillet, and B\'ezeir-fillet motion primitives.
           RRT* is shown in blue, I-RRT* in brown, S-RRT* in green, and SI-RRT* in red.}
  \label{fig:convergence_bar_graph}
\end{figure*}

\subsubsection{Convergence Times of Motion Primitives}
Once an initial solution is found, the algorithms work to converge on the shortest path.
In this section, the resulting path lengths generated from planning using different motion primitives are directly compared.
Note that this is inherently an unfair comparison as the primitives have different kinematic constraints.
The only two motion primitives that use the same kinematic constraints are Dubin's paths and arc-fillets.
The straight-line motion primitive does not respect curvature constraints while B\'ezeir-fillet primitive considers a curvature rate constraint in addition to the curvature constraint considered by the arc-fillet primitive.
Each scenario has been designed such that the optimal path for each motion primitive will be similar, unlike the scenario in Figure \ref{fig:curvature_demo_world}.
This allows the straight-line convergence rate to be a pseudo best-case solution.
The results will show that the fillet-based planners have comparable results to the straight-line planners.

Figure \ref{fig:transients} shows the performance of each motion primitive with the four different sampling techniques.
The green shaded areas start at the path length of the best averaged solution found for that world in 50 seconds.
The yellow shaded area starts once the path length is within 2 percent of the best averaged solution for the respective environment.
Similarly, the red shaded area starts once the path length is within 5 percent of the best averaged solution for the respective environment.

Table \ref{tab:best_planners} shows the best performing planners in each world and type of motion primitive.
In the Spiral world, the best averaged solution found within 50 seconds was found by \mbox{S-RRT*} planning with straight-lines with a path length of $116.14 m$.
For the Cluttered world, the best was found by \mbox{I-RRT*} planning with straight-lines with a path length of $128.79 m$.
For the Maze world, the best was found by \mbox{SI-RRT*} planning with straight-lines with a path length of $129.54 m$.
Note that in each case, the planner that found the shortest averaged solution was planning with straight-line motion primitives.
This is expected because the straight-line primitive is the least constrained.

\begin{table}[t]
  \caption{The best performing planner and associated average path length for that planner in each environment and motion primitive combination.
           Results are calculated after running the planners for 50 seconds.}
  \label{tab:best_planners}
  \centering
  \begin{tabular}{|c|c|c|c|c|c|c|}
    \hline
    Motion Primitive & \multicolumn{2}{c|}{Spiral} & \multicolumn{2}{c|}{Cluttered} & \multicolumn{2}{c|}{Maze} \\ \hline
    Straight-line    & \; S-RRT* & $116.14$        & I-RRT* & $128.79$              & SI-RRT* & $129.54$ \\ \hline
    Dubin's path     & SI-RRT*   & $122.6$ \;      & I-RRT* & $145.54$              & SI-RRT* & $143.17$ \\ \hline
    Arc-fillet       & SI-RRT*   & $116.34$        & I-RRT* & $129.46$              & SI-RRT* & $130.31$ \\ \hline
    B\'ezeir-fillet  & SI-RRT*   & $116.49$        & I-RRT* & $129.98$              & SI-RRT* & $130.92$ \\ \hline
  \end{tabular}
\end{table}

Figure \ref{fig:convergence_bar_graph} shows the 5 and 2 percent convergence times in a bar graph for a quick comparison of results.
It is worth noting that planners using the Dubin's path primitive struggle to approach the 5 percent convergence region and are subsequently left out of Figure \ref{fig:convergence_bar_graph}.
On the other hand, the arc-fillet planners perform quite well despite assuming the same motion constraints as the Dubin's primitive.
The arc-fillet based planners perform comparably, and in some cases better, than the straight-line primitives in converging to the 5 percent threshold.
Convergence begins to suffer for the 2 percent threshold, with some arc-fillet planners unable to cross that threshold in the Maze world.

The B\'ezier-fillet planners show an increase in convergence time, underscoring the cost of requiring continuity in curvature.
\mbox{SI-RRT*} is the only planner that is able to cross the 2 percent threshold when planning with B\`ezier-fillets in the Maze world.
However, when planning with B\'ezier-fillets in the Cluttered world, the smart sampling techniques (\mbox{S-RRT*} and \mbox{SI-RRT*}) fail to cross the 2 percent threshold while both \mbox{RRT*} and \mbox{I-RRT*} are able.

\subsubsection{The Effect of Sampling Techniques}
The environments have little effect on the trends of ranking the performance of the motion primitives.
The planning with straight-lines typically outperforms the arc-fillets, which outperforms the B\`ezier-fillets, which in turn outperforms planning with Dubin's paths.
However, the environment has a significant effect on the performance of the sampling procedures.

Table \ref{tab:best_planners} shows that each environment has a sampling procedure that works best in that environment.
For the Cluttered world, \mbox{I-RRT*} performs the best across all motion primitives.
Whereas in the Spiral and Maze worlds, the greedier sampling heuristics perform better.
In the Maze world, \mbox{SI-RRT*} performs the best for all motion primitives.
In the Spiral world, \mbox{S-RRT*} performs the best when planning with straight-lines but \mbox{SI-RRT*} performs best for all other motion primitives.
This shows that \mbox{S-RRT*}'s sampling works very well when planning with straight-lines but tends to struggle more when planning with kinematic constraints.

As mentioned, the Cluttered world is ideal for \mbox{I-RRT*}'s sampling approach and, as expected, the \mbox{I-RRT*} based planners perform the best in terms of 5 and 2 percent convergence.
However, the transient plots in Figure \ref{fig:transients} show a very interesting trend.
The smart approaches (\mbox{S-RRT*} and \mbox{SI-RRT*} based planners) show significantly faster initial convergence.
This is due to the fact that the smart approaches focus on refining the best solution instead of searching the environment, which also means that they tend to spend more time in local minima.
This is particularly noticeable in the Cluttered world where the smart approaches plateau for a time before dropping again.
The Spiral and Maze worlds present environments in which there are fewer local minima and as such the smart approaches continue rapidly refining the solution until the 5 and 2 percent thresholds.

Figures \ref{fig:transients} and \ref{fig:convergence_bar_graph} show that the \mbox{SI-RRT*} based planners focus heavily on refining the current shortest solution.
In the Cluttered world, this proves detrimental as it causes the planner to focus on local minima instead of searching for other paths through the obstacle topology.
\mbox{SI-RRT*}'s greedy convergence to local solutions is why \mbox{SI-RRT*}'s solution cost plateaus in the Cluttered world, see Figure \ref{fig:transients}.
In the Maze and Spiral worlds, it proves beneficial and results in the fastest convergence, both initially and to the 5 and 2 percent thresholds for all but the straight-line motion primitives.
This may be because there are fewer local minima in the obstacle topology of these worlds.
As expected, the smart-and-informed sampling does quite poorly for straight-line primitives.

Figure \ref{fig:s_rrt_beacon_size} shows that the \mbox{SI-RRT*} approach provides greedy refinement for fillet-based primitives without requiring the tuning of an extra beacon-size parameter.
Figure \ref{fig:s_rrt_beacon_size} compares the results from \mbox{S-RRT*} based planners over various beacon sizes to the results of the respective \mbox{SI-RRT*} based planner.
In the Cluttered world, \mbox{SI-RRT*} spends most of its time refining local minima and results in the worst convergence times.
In both the Maze and the Spiral worlds, the smart-and-informed sampling performs near the best, except when considering straight-line primitives.
While an iterative search over beacon sizes for \mbox{S-RRT*} can produce similar convergence results to \mbox{SI-RRT*}, the smart-and-informed sampling does not require additional tuning to find the best beacon size.
Note that the resulting best beacon size for \mbox{S-RRT*} sampling is dependent upon both the environment and the motion primitive, making such a search difficult prior to execution.

\begin{figure*}[ht]
  \centering
  \resizebox{\linewidth}{!}{
    \begin{tikzpicture}
      \begin{groupplot}[
        ybar=0.01cm,
        xticklabels={{5 Percent},{2 Percent}},
        group style={
          group size=3 by 3,
          vertical sep=5pt,
          horizontal sep=2pt,
        },
        ymin=0,
        ymax=50,
        y tick style={draw=none},
        ymajorgrids = true,
        ylabel near ticks,
        xlabel near ticks,
        xtick=data,
        legend style={nodes={scale=1, transform shape},at={(0,1)},anchor=north west},
        enlarge x limits={abs=5*\pgfplotbarwidth},
        ]

        \nextgroupplot[
          title={Straight Line},
          ylabel={Solve Time (seconds)},
          yticklabels={,,10,20,...,50},
          xticklabels={,,},
          ]
        \addplot[fill,color=brown]   coordinates {(1, 8.0463) (2, 12.331)};
        \addplot[fill,color=cyan]    coordinates {(1, 7.0806) (2, 10.6848)};
        \addplot[fill,color=magenta] coordinates {(1, 6.965) (2, 10.3689)};
        \addplot[fill,color=green]   coordinates {(1, 6.4003) (2, 9.8973)};
        \addplot[fill,color=orange]  coordinates {(1, 5.1057) (2, 8.5723)};
        \addplot[fill,color=purple]  coordinates {(1, 4.1901) (2, 7.4445)};
        \addplot[fill,color=violet]  coordinates {(1, 3.9705) (2, 8.0274)};
        \addplot[fill,color=red]     coordinates {(1, 11.5416) (2, 28.8949)};
        \legend{S-RRT$^*\text{,}0.1$,S-RRT$^*\text{,}0.5$,S-RRT$^*\text{,}1$,S-RRT$^*\text{,}3$,S-RRT$^*\text{,}5$,S-RRT$^*\text{,}10$,S-RRT$^*\text{,}15$,SI-RRT$^*$}
        \draw[black,thick] (1.5,0) -- (1.5,50);

        \nextgroupplot[
          title={Arc Fillet},
          yticklabels={,,},
          xticklabels={,,},
          ]
        \addplot[fill,color=brown]   coordinates {(1, 6.7317) (2, 24.8239)};
        \addplot[fill,color=cyan]    coordinates {(1, 7.2282) (2, 11.8611)};
        \addplot[fill,color=magenta] coordinates {(1, 6.5875) (2, 10.4048)};
        \addplot[fill,color=green]   coordinates {(1, 6.0875) (2, 11.0555)};
        \addplot[fill,color=orange]  coordinates {(1, 6.0147) (2, 13.5343)};
        \addplot[fill,color=purple]  coordinates {(1, 5.2763) (2, 17.5104)};
        \addplot[fill,color=violet]  coordinates {(1, 5.6666) (2, 19.5439)};
        \addplot[fill,color=red]     coordinates {(1, 5.0466) (2, 7.9373)};
        \draw[black,thick] (1.5,0) -- (1.5,50);

        \nextgroupplot[
          title={B\'ezeir Fillet},
          yticklabel pos=right,
          ylabel=Spiral,
          xticklabels={,,},
          ylabel near ticks,
          yticklabels={,,10,20,...,50},
          ]
        \addplot[fill,color=brown]   coordinates {(1, 9.9115) (2, 60)};
        \addplot[fill,color=cyan]    coordinates {(1, 9.3817) (2, 18.7032)};
        \addplot[fill,color=magenta] coordinates {(1, 9.5014) (2, 14.5848)};
        \addplot[fill,color=green]   coordinates {(1, 9.9919) (2, 18.1695)};
        \addplot[fill,color=orange]  coordinates {(1, 8.8802) (2, 22.896)};
        \addplot[fill,color=purple]  coordinates {(1, 7.5298) (2, 32.6979)};
        \addplot[fill,color=violet]  coordinates {(1, 8.1776) (2, 32.5038)};
        \addplot[fill,color=red]     coordinates {(1, 7.8861) (2, 11.4723)};
        \draw[black,thick] (1.5,0) -- (1.5,50);

        \nextgroupplot[
          ylabel={Solve Time (seconds)},
          yticklabels={,,10,20,...,50},
          xticklabels={,,},
          ]
        \addplot[fill,color=brown]   coordinates {(1, 12.3318) (2, 25.7019)};
        \addplot[fill,color=cyan]    coordinates {(1, 17.9295) (2, 32.7013)};
        \addplot[fill,color=magenta] coordinates {(1, 23.8988) (2, 38.146)};
        \addplot[fill,color=green]   coordinates {(1, 30.6909) (2, 60)};
        \addplot[fill,color=orange]  coordinates {(1, 19.127) (2, 45.0646)};
        \addplot[fill,color=purple]  coordinates {(1, 3.4543) (2, 12.2896)};
        \addplot[fill,color=violet]  coordinates {(1, 3.1376) (2, 7.0322)};
        \addplot[fill,color=red]     coordinates {(1, 60) (2, 60)};
        \draw[black,thick] (1.5,0) -- (1.5,50);

        \nextgroupplot[
          yticklabels={,,},
          xticklabels={,,},
          ]
        \addplot[fill,color=brown]   coordinates {(1, 9.6862) (2, 17.7157)};
        \addplot[fill,color=cyan]    coordinates {(1, 10.0484) (2, 17.6377)};
        \addplot[fill,color=magenta] coordinates {(1, 11.0817) (2, 18.4407)};
        \addplot[fill,color=green]   coordinates {(1, 20.5472) (2, 31.8692)};
        \addplot[fill,color=orange]  coordinates {(1, 16.1276) (2, 31.1069)};
        \addplot[fill,color=purple]  coordinates {(1, 7.1594) (2, 21.4392)};
        \addplot[fill,color=violet]  coordinates {(1, 4.5537) (2, 11.0599)};
        \addplot[fill,color=red]     coordinates {(1, 14.6412) (2, 42.502)};
        \draw[black,thick] (1.5,0) -- (1.5,50);

        \nextgroupplot[
          yticklabel pos=right,
          ylabel=Cluttered,
          ylabel near ticks,
          yticklabels={,,10,20,...,50},
          xticklabels={,,},
          ]
        \addplot[fill,color=brown]   coordinates {(1, 16.396) (2, 27.4526)};
        \addplot[fill,color=cyan]    coordinates {(1, 16.4866) (2, 29.1886)};
        \addplot[fill,color=magenta] coordinates {(1, 20.2217) (2, 33.4395)};
        \addplot[fill,color=green]   coordinates {(1, 35.7861) (2, 60)};
        \addplot[fill,color=orange]  coordinates {(1, 33.8309) (2, 60)};
        \addplot[fill,color=purple]  coordinates {(1, 15.8627) (2, 48.3957)};
        \addplot[fill,color=violet]  coordinates {(1, 7.9593) (2, 21.0883)};
        \addplot[fill,color=red]     coordinates {(1, 60) (2, 60)};
        \draw[black,thick] (1.5,0) -- (1.5,50);

        \nextgroupplot[
          ylabel={Solve Time (seconds)},
          ]
        \addplot[fill,color=brown]   coordinates {(1, 11.2885) (2, 26.7953)};
        \addplot[fill,color=cyan]    coordinates {(1, 5.7907) (2, 17.0431)};
        \addplot[fill,color=magenta] coordinates {(1, 7.3808) (2, 15.6016)};
        \addplot[fill,color=green]   coordinates {(1, 8.8587) (2, 22.4392)};
        \addplot[fill,color=orange]  coordinates {(1, 11.2695) (2, 32.6978)};
        \addplot[fill,color=purple]  coordinates {(1, 9.5233) (2, 35.9615)};
        \addplot[fill,color=violet]  coordinates {(1, 10.5697) (2, 38.5224)};
        \addplot[fill,color=red]     coordinates {(1, 8.0956) (2, 14.4597)};
        \draw[black,thick] (1.5,0) -- (1.5,50);

        \nextgroupplot[
          yticklabels={,,},
          ]
        \addplot[fill,color=brown]   coordinates {(1, 13.5596) (2, 60)};
        \addplot[fill,color=cyan]    coordinates {(1, 10.2883) (2, 22.6984)};
        \addplot[fill,color=magenta] coordinates {(1, 9.2505) (2, 20.5735)};
        \addplot[fill,color=green]   coordinates {(1, 10.2612) (2, 37.4922)};
        \addplot[fill,color=orange]  coordinates {(1, 11.1931) (2, 60)};
        \addplot[fill,color=purple]  coordinates {(1, 10.8287) (2, 60)};
        \addplot[fill,color=violet]  coordinates {(1, 12.0021) (2, 60)};
        \addplot[fill,color=red]     coordinates {(1, 8.8486) (2, 16.7486)};
        \draw[black,thick] (1.5,0) -- (1.5,50);

        \nextgroupplot[
          yticklabel pos=right,
          ylabel=Maze,
          ylabel near ticks,
          ]
        \addplot[fill,color=brown]   coordinates {(1, 23.5269) (2, 60)};
        \addplot[fill,color=cyan]    coordinates {(1, 17.4971) (2, 42.1236)};
        \addplot[fill,color=magenta] coordinates {(1, 15.2718) (2, 34.7308)};
        \addplot[fill,color=green]   coordinates {(1, 16.2296) (2, 60)};
        \addplot[fill,color=orange]  coordinates {(1, 17.6971) (2, 60)};
        \addplot[fill,color=purple]  coordinates {(1, 19.1168) (2, 60)};
        \addplot[fill,color=violet]  coordinates {(1, 19.4786) (2, 60)};
        \addplot[fill,color=red]     coordinates {(1, 14.7872) (2, 29.4243)};
        \draw[black,thick] (1.5,0) -- (1.5,50);
      \end{groupplot}
    \end{tikzpicture}
  }
  \caption{The average time it took each planner to find a solution what was within 5 and 2 percent of the best averaged solution found after 50 seconds.
           Images from left to right show the results of the straight-line, arc-fillet, and B\'ezeir-fillet motion primitives.
           S-RRT* with a beacon sizes of $0.1m$, $0.5m$, $1m$, $3m$, $5m$, $10m$, and $15m$ are shown in brown, cyan, magenta, green, orange, purple, and violet respectfully.
           SI-RRT* is shown in red.}
  \label{fig:s_rrt_beacon_size}
\end{figure*}
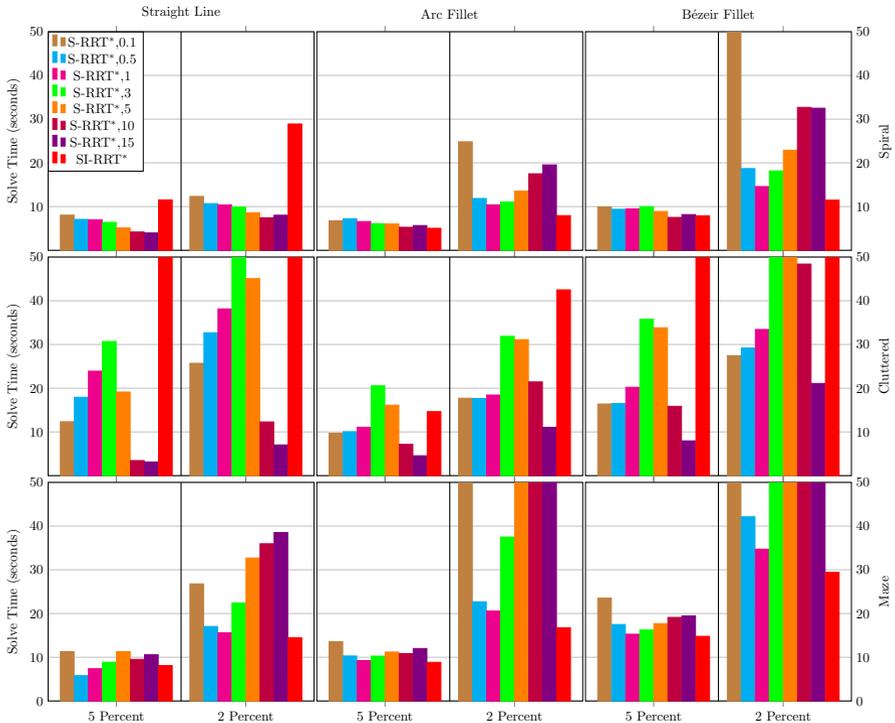



  \section{Conclusion}
  \label{sec:conclusion}
  In this work, an RRT-based path planning algorithm is proposed that uses general fillets as motion primitives.
  An arc-fillet is designed to provide path continuity similar to a Dubin's path while a B\`ezier-fillet is developed to provide continuity in path curvature.
  RRT*-like procedures are developed to accommodate fillet-based motion primitives.
  Simulation results show that planning with arc-fillets significantly outperforms the use of Dubin's paths as a motion primitive.
  Planning with arc-fillets is shown to perform almost as well as straight-line motion primitives.
  Planning with B\`ezier-fillets exhibits slightly worse performance than arc-fillets, although it far outperformed planning with Dubin's despite considering more complex dynamic constraints.
  A comparison is made between informed sampling, smart sampling, and a new smart-and-informed sampling technique.
  Like their straight-line counterparts, the fillet-based planners perform better with informed sampling when the heuristic for the shortest path is valid and better with smart sampling when there are fewer local minima.
  The smart-and-informed sampling performed well for fillet-based motion primitives and was found to be most applicable to environments with fewer local minima.
  In such environments, it performed on par with the best smart beacon size without the need for an iterative search for that beacon size.


  \section{Declarations}
    \subsection{Funding}
      The authors declare that no funds, grants, or other support were received during the preparation of this manuscript.
    \subsection{Competing Interests}
      The authors have no relevant financial or non-financial interests to disclose.
    \subsection{Code Availability}
      Simulation code can be found in our open-source repository \url{https://gitlab.com/utahstate/robotics/fillet-rrt-star}.
    \subsection{Author Contributions}
      All authors contributed to the work's conception and design.
      Software development, material preparation, and data collection were performed by James Swedeen.
      Data analyses were performed by James Swedeen and Dr. Greg Droge. James Swedeen is the primary author with major editorial and conceptual contributions made by Dr. Greg Droge and Dr. Randall Christensen at various stages of the writing process.
      All authors read and approved the final manuscript.
    \subsection{Ethics Approval}
      Not applicable.
    \subsection{Consent to Participate}
      Not applicable.
    \subsection{Consent for Publication}
      Not applicable.

  \bibliography{main}

  \begin{appendices}
    \section{\Revision{Reverse Fillet}}
    \label{sec:reverse_fillet}
    
This section describes a novel reverse fillet formulation that enables the ability to plan paths that go forward and backward.
The reverse fillet formulation uses a generic one-directional fillet internally that can be replaced by any fillet primitive desired.
This section ends with an example planning problem that necessitates reverse and forward motion to follow the shortest path possible.

When using the reverse fillet motion primitive the direction state, $d$, is added to the state space.
$d=1$ when the vehicle is moving forward and $d=-1$ when the vehicle is moving backward.
It can be determined if a fillet should keep the direction of travel the same or change it by comparing the $d$ values of the nodes that the fillet connect.
When using \mbox{FB-RRT*} with the reverse fillet formulation, $d$ is randomly sampled from a uniform distribution when the state space is sampled.

Before the $ReverseFillet$ procedure can be given a few pieces of notation must be defined.
We use the notion $x_{0,d}$ to denote the $d$ value of node $x_0$ and $x_{0,xy}$ to denote the position vector of node $x_0$.
The function $R\left(\cdot\right)$ consumes an angle and produces a two-by-two right handed rotation matrix.

\begin{algorithm}[t]
  \caption{\hbox{$X_{fillet} \leftarrow ReverseFillet\left(x_0,x_1,x_2,x_3\right)$}}
  \label{alg:reverse_fillet}
  \begin{algorithmic}[1]
    \If{$x_{2,d} = x_{3,d}$} \Comment{Fillet with uniform path direction}
    \State $X_{unidir} \leftarrow Fillet\left(x_{0,xy},x_{1,xy},x_{2,xy},x_{3,xy}\right)$ \label{alg:reverse_fillet:normal_fillet}
    \State $X_{fillet} \leftarrow \left\{\begin{bmatrix} x_{u,x} & x_{u,y} & x_{3,d} \end{bmatrix}^T, x_u \in X_{unidir}\right\}$ \label{alg:reverse_fillet:normal_fillet_dir}
    \Else \Comment{Fillet with switching path direction}
      \LineComment{Rotation matrix for the orientation of $x_2$}
      \State $R_2 \leftarrow R\left(-atan2\left(x_{2,y} - x_{1,y}, x_{2,x} - x_{1,x}\right)\right)$ \label{alg:reverse_fillet:rot_mat}
      \LineComment{Put each point in the frame with $x_2$ at the origin and $x_1$ on the x-axis}
      \State $_2x_{0} \leftarrow R_2 \cdot \left[x_{0,xy} - x_{2,xy}\right]$
      \State $_2x_{1} \leftarrow R_2 \cdot \left[x_{1,xy} - x_{2,xy}\right]$
      \State $_2x_2 \leftarrow \begin{bmatrix} 0 & 0 \end{bmatrix}^T$
      \State $_2x_{3} \leftarrow R_2 \cdot \left[x_{3,xy} - x_{2,xy}\right]$ \label{alg:reverse_fillet:end_frame_shift}
      \State $_2x_{3f} \leftarrow \begin{bmatrix} -{_2x_{3,x}} & {_2x_{3,y}}\end{bmatrix}^T$ \label{alg:reverse_fillet:flip_x3} \Comment{Flip $_2x_{3}$ across y-axis}
      \State $X_{unidir} \leftarrow Fillet\left({_2x_{0}}, {_2x_{1}},{_2x_2},{_{2}x_{3f}}\right)$ \label{alg:reverse_fillet:forward_fill} \Comment{Make unidirectional fillet}
      \LineComment{Flip any part of the fillet that is in front of $x_2$ over the y-axis}
      \State $X_{fillet} \leftarrow \left\{
        \begin{array}{lr}
          \begin{bmatrix}  x_{u,x} & x_{u,y} & x_{2,d} \end{bmatrix}^T & \text{for } x_{u,x} \leq 0 \\
          \begin{bmatrix} -x_{u,x} & x_{u,y} & x_{3,d} \end{bmatrix}^T & \text{for } x_{u,x} > 0
        \end{array}, x_u \in X_{unidir} \right\}$ \label{alg:reverse_fillet:flip_loop}
      \LineComment{Move the fillet back to the original coordinate frame}
      \State $X_{fillet} \leftarrow \left\{\begin{bmatrix} R_2^T x_{f,xy} + x_{2,xy} \\ x_{f,d} \end{bmatrix}, x_f \in X_{fillet}\right\}$ \label{alg:reverse_fillet:trans_back}
    \EndIf
    \State \Return $X_{fillet}$
  \end{algorithmic}
\end{algorithm}

\begin{figure}[h]
  \centering
  \newcommand{\CaptionWidth}{0.3\linewidth}
  \newcommand{\FigureWidth}{\linewidth}
  \newcommand{\MyScale}{1}
  \newcommand{\LabelSize}{\Huge}
  \newcommand{\PointSize}{2pt}

  \newcommand{\DrawAxis}{
    \draw[black] (-5.1,0) -- (5.1,0);%
    \draw[black] (0,-5.1) -- (0,5.1);%
  }

  \tikzstyle{node-point}=[circle,draw=black,fill=black,black,inner sep=\PointSize]

  \begin{subfigure}[t]{\CaptionWidth}
    \centering
    \resizebox{\FigureWidth}{!}{
      \tikzsetnextfilename{reverse_fillet_demo_a}
      \begin{tikzpicture}[scale=\MyScale]
        \DrawAxis
        \node[node-point,label={above:     {\LabelSize $x_0$}}] (x_0) at (3, -5) {};%
        \node[node-point,label={left:      {\LabelSize $x_1$}}] (x_1) at (3.5, -3.5) {};%
        \node[node-point,label={below:     {\LabelSize $x_2$}}] (x_2) at (1, -1) {};%
        \node[node-point,label={above left:{\LabelSize $x_3$}}] (x_3) at (5, -1.5) {};%
      \end{tikzpicture}
    }
    \caption{The four points that the fillet will be made between.}
    \label{fig:rev_fill_demo:sub0}
  \end{subfigure}
  \begin{subfigure}[t]{\CaptionWidth}
    \centering
    \resizebox{\FigureWidth}{!}{
      \tikzsetnextfilename{reverse_fillet_demo_b}
      \begin{tikzpicture}[scale=\MyScale]
        \DrawAxis
        \begin{scope}[rotate around z=45]
          \node[node-point,label={above:     {\LabelSize $_2x_0$}}] (_2x_0) at ($(3, -5)-(1,-1)$) {};%
          \node[node-point,label={above:     {\LabelSize $_2x_1$}}] (_2x_1) at ($(3.5, -3.5)-(1,-1)$) {};%
          \node[node-point,label={below left:{\LabelSize $_2x_2$}}] (_2x_2) at ($(1, -1)-(1,-1)$) {};%
          \node[node-point,label={above:     {\LabelSize $_2x_3$}}] (_2x_3) at ($(5, -1.5)-(1,-1)$) {};%
        \end{scope}
      \end{tikzpicture}
    }
  \caption{The four points after they have been transformed.} 
    \label{fig:rev_fill_demo:sub1}
  \end{subfigure}
  \begin{subfigure}[t]{\CaptionWidth}
    \centering
    \resizebox{\FigureWidth}{!}{
      \tikzsetnextfilename{reverse_fillet_demo_c}
      \begin{tikzpicture}[scale=\MyScale]
        \DrawAxis
        \begin{scope}[rotate around z=45]
          \node[node-point,label={above:     {\LabelSize $_2x_0$}}] (_2x_0) at ($(3, -5)-(1,-1)$) {};%
          \node[node-point,label={above:     {\LabelSize $_2x_1$}}] (_2x_1) at ($(3.5, -3.5)-(1,-1)$) {};%
          \node[node-point,label={below left:{\LabelSize $_2x_2$}}] (_2x_2) at ($(1, -1)-(1,-1)$) {};%
          \node (_2x_3) at ($(5, -1.5)-(1,-1)$) {};%
        \end{scope}
        \draw let \p1=(_2x_3) in node[node-point,label={above:{\LabelSize $_{2}x_{3f}$}}] (flipped_2x_3) at (-\x1,\y1) {};%
      \end{tikzpicture}
    }
    \caption{$_2x_3$ is flipped over the y-axis.}
    \label{fig:rev_fill_demo:sub2}
  \end{subfigure}
  \begin{subfigure}[t]{\CaptionWidth}
    \centering
    \resizebox{\FigureWidth}{!}{
      \tikzsetnextfilename{reverse_fillet_demo_d}
      \begin{tikzpicture}[scale=\MyScale]
        \DrawAxis
        \begin{scope}[rotate around z=45]
          \node[node-point,label={above:     {\LabelSize $_2x_0$}}] (_2x_0) at ($(3, -5)-(1,-1)$) {};%
          \node[node-point,label={above:     {\LabelSize $_2x_1$}}] (_2x_1) at ($(3.5, -3.5)-(1,-1)$) {};%
          \node[node-point,label={below left:{\LabelSize $_2x_2$}}] (_2x_2) at ($(1, -1)-(1,-1)$) {};%
          \node (_2x_3) at ($(5, -1.5)-(1,-1)$) {};%
        \end{scope}
        \draw let \p1=(_2x_3) in node[node-point,label={above:{\LabelSize $_{2}x_{3f}$}}] (flipped_2x_3) at (-\x1,\y1) {};%
        \draw[blue,thick]%
          let \p1=($(_2x_1)-(_2x_2)$),%
              \p2=($(_2x_2)-(flipped_2x_3)$),%
              \n1={atan2(\y1,\x1)+deg(pi/2)},%
              \n2={atan2(\y2,\x2)+deg(pi/2)},%
              \n3={deg(pi) - abs(mod(abs(\n2 - \n1), deg(2*pi)) - deg(pi))},%
              \n4={abs(-8cm * ((1-cos(\n3))/sin(\n3)))},%
              \p3=($(_2x_2) !\n4! (_2x_1)$),%
              \p4=($(_2x_2) !\n4! (flipped_2x_3)$)%
           in (_2x_1) -- (\p3) arc (\n1:\n2:-8cm) -- (flipped_2x_3);%
      \end{tikzpicture}
    }
    \caption{The unidirectional fillet is generated.}
    \label{fig:rev_fill_demo:sub3}
  \end{subfigure}
  \begin{subfigure}[t]{\CaptionWidth}
    \centering
    \resizebox{\FigureWidth}{!}{
      \tikzsetnextfilename{reverse_fillet_demo_e}
      \begin{tikzpicture}[scale=\MyScale]
        \DrawAxis
        \begin{scope}[rotate around z=45]
          \node[node-point,label={above:     {\LabelSize $_2x_0$}}] (_2x_0) at ($(3, -5)-(1,-1)$) {};%
          \node[node-point,label={above:     {\LabelSize $_2x_1$}}] (_2x_1) at ($(3.5, -3.5)-(1,-1)$) {};%
          \node[node-point,label={below left:{\LabelSize $_2x_2$}}] (_2x_2) at ($(1, -1)-(1,-1)$) {};%
          \node[node-point,label={above:     {\LabelSize $_2x_3$}}] (_2x_3) at ($(5, -1.5)-(1,-1)$) {};%
        \end{scope}
        \draw let \p1=(_2x_3) in node[node-point,label={above:{\LabelSize $_{2}x_{3f}$}}] (flipped_2x_3) at (-\x1,\y1) {};%
        \begin{scope}
          \clip (0,0) rectangle (5,5);
          \draw[blue,thick]%
            let \p1=($(_2x_1)-(_2x_2)$),%
                \p2=($(_2x_2)-(flipped_2x_3)$),%
                \n1={atan2(\y1,\x1)+deg(pi/2)},%
                \n2={atan2(\y2,\x2)+deg(pi/2)},%
                \n3={deg(pi) - abs(mod(abs(\n2 - \n1), deg(2*pi)) - deg(pi))},%
                \n4={abs(-8cm * ((1-cos(\n3))/sin(\n3)))},%
                \p3=($(_2x_2) !\n4! (_2x_1)$),%
                \p4=($(_2x_2) !\n4! (flipped_2x_3)$)%
             in (_2x_1) -- (\p3) arc (\n1:\n2:-8cm) -- (flipped_2x_3);%
        \end{scope}
        \begin{scope}
          \clip (0,0) rectangle (5,5);
          \draw[blue,thick]%
            let \p1=($(_2x_1)-(_2x_2)$),%
                \p2=($(_2x_2)-(flipped_2x_3)$),%
                \n1={atan2(\y1,\x1)+deg(pi/2)},%
                \n2={atan2(\y2,\x2)+deg(pi/2)},%
                \n3={deg(pi) - abs(mod(abs(\n2 - \n1), deg(2*pi)) - deg(pi))},%
                \n4={abs(-8cm * ((1-cos(\n3))/sin(\n3)))},%
                \p3=($(_2x_2) !\n4! (_2x_1)$),%
                \p4=($(_2x_2) !\n4! (flipped_2x_3)$),%
                \p5=(_2x_1),%
                \p6=(flipped_2x_3)%
             in (-\x5,\y5) -- (-\x3,\y3) arc (-\n1:-\n2:8cm) -- (-\x6,\y6);%
        \end{scope}
        \filldraw[black] (_2x_3) circle(\PointSize);
      \end{tikzpicture}
    }
    \caption{The second half of the unidirectional fillet is flipped to reverse it.}
    \label{fig:rev_fill_demo:sub4}
  \end{subfigure}
  \begin{subfigure}[t]{\CaptionWidth}
    \centering
    \resizebox{\FigureWidth}{!}{
      \tikzsetnextfilename{reverse_fillet_demo_f}
      \begin{tikzpicture}[scale=\MyScale]
        \DrawAxis
        \node[node-point,label={above:     {\LabelSize $x_0$}}] (x_0) at (3, -5) {};%
        \node[node-point,label={left:      {\LabelSize $x_1$}}] (x_1) at (3.5, -3.5) {};%
        \node[node-point,label={below:     {\LabelSize $x_2$}}] (x_2) at (1, -1) {};%
        \node[node-point,label={above left:{\LabelSize $x_3$}}] (x_3) at (5, -1.5) {};%
        \begin{scope}[rotate around z=45]
          \node[inner sep=\PointSize] (_2x_0) at ($(3, -5)-(1,-1)$) {};%
          \node[inner sep=\PointSize] (_2x_1) at ($(3.5, -3.5)-(1,-1)$) {};%
          \node[inner sep=\PointSize] (_2x_2) at ($(1, -1)-(1,-1)$) {};%
          \node[inner sep=\PointSize] (_2x_3) at ($(5, -1.5)-(1,-1)$) {};%
        \end{scope}
        \draw let \p1=(_2x_3) in node (flipped_2x_3) at (-\x1,\y1) {};%
        \begin{scope}[transform canvas={rotate around={-45:(0,0)},shift={(x_2)}}]
          \begin{scope}
            \clip (0,0) rectangle (5,5);
            \draw[blue,thick]%
              let \p1=($(_2x_1)-(_2x_2)$),%
                  \p2=($(_2x_2)-(flipped_2x_3)$),%
                  \n1={atan2(\y1,\x1)+deg(pi/2)},%
                  \n2={atan2(\y2,\x2)+deg(pi/2)},%
                  \n3={deg(pi) - abs(mod(abs(\n2 - \n1), deg(2*pi)) - deg(pi))},%
                  \n4={abs(-8cm * ((1-cos(\n3))/sin(\n3)))},%
                  \p3=($(_2x_2) !\n4! (_2x_1)$),%
                  \p4=($(_2x_2) !\n4! (flipped_2x_3)$)%
               in (_2x_1) -- (\p3) arc (\n1:\n2:-8cm) -- (flipped_2x_3);%
          \end{scope}
          \begin{scope}
            \clip (0,0) rectangle (5,5);
            \draw[blue,thick]%
              let \p1=($(_2x_1)-(_2x_2)$),%
                  \p2=($(_2x_2)-(flipped_2x_3)$),%
                  \n1={atan2(\y1,\x1)+deg(pi/2)},%
                  \n2={atan2(\y2,\x2)+deg(pi/2)},%
                  \n3={deg(pi) - abs(mod(abs(\n2 - \n1), deg(2*pi)) - deg(pi))},%
                  \n4={abs(-8cm * ((1-cos(\n3))/sin(\n3)))},%
                  \p3=($(_2x_2) !\n4! (_2x_1)$),%
                  \p4=($(_2x_2) !\n4! (flipped_2x_3)$),%
                  \p5=(_2x_1),%
                  \p6=(flipped_2x_3)%
               in (-\x5,\y5) -- (-\x3,\y3) arc (-\n1:-\n2:8cm) -- (-\x6,\y6);%
          \end{scope}
        \end{scope}
        \filldraw[black] (x_3) circle(\PointSize);
      \end{tikzpicture}
    }
    \caption{The fillet is transformed back into the original coordinate frame.}
    \label{fig:rev_fill_demo:sub5}
  \end{subfigure}
  \caption{An illustration of the $ReverseFillet$ procedure.}
  \label{fig:rev_fill_demo}
\end{figure}
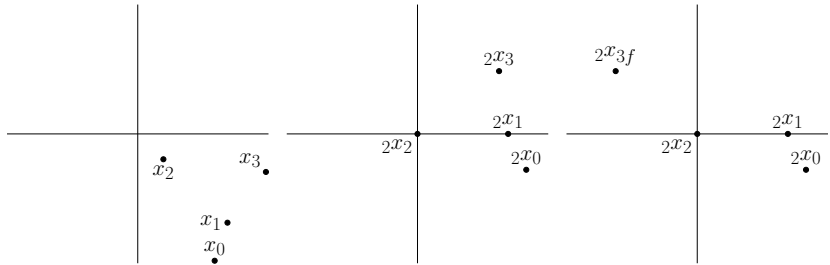
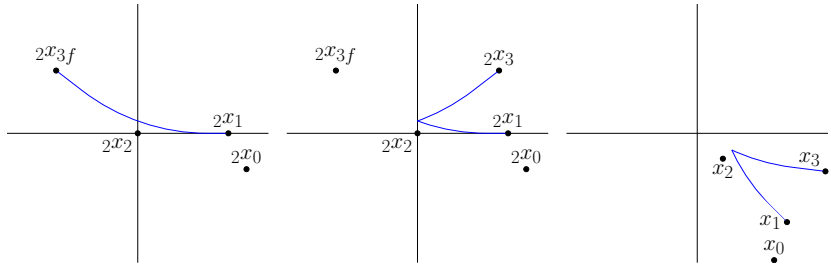

Algorithm \ref{alg:reverse_fillet} gives the procedure for generating a reverse fillet and Figure \ref{fig:rev_fill_demo} illustrates the process.
If the direction of travel for the two points being connected, $x_2$ and $x_3$, are the same then a normal fillet can be made to connect them, see lines \ref{alg:reverse_fillet:normal_fillet} and \ref{alg:reverse_fillet:normal_fillet_dir}.
If the direction of travel changes, more logic is needed to make use of a unidirectional fillet primitive to make a fillet that changes direction.
The process, shown on lines \ref{alg:reverse_fillet:rot_mat} through \ref{alg:reverse_fillet:trans_back}, involves flipping $x_3$ over the plane that intersects $x_2$ and is perpendicular to $\overline{x_1 x_2}$ to produce a new point, $x_{3f}$.
An unidirectional fillet can be designed using $x_{3f}$.
The fillet to execute is obtained by flipping the portion of the unidirectional fillet from $x_2$ to $x_{3f}$ back so that the path ends at $x_3$.
The portion from $x_2$ to $x_{3}$ will then be executed in the opposite direction of that from $x_1$ to $x_2$.

First $x_0$, $x_1$, and $x_3$ are transformed into a coordinate frame with $x_2$ at the origin and $x_1$ on the x-axis, see lines \ref{alg:reverse_fillet:rot_mat} through \ref{alg:reverse_fillet:end_frame_shift} and Figure \ref{fig:rev_fill_demo:sub1}.
This frame is defined to make flipping $x_3$ easier.
The prescript $2$ is used to denote a point in this frame, i.e., $_2x_0$ is the point $x_{0,xy}$ in the new frame.
On line \ref{alg:reverse_fillet:flip_x3}, $_2x_3$ is flipped across the y-axis, see Figure \ref{fig:rev_fill_demo:sub2}.
The transformed set of points $_2x_{0}$, $_2x_{1}$, $_2x_2$, and $_{2}x_{3f}$ form a chain of nodes that do not change direction.
Line \ref{alg:reverse_fillet:forward_fill} generates a unidirectional fillet called $X_{unidir}$ with these modified points, see Figure \ref{fig:rev_fill_demo:sub3}.
Line \ref{alg:reverse_fillet:flip_loop} fills $X_{fillet}$ with a fillet that starts moving in the same direction at $x_2$ while the x component of $X_{unidir}$ is negative.
When the x component of $X_{unidir}$ hits the y-axis the fillet switches to the direction of travel of $x_3$ and flips the fillet across the y-axis.
The result is a fillet that comes to a point on the y-axis and switches direction at that point, see Figure \ref{fig:rev_fill_demo:sub4}.
Line \ref{alg:reverse_fillet:trans_back} transforms fillet back to the original coordinate frame, as shown in Figure \ref{fig:rev_fill_demo:sub5}.

\begin{figure}[h]
  \centering
  \resizebox{0.5\linewidth}{!}{
    \begin{tikzpicture}
      \begin{axis}[
        enlargelimits=false,
        axis lines=none,
        xmin=-15,ymin=-15,xmax=15,ymax=15,
        scale only axis,
        axis equal image,
        axis equal=true,
        ]
        \clip  (15,15) rectangle (-15,-8);

        \addplot graphics [xmin=-15,ymin=-15,xmax=15,ymax=15] {pics/curvature_constrained_path/curvature_constrained_path.png};

        \addplot[color=blue]  table [x=x,y=y, col sep=comma] {pics/curvature_constrained_path/arc_fillet.csv};
        \addplot[color=green] table [x=x,y=y, col sep=comma] {pics/curvature_constrained_path/reverse_arc_fillet.csv};
        \addplot[color=red]   table [x=x,y=y, col sep=comma] {pics/curvature_constrained_path/reverse_arc_fillet_reverse.csv};

        \node[label={left: {\small$x_r$}},inner sep=0pt] () at (-12.5,-6.1) {};
        \node[label={above:{\small$X_t$}},inner sep=0pt] () at (-3,12) {};

      \end{axis}
    \end{tikzpicture}
  }
  \caption{The resulting paths from running \mbox{FB-RRT*} in an environment where the shortest path turns sharply down a corridor.
           Solutions from planning with arc-fillet paths are shown in blue.
           Solutions from planning with reverse-arc-fillet paths are shown in green when $d=1$, forward, and red when $d=-1$, reverse.}
  \label{fig:reverse_curvature_demo_world}
\end{figure}
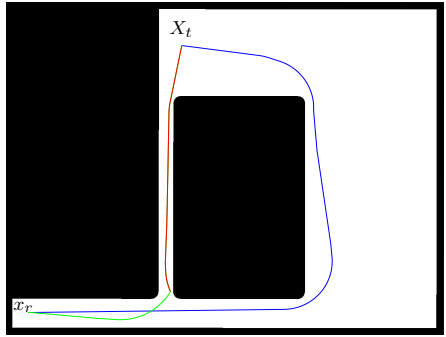

One scenario where the ability to plan forward and reverse fillets is beneficial is shown in Figure \ref{fig:reverse_curvature_demo_world}.
Figure \ref{fig:reverse_curvature_demo_world} uses the same planning configuration as Figure \ref{fig:curvature_demo_world} from Section \ref{sec:curvature_constrained_paths}.
The only difference between Figures \ref{fig:reverse_curvature_demo_world} and \ref{fig:curvature_demo_world} is that \ref{fig:reverse_curvature_demo_world} shows the result of path planning using the $ReverseFillet$ procedure in green and red instead of the solution found with straight-line primitives.
The solution found from planning without the $ReverseFillet$ procedure is shown in blue.
Both planners are using arc-fillets with the same maximum curvature constraint, but the red path makes use of the $ReverseFillet$ procedure.

As is described in Section \ref{sec:curvature_constrained_paths}, a path that goes through the narrow hallway cannot satisfy the curvature constraints of the problem without hitting walls when solely forward motion is considered.
Figure \ref{fig:reverse_curvature_demo_world} shows that it is possible using the $ReverseFillet$ procedure.
Following the green and red solution, generated with forward and reverse motion, the path turns partially into the narrow hallway.
When it nears the wall, the path stops and continues the turn in reverse.
The path follows the hallway in reverse until it gets to $X_t$.
Without the functionality added with the $ReverseFillet$ procedure, the blue solution is unable to follow the hallway and instead must plan a significantly longer path that goes around the obstacles.
Note that the inclusion of the reverse motion causes an increase in convergence time due to the added dimension in the sampling space.
Future work could include methods to reduce this complexity and also to penalize long stretches of reverse motion.

  \end{appendices}
\end{document}